\documentclass[11pt]{article}
\usepackage{hyperref}
\usepackage{mathtools}
\usepackage{empheq}
\usepackage{amsfonts}
\usepackage{amsgen,amsmath,amstext,amsbsy,amsopn,amssymb,subfigure,amsthm,stmaryrd}
\usepackage[dvips]{graphicx}
\usepackage{comment}
\usepackage[round]{natbib}
\usepackage{booktabs}
\usepackage{blkarray}
\usepackage{algorithm}
\usepackage{algpseudocode}
\usepackage{url}
\usepackage{longtable}
\usepackage{mathtools}
\usepackage{multirow}
\usepackage{caption}
\usepackage{chngcntr}
\counterwithin{table}{section}

\usepackage[usenames]{color}

\usepackage{chngcntr}
\counterwithin{figure}{section}

\usepackage[usenames]{color}
\usepackage{slashbox}

\allowdisplaybreaks

\newtheorem{thm}{Theorem}
\newtheorem{lemma}{Lemma}[section]

\allowdisplaybreaks

\DeclareMathAlphabet\mathbfcal{OMS}{cmsy}{b}{n}

\newcommand{\be}{\begin{equation}}
\newcommand{\ee}{\end{equation}}
\newcommand{\bea}{\begin{eqnarray}}
\newcommand{\eea}{\end{eqnarray}}
\newcommand{\beas}{\begin{eqnarray*}}
	\newcommand{\eeas}{\end{eqnarray*}}

\newcommand{\bbR}{\mathbb{R}}
\newcommand{\cC}{\mathcal{C}}
\newcommand{\cF}{\mathcal{F}}

\newcommand{\cT}{\mathcal{T}}

\newcommand{\MM}{\mathrm{(MM)}}
\newcommand{\MN}{\mathrm{MN}}
\newcommand{\LN}{\mathrm{LN}}
\newcommand{\bern}{\mathrm{Bern}}

\newcommand{\X}{{\mathbf{X}}}
\newcommand{\B}{{\mathbf{B}}}

\newcommand{\Cov}{{\rm Cov}}

\newcommand{\widesim}[2][1.5]{
  \mathrel{\overset{#2}{\scalebox{#1}[1]{$\sim$}}}
}

\newcommand{\argmin}{\mathop{\rm arg\min}}

\newcommand{\bbP}{\mathbb{P}}
\newcommand{\bbE}{\mathbb{E}}

\newcommand{\ie}[0]{\emph{i.e., }}

\newcommand{\eg}[0]{\emph{e.g., }}

\usepackage{bbm}
\newcommand{\ones}{{{\mathbbm{1}}}}
\newcommand{\ind}[1]{\ones_{\{#1\}}}
\usepackage{enumerate}

\makeatletter
\newcommand*{\rom}[1]{\expandafter\@slowromancap\romannumeral #1@}
\makeatother 

\allowdisplaybreaks

\begin{document}
\title{Context-dependent self-exciting point processes: \\ models, methods, and risk bounds in high dimensions}
\author{Lili Zheng$^1$, Garvesh Raskutti$^1$, Rebecca Willett$^{2}$, Benjamin Mark$^3$}
	
	\date{}
	
	\maketitle
\footnotetext[1]{Department of Statistics, University of Wisconsin-Madison}
\footnotetext[2]{Departments of Statistics and Computer Science, University of Chicago}
\footnotetext[3]{Department of Mathematics, University of Wisconsin-Madison}
	
	\bigskip
	\begin{abstract}
High-dimensional autoregressive point processes model how current events trigger or inhibit future events, such as activity by one member of a social network can affect the future activity of his or her neighbors.
While past work has focused on estimating the underlying network structure based solely on the times at which events occur on each node of the network,
this paper examines the more nuanced problem of estimating {\em context-dependent} networks that reflect how features associated with an event (such as the content of a social media post) modulate the strength of influences among nodes. 
Specifically, we leverage ideas from compositional time series and
regularization methods in  machine learning to conduct network estimation for high-dimensional \emph{marked} point processes. 
Two models and corresponding estimators are considered in detail: an autoregressive multinomial model suited to categorical marks and a logistic-normal model suited to marks with mixed membership in different categories.
Importantly, the logistic-normal model leads to a convex negative log-likelihood objective and captures dependence across categories. We provide theoretical guarantees for both estimators, which we validate by simulations and a synthetic data-generating model. We further validate our methods through two real data examples and demonstrate the advantages and disadvantages of both approaches. 
\end{abstract}


\section{Introduction}

High-dimensional self-exciting point processes arise in a broad range of applications. For instance, in a social network, we may observe a time series of members' activities,
such as posts on social media where each person's post can influence their neighbors' future posts (\eg 
\cite{stomakhin2011reconstruction,romero2011influence}).
In the broadcast of social events, news media sources play a key role and influential news media sources often trigger others to post new articles \mbox{\citep{leskovec2009meme,farajtabar2017fake}}. In electrical systems, cascading chains
of power failures reveal critical information about the underlying
power distribution network \citep{rudin2011machine,ertekin2015reactive}. 
During epidemics, networks among computers or people are reflected by the time at which each node becomes infected \citep{ganesh2005effect,yang2013epidemics}. In biological neural networks, firing neurons can trigger or inhibit the firing of their neighbors, so that information about the network structure is embedded within spike train observations \citep{linderman2016bayesian, fletcher2014scalable,hall2015online,pillow2008spatio,gerhard2017stability}. 
The above processes are {\em self-exciting} in that the
likelihood of future events depends on past events (\ie a particular
type of
autoregressive process).

In many applications, events are associated with feature vectors
describing the events. For instance, interactions in a social network
have accompanying text, images, or videos; and power failures are
accompanied by information about current-carrying cables, cable ages,
and cable types. This feature vector associated with an event is referred to as a {\em mark} in
the point process literature. Prior works \citep{hall2016inference, mark2018network} describe methods and theoretical guarantees for network influence estimation given multivariate event data without accounting for the type or context of the event. The contribution of this paper focuses on {\em estimation methods and theoretical guarantees for context-dependent network structures} which exploit marks. 
  The key idea is that different categories of events are characterized
  by different (albeit related) functional networks; we think of the
  feature vector as revealing the {\em context} of each event, and our
  task is to infer context-specific functional networks. Allowing for marks provides a much richer model class that can
reflect, for instance, that people interact in a social network
differently when interactions are family-focused vs.\ work-focused
vs.\ political \citep{puniyani2010social,feller2011divided,williams2013system}. Learning richer models like these allows our methods
to have much stronger predictive capabilities, and to provide more insights in the network structure.


However, developing a statistical model for marked self-exciting point processes is a non-trivial task. One particular challenge is that we usually cannot determine the exact category of the event. For example, a post on social media may exhibit membership in several topics \citep{blei2003latent}; an infected patient's symptoms can be caused by different diseases \citep{woodbury1978mathematical}; a new product released to the market could contain several features or styles. Some natural-seeming models lead to computationally-intractable estimators, while others fail to account for ambiguity in the marks. 
In this paper we propose two models that suit distinct scenarios:

\begin{enumerate}[(i)]
\item {\bf Multinomial Model:} This model is applied when each event (\ie its mark) naturally belongs to a single category. For example, a tweet may clearly belong to a single category (\eg ``political''). 
\item {\bf Logistic-normal Model:} This model is applied when each event is a mixture of multiple categories (\ie mixed membership). For example, a news article may belong to two or more categories (\eg ``political'' and ``finance'') and we may only have measurements of the relative extent to which it's in each category. 
\end{enumerate}

To the best of our knowledge,
the multinomial model we consider appeared first in \cite{tank2017granger}, while no theoretical guarantee was provided. From both a modeling and theoretical perspective, the logistic-normal model is more nuanced. It employs the logistic-normal distribution widely used in compositional data analysis (\eg  \citet{aitchison1982statistical,brunsdon1998time,ravishanker2001compositional}). 
The logistic-normal model has advantages over other mixed membership models such as the Dirichlet distribution and the more recent Gumbel soft-max distribution since it leads to a convex negative log-likelihood function and models dependence among sub-compositions of the membership vector, which will be explained in detail in the beginning of Section \ref{sec:set_up_LN_cts}.
\paragraph{High-dimensional setting:}
Throughout this paper we focus on the \emph{high-dimensional} setting, where the number of nodes in the network is large and grows with sample size. We assume the number of edges within the huge network to be \emph{sparse}: each  node should only be influenced by a limited number of other nodes. We state this condition more formally in  Section \ref{sec:theory}.

\subsection{Contributions} 
Our contributions are summarized as follows:
\begin{itemize}
\item For both models, we present estimation algorithms based on minimizing a convex loss function using a negative log-likelihood loss plus a regularization term that accounts for the sparsity of networks but shared network structure between models corresponding to different categories. 
\item Furthermore, we establish risk bounds that characterize the error decay rate as a function of network size, sparsity, shared structure, and number of observations, and these bounds are illustrated with a variety of simulation studies. 
\item Finally, we validate the hypothesis that the logistic-normal is more suitable for mixed membership settings while the multinomial model is more suitable for settings with a clear dominant category through experimental results on real data from two datasets and a synthetic data-generating model. 
The synthetic data model is based on a noisy logistic-normal distribution with some nodes having events with a single dominant category and other nodes following a mixed membership setting. The multinomial model tends to correctly detect the edges between nodes with a single dominant category while the logistic-normal approach tends to correctly detect the edges corresponding to nodes with mixed membership categories. We further validate the hypothesis with two datasets: (1) a political tweets data set focusing on the network which varies according to political leanings of tweets and (2) online media data set where the network depends on topics of memes. The networks detected for both datasets tend to support the above hypothesis. 

\end{itemize}

\subsection{Related Work}\label{sec:related_work}
There has been substantial literature on recovering network structure using time series of event data in recent years, including continuous-time approaches \citep{zhou2013learning,yang2017online} based on Hawkes process \citep{hawkes1971spectra} and discrete-time approaches \citep{linderman2016bayesian, fletcher2014scalable,hall2016inference,mark2018network}. 
Our work follows the line of works (discrete time approaches): \cite{hall2016inference, mark2018network}, but with the additional challenge of incorporating the context information of events. \cite{tank2017granger} considers the multinomial model with exact categorical information of events but provides no theoretical guarantees. 


Another popular approach aiming to recover the text-dependent network structure in social media is the cascade analysis \citep{lerman2010information,yu2017topicbased,yu2018learning}, which focuses on the diffusion of information, \eg retweeting or sharing the same hyperlink. However, it is also possible for users to interact in social media by posting about similar topics (\eg showing condolence for shooting events) or arguing about opposite opinions (\eg tweets sent by presidential candidates) without sharing exactly the same text. This kind of interaction is captured by our approach but not by the cascade analysis. Due to the nature of our models, we can also study time series of event data with any categorical marks (either exact or with uncertainty/mixed membership), without diffusion of information involved. Examples include the stock price changes with corresponding business news as side information. We can also analyze multi-node compositional time series (\cite{brunsdon1998time,ravishanker2001compositional} are existing works on single-node compositional time series) if we consider a special case of the logistic-normal model \eqref{eq:model_disc} and \eqref{eq:model_cts} with $q=1$.
Our work also incorporates proof techniques from the high-dimensional statistics literature (\eg \cite{bickel2009simultaneous, raskutti2010restricted}), whilst incorporating the nuances of temporal dependence, non-linearity and context-based information not captured in prior works. 

The remainder of this paper is organized as follows: we elaborate on our problem formulations and corresponding estimators in Section \ref{sec:model_est}; theoretical guarantees on estimation errors are provided in Section \ref{sec:theory}; we also present simulation results on synthetic data and our synthetic model example in Section \ref{sec:numeric_exp} and real data experiments in Section \ref{sec:real_data}, respectively.

\section{Problem Formulation and Estimators}\label{sec:model_est}
We begin by introducing basic notation. For any two tensors $A$ and $B$ of the same dimension, let $\langle A, B\rangle$ denote the Euclidean inner product of $A$ and $B$. Also, define the Frobenius norm of tensor $A$ as 
$\|A\|_{F}=\langle A, A\rangle^{\frac{1}{2}}.$

For any $4$th-order tensor $A\in \bbR^{n_1\times n_2\times n_3\times n_4}$, define the
regularization norm $\|\cdot\|_{R}$ as 
\begin{equation}\label{eq:nmR}
    \|A\|_{R}=\sum_{m,m'}\|A_{m,:,m',:}\|_{F}.
\end{equation} 
For any matrix $A$, we let $\lambda_{\min}(A)$ denote the smallest eigenvalue of $A$.
Let
$\ind{E} = {\footnotesize \begin{cases}1,& \text{if } E \text{ true} \\ 0,& \text{else}\end{cases}}$ be the indicator function.

$M$ refers to the number of nodes (multiple time series) and let $X^t\in \mathbb{R}^{M\times K}$ be the observed data during time period $t$ for $t = 0,1,...,T$, where $K$ is the number of categories of events.
For each $1\leq m\leq M$, $0\leq t\leq T$, if there is no event, $X^t_m\in \mathbb{R}^K$ is a zero vector. 
For times and nodes with events, we consider two different observation models. The first model is the multinomial model (Sec.~\ref{sec:set_up_mult}) corresponding to the setting in which
each event only belongs to a single category. In this case, if the event at time $t$ and node $m$ is in category $k \in \{1,\ldots,K\}$, then we let $X^t_m=e_k$, where $e_k$ is the $k$-th vector in the canonical basis of $\mathbb{R}^K$. 
The second is the logistic-normal model (Sec.~\ref{sec:setup_mixed}) corresponding to the setting in which each event has mixed category membership and that membership is potentially observed with noise. In this case, we let $X^t_m$ be a vector on the simplex $\triangle^{K-1}$, with non-negative elements summing up to one. The following two sections address these two cases separately. 

\subsection{Multinomial Model}\label{sec:set_up_mult}
When each event belongs to a single category, the distribution of $\{X^{t+1}_m\}_{m=1}^M$ conditioned on the past data $X^t$ can be modeled as independent multinomial random vectors. Specifically, let tensor $A^{\MN}\in \mathbb{R}^{M\times K\times M\times K}$ encode the context-dependent network, and each entry $A^{\MN}_{mkm'k'}$ is the influence exerted upon \{node $m$, category $k$\} by \{node $m'$, category $k'$\}. We will refer to this influence as {\em absolute} influence, contrasted with the relative influence and overall influence in the logistic-normal model introduced later.
That is, an event from node $m'$ in category $k'$ may increase or decrease the likelihood of a future event by node $m$ in category $k$, and $A^{\MN}_{mkm'k'}$ parameterizes that change in likelihood. We can also think of this network as a collection of $K\times K$ subnetworks indexed by $\{A^{\MN}_{:,k,:,k'}\in \bbR^{M\times M},1\leq k,k'\leq K\}$, where each sub-network is the influence among the $M$ nodes, for a  pair of categories $(k, k')$. Further define $\nu^{\MN}\in \mathbb{R}^{M\times K}$ as the intercept term where each entry $\nu^{\MN}_{mk}$ determines the event rate of \{node $m$, category $k$\} when there are no past stimuli. The overall event rate is parameterized by the {\em intensity}. Then the intensity of \{node $m$, category $k$\} at time $t+1$ given the past is 
$$
\mu^{t+1}_{mk}=\langle A^{\MN}_{mk},X^t\rangle+\nu^{\MN}_{mk}=\sum_{m',k'} A^{\MN}_{mkm'k'} X^t_{m'k'}+\nu^{\MN}_{mk},
$$
and the conditional distribution of $X^{t+1}_m$ is
\begin{align}
\begin{split}
    \mathbb{P}(X_m^{t+1}=e_k|X^t)=&\frac{e^{\mu^{t+1}_{mk}}}{1+\sum_{k'=1}^Ke^{\mu^{t+1}_{mk'}}},\, 1\leq k\leq K \\
    \mathbb{P}(X_m^{t+1}=0|X^t)=&\frac{1}{1+\sum_{k'=1}^Ke^{\mu^{t+1}_{mk'}}}.
\label{eq:mult}
\end{split}
\end{align}
This is also the multinomial logistic transition distribution (mLTD) model considered in \citet{tank2017granger}. 

To estimate the parameter $A^{\MN}\in \mathbb{R}^{M\times K\times M\times K}$, one straightforward method is to find the minimizer of the penalized negative log-likelihood:
\begin{equation}\label{eq:est_mult}
    \widehat{A}^{\MN}=\argmin_{A\in \mathbb{R}^{M\times K\times M\times K}} L^{\MN}(A)+\lambda\|A\|_{R},
\end{equation}
where
\begin{align}
\begin{split}
L^{\MN}(A)
    =\frac{1}{T}\sum_{t=0}^{T-1} \sum_{m=1}^M\Big[f&(\langle A_m, X^{t}\rangle+\nu^{\MN}_m)-\sum_{k=1}^K (\langle A_{mk}, X^t\rangle+\nu^{\MN}_{mk}) X^{t+1}_{mk}\Big],
\end{split}
\end{align}
and $f:\bbR^K\rightarrow \bbR$ is defined by $f(x)=\log \left(\sum_{i=1}^K e^{x_i}+1\right)$. Note that $\|A\|_{R}$ is the group sparsity penalty defined in \eqref{eq:nmR}.

\subsection{Logistic-normal Model}\label{sec:setup_mixed}
When there is mixed membership, for each $0\leq t\leq T, 1\leq m\leq M$, the $K \times 1$ vector $X^t_m$ is either the zero vector or a vector on the simplex corresponding to the mixed membership probability of categories, thus we need to address the distribution in two parts: the probability mass of
$\ind{X^{t+1}_m\neq 0}$, and the distribution of $X^{t+1}_m$ given $X^{t+1}_m\neq 0$. 

Let $Z^{t+1}_m\in \triangle^{K-1}$ be a random vector on the simplex with a distribution to be specified shortly. We model the distribution of $\{X^{t+1}_m\}_{m=1}^M$ conditioned on the past as:
\begin{equation}\label{eq:model_disc}
    \begin{split}
        X^{t+1}_m =  &\begin{cases} Z^{t+1}_m, &\text{ with probability $q_m^{t+1}$},\\
0_K, &\text{ with probability 1-$q_m^{t+1}$},
\end{cases}
\end{split}
\end{equation}
and further assume conditional independence of entries for $\{X^{t+1}_m\}_{m=1}^M$. For $q^{t+1}\in [0,1]^M$, each element is the probability that an event occurs at the corresponding node and time $t+1$. We specify how $q^{t+1}$ is modeled later. 

\subsubsection{Modeling $Z^t$}\label{sec:set_up_LN_cts}
$Z^{t+1}_m$ may be modeled by two kinds of distributions widely used for compositional data: the Dirichlet distribution \citep{bacon2011short} and the logistic-normal distribution \citep{aitchison1982statistical}. The Dirichlet model gains its popularity in Bayesian statistics, but makes the limiting assumption that the sub-compositions are independent. More specifically, for any r.v. $X\in \bbR^K \sim \mathrm{Dir}(\alpha)$, 
$$
\left(\frac{X_1}{\sum_{i=1}^{k}X_i}, \dots, \frac{X_{k}}{\sum_{i=1}^{k}X_i}\right)\quad \text{and}\quad \left(\frac{X_{k+1}}{\sum_{i=k+1}^{K}X_i}, \dots, \frac{X_{K}}{\sum_{i=k+1}^{K}X_i}\right)$$ are independent for any $1\leq k\leq K-1$. Another difficulty associated with the Dirichlet modelling is the non-convexity of the negative log-likelihood objective which presents challenges both in terms of run-time and from a statistical perspective. 

Hence we employ the logistic-normal distribution which (i) has log-concave density function and thus provides fast run-time and more tractable theoretical analysis; (ii) incorporates the potential dependence among sub-compositions in different categories by introducing dependent Gaussian noise in the log-ratio \citep{atchison1980logistic,blei2006correlated}. 
The logistic-normal distribution is also related to the Gumbel-Softmax distribution \citep{jang2016categorical}, which has gained popularity in approximating a categorical distribution using a continuous one. The difference is that the logistic-normal distribution assumes the noise to be Gaussian and is thus more amenable to statistical analysis,  whereas the Gumbel-Softmax
employs the Gumbel distribution.

Specifically, for any $t\geq 0, 1\leq m\leq M$, given  $\{X^{t'}\}_{t'=0}^t$,
\begin{equation}\label{eq:model_cts}
\begin{split}
&\log \frac{Z^{t+1}_{mk}}{Z^{t+1}_{mK}}=\mu^{t+1}_{mk}+\epsilon^{t+1}_{mk}, \quad 1\leq k\leq K-1,\\
&\{\epsilon^{t+1}_m\}_{t,m}\widesim{i.i.d.} \mathcal{N}(0,\Sigma), \; \Sigma \in \bbR^{(K-1)\times (K-1)},
    \end{split}
\end{equation}
where $\mu_{mk}^{t+1}$ is a function of the $\{X^{t'}\}_{t'=0}^t$ that we specify below.
Here the $K$th category is used as a baseline category, so that we could transform $Z^{t+1}_m\in \triangle^{K-1}$ to log-ratios $\log \frac{Z^{t+1}_{mk}}{Z^{t+1}_{mK}}, k=1,\dots,K-1$, which take values on the entire $\bbR^{K-1}$ and can be modeled by a multivariate normal distribution.

$\mu^{t+1}_{mk}$ is the relative intensity of \{node $m$, category $k$\} compared to \{node $m$, category $K$\} at time $t+1$, given the past. $\epsilon^{t+1}_{m}\in \mathbb{R}^{(K-1)}$ is a Gaussian noise vector with covariance $\Sigma\in \mathbb{R}^{(K-1)\times (K-1)}$. 
To model $\mu^{t+1}_{mk}$, let $A^{\LN}\in \mathbb{R}^{M \times (K-1) \times M \times K}$ 
encode the network, where $A^{\LN}_{mkm'k'}$ is the relative influence exerted upon \{node $m$, category $k$\} relative to \{node $m$, category $K$\} by \{node $m'$, category $k'$\}. We can also think of this network as a collection of $K\times (K-1)$ relative sub-networks among the $M$ nodes, parameterized by $\{A^{\LN}_{:,k,:,k'}\in \bbR^{M\times M},1\leq k\leq K-1,1\leq k'\leq K\}$. Let $\nu^{\LN}\in \mathbb{R}^{M\times (K-1)}$ be the corresponding intercept term, where $\nu^{\LN}_{mk}$ is the intensity of \{node $m$, category $k$\} compared to \{node $m$, category $K$\}. Then we define 
$$
\mu^{t+1}_{mk}=\left\langle A^{\LN}_{mk},X^{t}\right\rangle+\nu^{\LN}_{mk}.
$$

Using a different baseline category does not change our model form, but only reparameterizes the model parameters. Specifically, if we take a different category, say $l$, as the baseline and want to model the distribution of $\log \frac{Z^{t+1}_{mk}}{Z^{t+1}_{ml}}$, then the model \eqref{eq:model_cts} can be equivalently written as: 
\begin{equation*}
\begin{split}
&\log \frac{Z^{t+1}_{mk}}{Z^{t+1}_{ml}}=\left\langle \widetilde{A}^{\LN}_{mk},X^t\right\rangle+\widetilde{\nu}^{\LN}_{mk}+\widetilde{\epsilon}^{t+1}_{mk}, \quad 1\leq k\leq K, k\neq l,\\
&\{\widetilde{\epsilon}^{t+1}_m\}_{t,m}\widesim{i.i.d.} \mathcal{N}(0,\widetilde{\Sigma}), \; \widetilde{\Sigma} \in \bbR^{(K-1)\times (K-1)},
    \end{split}
\end{equation*}
where $\widetilde{A}^{\LN}_{mk}=A^{\LN}_{mk}-A^{\LN}_{ml}, k\notin \{l,K\}$, 
$\widetilde{A}^{\LN}_{mK}=-A^{\LN}_{ml}$; $\widetilde{\nu}^{\LN}_{mk}=\nu^{\LN}_{mk}-\nu^{\LN}_{ml}, k\notin \{l,K\}$, $\widetilde{\nu}^{\LN}_{mK}=-\nu^{\LN}_{ml}$; $\widetilde{\epsilon}^{t+1}_m$ is transformed from $\epsilon^{t+1}_m$ through a linear full rank transformation, thus $\widetilde{\Sigma}$ is still of full rank (function of $\Sigma$). The interpretation of $\widetilde{A}^{\LN}_{mk}$ is the relative influence exerted upon \{node $m$, category $k$\} relative to \{node $m$, category $l$\}. Therefore, our model is invariant to the choice of the baseline category; this choice only affects the interpretation of parameters. It is up to the practitioner to choose the baseline depending on what they want to learn. In particular, if we choose a baseline category where the influence upon other categories is weak, then the relative influences upon other categories compared to the baseline are close to the absolute influences upon them. In Section \ref{sec:real_data} we will discuss our choice of the baseline category for each real data example.

\subsubsection{Modeling $q^t$}
We now discuss models for the event probability $q^{t+1}$ in \eqref{eq:model_disc} and study the following two cases:
(a) $q^{t+1}$ is a constant vector across $t$ which can be specified by $q\in\mathbb{R}^M$ and (b) $q^{t+1}$ depends on the past $X^t$. 

\paragraph{Constant $q^t=q$:}
This model is reasonable if we consider event rates that are constant over time 
or  multi-node compositional time series. For example, users on social media may have constant activity levels or compositional data (\eg labor/expenditure statistics) for each node (\eg state/country) are released on a regular schedule. The latter case can be thought of as a special case with $q=1$ 
\footnote{In this case, all $X^t_m$ are non-zero and constrained in the $(K-1)$-dimensional simplex $\triangle^{K-1}$, 
so for identifiability we have to take $X^{t}_{:,1:(K-1)}\in \mathbb{R}^{M\times (K-1)}$ 
instead of $X^t\in \mathbb{R}^{M\times K}$ as the covariate for predicting $X^{t+1}$, 
and thus assume $A^{\LN}\in \mathbb{R}^{M\times (K-1)\times M\times (K-1)}$. 
The problem would not be too different and the theoretical result still hold true with slight modification.}. 
   
In the case of constant $q^t$, we only estimate $A^{\LN}\in \mathbb{R}^{M\times (K-1)\times M\times K}$, and assume $\nu^{\LN}$ to be known for ease of exposition, while $q$ and the covariance matrix $\Sigma$ are unknown nuisance parameters.
 We define the estimator as the minimizer of a penalized squared error loss:
    \begin{equation}\label{eq:est_cq}
\widehat{A}^{\LN}=\argmin_{A\in \bbR^{M\times (K-1)\times M\times K}} L^{\LN}(A)+\lambda \|A\|_{R},
\end{equation}
 where 
 \begin{equation}\label{eq:gsm_loss}
 \begin{split}
     L^{\LN}(A)=&\frac{1}{2T}\sum_{t,m}\ind{X^{t+1}_m\neq 0}\|Y^{t+1}_m-\mu^{t+1}_m(A_m)\|_2^2,\\
      Y^t_{mk}=&\begin{cases}
\log(X^t_{mk}/X^t_{mK}),&X^t_m\neq 0\\
0,&X^t_m= 0
\end{cases}, \qquad 1\leq k\leq K-1\\
\mu^t_{m}(A)=&\langle A_{m},X^{t-1}\rangle+\nu^{\LN}_{m}\in \bbR^{K-1}.
 \end{split}
\end{equation}

Note that if $\Sigma=I_{K-1}$, the squared loss is exactly the negative log-likelihood loss, while for a general $\Sigma$, this loss is still applicable without knowing $\Sigma$. One may note that $q$ does not appear in the objective function. This is due to the fact that the log-likelihood can be written as summation of a function of $A$ and a function of $q$, and thus we could directly minimize an objective function that does not depend on $q$.

\paragraph{$q^t$ depends on past events:}
We model $q^{t+1}$ using the logistic link: for $1\leq m\leq M$, 
\begin{equation}\label{eq:q_A_dis}
    q_m^{t+1}=\frac{\exp\{\langle B^{\bern}_m,X^t\rangle+\eta^{\bern}_m\}}{1+\exp\{\langle B^{\bern}_m,X^t\rangle+\eta^{\bern}_m\}},
\end{equation} 
where $B^{\bern}\in \mathbb{R}^{M\times M\times K}$, and $B^{\bern}_{mm'k'}$ is the overall influence exerted on node $m$ by \{node $m'$, category $k'$\}, while $\eta^{\bern}\in \mathbb{R}^{M}$ is the offset parameter. 


If we set $B^{\bern}=0$ this reduces to the constant $q^t=q$ case with $q_m=\left(1+\exp\{-\nu^{\LN}_{m}\}\right)^{-1}$. In general, our goal is to jointly estimate $A^{\LN}$ and $B^{\bern}$, while $\nu^{\LN}$ and $\eta^{\bern}$ are assumed known for ease of exposition, and the covariance matrix $\Sigma$ is regarded as an unknown nuisance parameter. 

%


The loss function $L^{\LN}(A)$ defined in \eqref{eq:gsm_loss} can still be used to estimate $A^{\LN}$; while for $B^{\bern}$, we can define $L^{\bern}(B)$ as the log-likelihood loss of the Bernoulli distributed $\ind{X^t_m\neq 0}$:
\begin{equation}\label{eq:Ber_loss}
\begin{split}
    &L^{\bern}(B)=\frac{1}{T}\sum_{t,m} f(\langle B_m, X^t\rangle+\eta_m^{\bern}) -(\langle B_m, X^t\rangle+\eta_m^{\bern})\ind{X^{t+1}_m\neq 0},
\end{split}
\end{equation}
where $f:\bbR\rightarrow \bbR$ is defined by $f(x)=\log \left(e^{x}+1\right)$.
To exploit the sparsity structure shared by $A^{\LN}$ and $B^{\bern}$, we pool the two loss functions together and add a group sparsity penalty on $A^{\LN}$ and $B^{\bern}$. To account for various noise levels $\Sigma$, we put different weights on the two losses, and intuitively the weight on $L^{\LN}(A)$ should be smaller if $\Sigma$ is large. Formally, 
\begin{equation}\label{eq:est_q_past}
\begin{split}
    (\widehat{A}^{\LN},\widehat{B}^{\bern}&)=\argmin_{\substack{A\in \mathbb{R}^{M\times (K-1)\times M\times K}\\B\in \mathbb{R}^{M\times M\times K}}} \alpha L^{\LN}(A) +(1-\alpha)L^{\bern}(B) +\lambda R_{\alpha}(A,B).
\end{split}
\end{equation}
The penalty term $R_{\alpha}(A,B)$ is defined as $$R_{\alpha}(A,B)=\sum_{m,m'}\left(\alpha\|A_{m,:,m',:}\|_{F}^2+(1-\alpha)\|B_{m,m',:}\|_2^2\right)^{\frac{1}{2}}.$$ If we let $\alpha=0.5$, this type of estimator has been widely seen in the literature of multi-task learning \citep{zhang2017survey,obozinski2006multi,lounici2009taking}. When $\alpha=0$ or $1$, we are estimating $A^{\LN}$ or $B^{\bern}$ only and the penalty is $\lambda \|A\|_R$ or $\lambda \|B\|_R$, respectively.

\subsection{Interpreting the Relative and Absolute Network Parameters}\label{sec:network_interpretation}

So far we have defined an absolute network parameter $A^{\MN}$ for the multinomial model and a relative network parameter $A^{\LN}$ for the logistic-normal model. In this section we discuss how to interpret and connect these parameters. As we have mentioned previously, $A^{\MN}\in \bbR^{M\times K\times M\times K}$ is an {\em absolute} network parameter where each entry measures the {\em absolute} influence for each node, category pair; $A^{\LN}\in\bbR^{M\times (K-1)\times M\times K}$ is a {\em relative} network parameter whose entries measure the {\em relative} influence on each node category pair {\em relative} to the same node and a ``baseline'' category (encoded as category $K$) chosen by the practitioner. 

Note that there exists a simple transformation from the absolute network $A^{\mathrm{abs}}$ to the relative network $A^{\mathrm{rel}}$ as follows:
\begin{equation}\label{eq:network_transform}
    A^{\mathrm{rel}}_{mk}=A^{\mathrm{abs}}_{mk}-A^{\mathrm{abs}}_{mK},\quad 1\leq m\leq M, 1\leq k\leq K-1.
\end{equation}
However, the absolute network can not be determined from the relative network due to that we don't know the absolute influence upon the baseline category. 
For comparison purposes, one could contrast the estimated relative network transformed for the multinomial approach
using Eq.~\eqref{eq:network_transform} and estimated relative network for the logistic-normal approach. In the simulations and real data sections, we present three network estimates, the estimated absolute and relative networks for the multinomial model and the estimated relative network for the logistic-normal model .

\subsection{Connection to Prior Work}

In this section, we discuss connections between our model and existing approaches in the literature.

\paragraph{Connection to Point Process Literature:}
Our work is most closely related to \cite{hall2016inference}, which discusses a discrete-time modeling approach for point process data. More specifically, they investigate high-dimensional generalized linear autoregressive process:
\begin{equation}
    X^{t+1}|X^t\sim P(\nu+A^{*}X^t),
\end{equation}
where $\{X_t\in \bbR^M\}_{t=0}^T$ is the observed time series data, $\nu\in \bbR^M$ is a known offset parameter, and $A^*\in \bbR^{M\times M}$ is the network parameter of interest. \cite{hall2016inference} specify $P$ to be the product measure of independent Poisson or Bernoulli distributions. Specifically, for a Bernoulli autoregressive process the model is:
\begin{equation}\label{eq:model_BAR}
    \bbP(X^{t+1}|X^t)=\prod_{m=1}^M\frac{\exp\{(\nu_m+A_m^{*\top} X^t)X^{t+1}_m\}}{1+\exp\{\nu_m+A_m^{*\top} X^t\}}.
\end{equation}
This model ignores the context/categorical information of the events, which is what our methods aim to capture.

When there is only one exact category for each event, the multinomial model \eqref{eq:mult} can estimate the category-dependent network as a natural extension from Bernoulli autoregressive process. However, when the event presents imprecise mixed membership in multiple categories, there is no established model that can be directly applied or naturally extended for this type of data. Our logistic-normal approach \eqref{eq:model_disc}, \eqref{eq:model_cts} combines ideas from compositional time series and autoregressive process framework.

As illustrated by \cite{mark2018network}, the multivariate Hawkes process \citep{hawkes1971spectra,daley2003introduction,yang2017online} can be discretized and represented as a Poisson generlized linear ARMA model. \cite{mark2018network} consider analysis that involves a moving average term, while since the focus of this paper is mainly network influence, we only consider the autoregressive model without a moving average term.

\paragraph{Connection to Compositional Time Series:}

Compositional time series arise from the study of labor statistics \citep{brunsdon1998time}, expenditure shares \citep{mills2010forecasting} and industrial production \citep{kynvclova2015modeling}. In a classical setup, one would observe a time series 
$\{X^t\}_{t=0}^{T}$ where 
$X^t\in \bbR^K$ 
lies on a simplex 
$\triangle^{K-1}$, representing the composition of a quantity of interest (i.e. proportion belonging to each category). Directly modeling
compositional time series data is difficult because the observations are all constrained on the simplex. This challenge can be avoided by modeling the data after transforming the data via taking the log of ratios between each category and some baseline category as discussed earlier. 
In classical compositional time series analysis, we might use an ARMA model to describe the transformed data.

Our {\em logistic-normal} model is closely connected to the compositional time series models, but deviates from this classical setting in two ways.
On the one hand, even when we consider the special case where event probability $q=1$, we have a multi-variate compositional time series (one for each node in our network), and so our model reflects not only an ARMA model for each node independently, but also the autoregressive model of {\em interactions} between them. A more significant difference is that we consider the scenario where there is {\em no event} during a time period $t$ for node $m$ meaning $X^t_m= 0_{K}$ instead of lying on the simplex. This presents a significant methodological challenge as discussed earlier and we cannot simply apply the log ratio transformations to all $X^t_m$. Hence we introduce a latent variable $Z^t_m$ lying on the simplex to address this issue: we only apply the log-ratio transformation on $Z^t_m$ when modeling the conditional distribution of $Z^{t}_m$ given $X^{t-1}$, and with probability $q^t_m$ we observe $X^t_m=Z^t_m$, otherwise $X^t_m=0_K$. 

\section{Theoretical Guarantees}\label{sec:theory}
In this section we derive the estimation error bounds for the three estimators defined in Section \ref{sec:set_up_mult}, Section \ref{sec:setup_mixed}, under each corresponding set-up. We first introduce sparsity and boundedness notions that will appear in the theoretical results. 
In particular, for the multinomial model \eqref{eq:mult}, we define the following notions:
\begin{itemize}
    \item[(i)]{\bf Group sparsity parameters}: For $1\leq m\leq M$, let $S^{\MN}_m:=\{m':\|A^{\MN}_{m,:,m',:}\|_{F}>0\}$ be the set of nodes that have influence on node $m$ in any category, sparsity $\rho^{\MN}_m:=|S^{\MN}_m|$, and $\rho^{\MN}:=\max_{1\leq m\leq M}\rho^{\MN}_m$. Further let $s^{\MN}:= \sum_{m=1}^M{\rho^{\MN}_m}$. 
    \item[(ii)]{\bf Boundedness parameters}: Let $R^{\MN}_{\max}:=\|A^{\MN}\|_{\infty,\infty,1,\infty}=\max_{m,k}\sum_{m'}\max_{k'}|A^{\MN}_{mkm'k'}|$. 
\end{itemize}    
For the logistic-normal model with constant event probability (\eqref{eq:model_disc}, \eqref{eq:model_cts} with $q^t=q$), we can define $S^{\LN}_m$, $\rho^{\LN}_m$, $\rho^{\LN}$, $s^{\LN}$ and $R_{\max}^{\LN}$ similarly from above, except that we substitute $A^{\MN}$ by $A^{\LN}$. While for the logistic-normal model with event probability depending on the past (\eqref{eq:model_disc}, \eqref{eq:model_cts}, \eqref{eq:model_BAR}), we assume shared sparsity in $A^{\LN}$ and $B^{\bern}$ among nodes, and both of them need to be bounded. Thus under this model, we define $S^{\LN,\bern}_m$, $\rho^{\LN,\bern}_m$, $\rho^{\LN,\bern}$, $s^{\LN,\bern}$ and $R_{\max}^{\LN,\bern}$ similarly from above, except that we substitute $A^{\MN}$ by the concatenated tensor $(A^{\LN},B^{\bern})\in \bbR^{M\times K\times M\times K}$ (concatenated in the second dimension).
\subsection{Multinomial Model}\label{sec:thm_mult}

\begin{thm}\label{thm:mult}
Consider the generation process \eqref{eq:mult} and estimator \eqref{eq:est_mult}. If $\lambda=CK\sqrt{\frac{\log M}{T}}$, $T\geq C_1(\rho^{\MN})^2\log M$, then with probability at least $1-C\exp\{-c\log M\}$, 
\begin{equation*}
\left\|\widehat{A}^{\MN}-A^{\MN}\right\|^2_{F}\leq C_2\frac{s^{\MN}\log M}{T},\quad \left\|\widehat{A}^{\MN}-A^{\MN}\right\|_{R}\leq C_3s^{\MN}\sqrt{\frac{\log M}{T}},
\end{equation*}
where constants $c, C>0$ are universal constants, while $C_1, C_2, C_3>0$ depend only on $R^{\MN}_{\max}$, $\|\nu^{\MN}\|_{\infty}$ and $K$. 
\end{thm}
\noindent The proof can be found in Section \ref{sec:proof_mult}.

This type of estimation error bound is widely seen in the high-dimensional statistics literature (see \eg~\cite{bickel2009simultaneous,zhang2017survey}). 
As in \citep{hall2016inference} and \citep{mark2018network}, a martingale concentration inequality is applied to adapt to the time series setting, and the major difference in this proof from past work includes lower bounds on the strong convexity parameter for our multinomial loss function, and the eigenvalues of covariance matrices of multinomial random vectors. 


\subsection{Logistic-normal Model with Constant $q^t=q$}\label{sec:thm_cq}
\begin{thm}\label{thm:cq}
Consider the generation process \eqref{eq:model_disc}, \eqref{eq:model_cts} with $q^t=q$, and estimator \eqref{eq:est_cq}. If $T\geq C_1\frac{(\rho^{\LN})^2\log M}{\min_{m} q_{m}^2}$, $\lambda=CK\max_{k}\Sigma_{kk}\sqrt{\frac{\max_{m}T_{m}\log M}{T^2}}$, where $T_{m}=\sum_{t=1}^{T}\ind{X^{t}_{m}\neq 0}$, 
then with probability at least $1-C\exp\{-c\log M\}$,
\begin{equation}
    \begin{split}
        \|\widehat{A}^{\LN}-A^{\LN}\|_{F}^2 \leq C_2\frac{\max_{m}q_{m}}{\min_m q_m^2}\frac{s^{\LN}\log M}{T},\\
         \|\widehat{A}^{\LN} -A^{\LN}\|_{R} \leq C_3s^{\LN}\sqrt{\frac{\max_{m}q_{m}}{\min_m q_m^2}\frac{\log M}{T}},
    \end{split}
\end{equation}
Here $c, C>0$ are universal constants, while constants $C_1, C_2, C_3>0$ depend only on $R^{\LN}_{\max}$, $\|\nu^{\LN}\|_{\infty}$, $\|\Sigma\|_{\infty}$, $\lambda_{\min}(\Sigma)$ and $K$.
\end{thm}
\noindent The proof is provided in Section \ref{sec:proof_cq}.

The error bounds in Theorem \ref{thm:cq} have an extra factor depending on $q$. 
If $q_m=q_0$ for $1\le m \le M$ and some $0<q_0<1$, then this factor becomes $\frac{1}{q_0}$. If $q_m$'s differ too much from each other, a better choice is to use specific $\lambda_m=C\sqrt{\frac{|\mathcal{T}_{m}|\log M}{T^2}}$ for the estimation of each $A^{\LN}_m$, which would lead to a term $\frac{1}{q_m}$ instead of $\frac{\max_{m'}q_{m'}}{q_m^2}$ in the error bounds. This extra factor can be understood as follows: under the multinomial model \eqref{eq:mult}, the number of samples for estimating $A^{\MN}_m$ is $T$, while in this section, the expected number of samples is $q_mT$ for estimating $A^{\LN}_m$. 

The estimation error rates for the other two models do not depend on $\bbP(X^{t}_m\neq 0)$ ($q^t_m$ under this model) in this way, since no event at time $t$ also reveals useful information for estimating their network parameters: $P(X^t_m\neq 0|X^{t-1})$ depends on network parameters under the other two models, but is a constant under this model.

\subsection{Logistic-normal Model with $q^t$ Depending on the Past}\label{sec:thm_q_past}
\begin{thm}\label{thm:q_A}
Consider the generation process \eqref{eq:model_disc}, \eqref{eq:model_cts}, \eqref{eq:q_A_dis} and estimator \eqref{eq:est_q_past} for some $0\leq \alpha<1$.\footnote{Although Theorem \ref{thm:q_A} is only stated for $0\leq\alpha<1$, our proof also leads to the same estimation error bound if $\alpha=1$, for $T\geq C_1(\rho^{\LN,\bern})^2\log M$ instead of $T\geq \frac{C_1}{1-\alpha}(\rho^{\LN,\bern})^2\log M$.} If $T\geq \frac{C_1}{1-\alpha}(\rho^{\LN,\bern})^2\log M$, $\lambda=C_2(\alpha)K\sqrt{\frac{\log M}{T}}$
then with probability at least $1-C\exp\{-c\log M\}$,
\begin{equation*}
\begin{split}
    \alpha\|\widehat{A}^{\LN}-A^{\LN}\|_{F}^2+(1-\alpha)\|\widehat{B}^{\bern}-B^{\bern}\|_{F}^2 \leq& C_3 C_2(\alpha)\frac{s^{\LN,\bern}\log M}{T},\\
    R_{\alpha}(\widehat{A}^{\LN}-A^{\LN},\widehat{B}^{\bern}-B^{\bern})\leq &C_4 C_2(\alpha)s^{\LN,\bern}\sqrt{\frac{\log M}{T}},
\end{split}
\end{equation*}
where $c>0$ is a universal constant, $C_1, C_3, C_4>0$ depend only on $K$, $R^{\LN,\bern}_{\max}$, $\|\Sigma\|_{\infty}$, $\lambda_{\min}(\Sigma)$, $\|\nu^{\LN}\|_{\infty}$, $\|\eta^{\bern}\|_{\infty}$, and $C_2(\alpha)=\left[C_5\max_k \Sigma_{kk}\alpha+C_6(1-\alpha)\right]^{\frac{1}{2}}$ for some universal constants $C_5, C_6>0$.
\end{thm}
\noindent The proof can be found in Section \ref{sec:proof_q_A}.

When $0<\alpha<1$, the estimation errors for $A^{\LN}$ and $B^{\bern}$ are implied directly, although they may be loose in their dependence on $\alpha$. It's difficult to determine an optimal $\alpha$ for estimation based on the theoretical result. Intuitively we need $\alpha$ to be away from 0 and 1 so that we boost the estimation performance by pooling the two estimation tasks together.
We will demonstrate the interplay between $\alpha$ and the noise level $\Sigma$ in terms of estimation errors in the numerical results in Section \ref{sec:simulation_q_A}.

\section{Synthetic Data Simulation}\label{sec:numeric_exp}
In this section, we validate our approach in two ways:
First we use synthetic data generated according to the three aforementioned models to validate our theoretical results on the rates of estimation error; and then test our method(s) on data generated from a synthetic mixture model, which is a hybrid of the multinomial model \eqref{eq:mult} and the logistic-normal model with $q$ depending on the past (\eqref{eq:model_disc}, \eqref{eq:model_cts} and \eqref{eq:q_A_dis}). The latter aims to compare our methods and provide guidelines for practitioners on which approach is more suitable. For all numerical experiments, we use the standard proximal gradient descent algorithm with a group sparsity penalty \citep{wright2009sparse} to solve the optimization problem, after a reparameterization\footnote{To solve \eqref{eq:est_mult} and \eqref{eq:est_cq}, we reparameterize $\{A_m\}_{m=1}^M$ to vectors, with group size $K^2$ and $K(K-1)$ respectively. To solve \eqref{eq:est_q_past}, we reparameterize $\{(\sqrt{\alpha} A_m,\sqrt{1-\alpha}B_m)\}_{m=1}^M$ to vectors in $\bbR^{MK^2}$ with group size $K^2$. A vectorial soft-threshold method can then be used for the projection step.}.

\subsection{Estimation Error Rates}\label{sec:simulation_3models}
For each of the three generation processes defined in Section \ref{sec:model_est}, we investigate the performance of the corresponding estimators \eqref{eq:est_mult},\eqref{eq:est_cq} and \eqref{eq:est_q_past}. For all the figures in this section, the mean of 50 trials are shown and error bars are three times the standard error of the mean. 
\subsubsection{Multinomial}\label{sec:simulation_mult}
The synthetic data is generated according to \eqref{eq:mult} (initial data $\{X^0_m\}_{m=1}^M$ are i.i.d. multinomial random vectors) and $A^{\MN}$ is estimated by \eqref{eq:est_mult}. 
Under all settings, for each $m$, the $\rho_m^{\MN}=\frac{s^{\MN}}{M}$ non-zero slices $A^{\MN}_{m,:,m',:}$ are sampled uniformly from $1\leq m'\leq M$. We set $K=2$, and given that $A^{\MN}_{m,:,m',:}$ is non-zero, each of its $K^2$ entries is sampled independently from $U(-2,2)$. To ensure the same baseline event rate under the three generation processes, which is set as $0.8$, we let $\nu^{\MN}=(\log \frac{4}{K})_{M\times K}$. 
Across the experiment we use penalty parameter $\lambda=0.12\times K\sqrt{\frac{\log M}{T}}$ where $0.12$ arises from cross-validation.\footnote{Since the time series data is not exchangeable, we make a modification to the $k$-fold cross-validation. For each candidate $\lambda$, the algorithm is run on 5 subsets of the data $\{X^t\}_{t=0}^T$, each including 80\% of consecutive data points: $\{X^{t_i},\dots,X^{{t_i}+0.8T}\}, i=1,\dots, 5$, where $t_i=0.05*T*(i-1)$. The estimators learned by each subset are tested on the rest 20\% of the data, and we choose the $\lambda$ that results in lowest average log-likelihood loss 
.} The scaling of mean squared error $\|\widehat{A}^{\MN}-A^{\MN}\|_F^2$ with respect to sparsity $s^{\MN}$, dimension $M$ and sample size $T$ are shown in Figure \ref{fig:MSE_mult}. 
\begin{figure}[ht]
\centering
\includegraphics[width=0.4\linewidth]{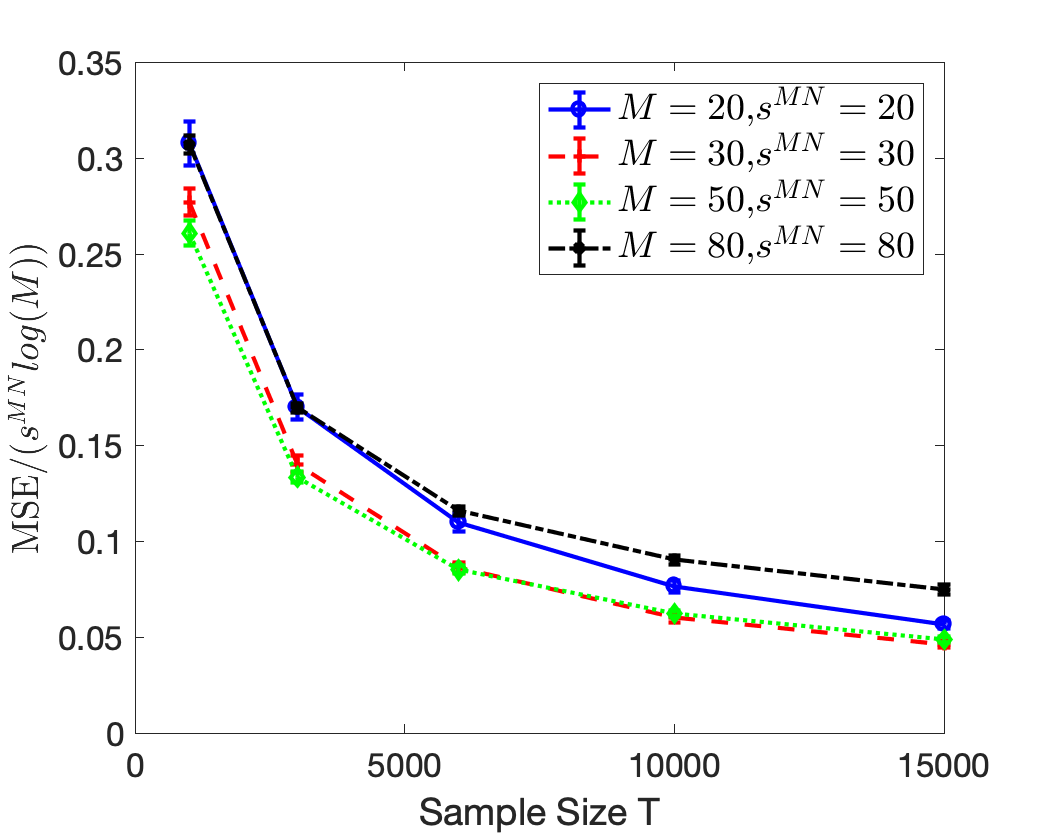}
\includegraphics[width=0.4\linewidth]{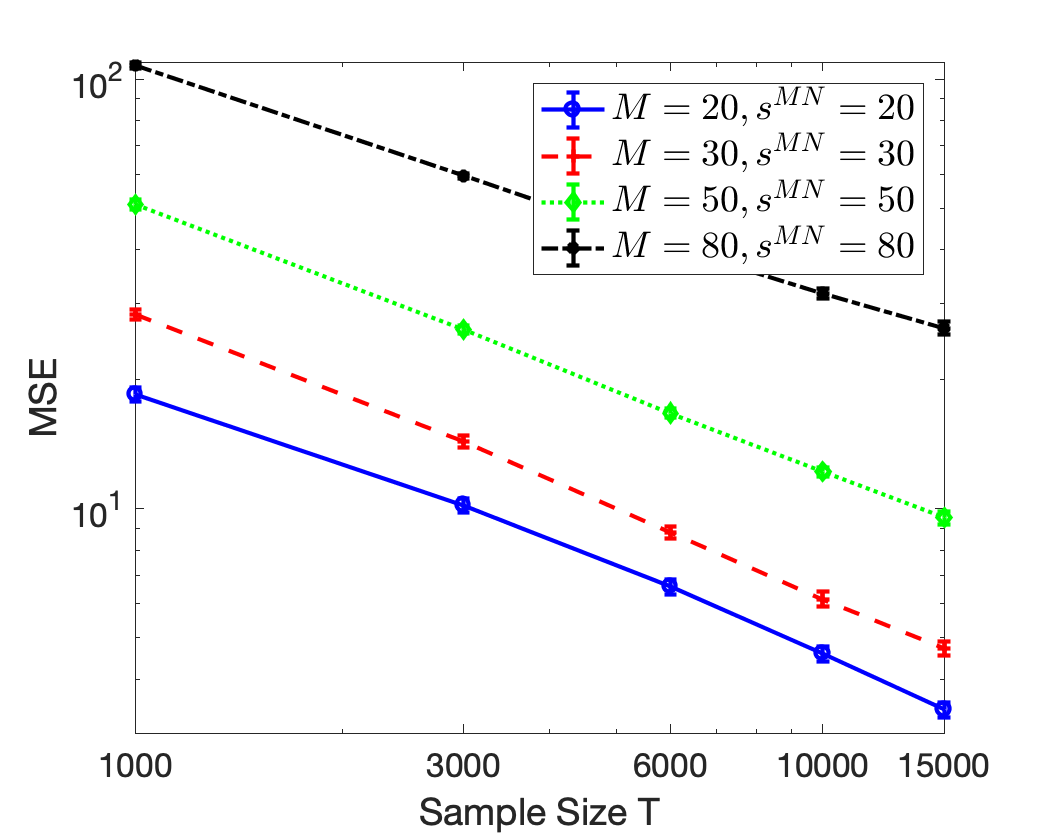}
\caption{\small MSE/$(s^{\MN}\log M)$, MSE vs. $T$ under the multinomial data generation process, and estimator \eqref{eq:est_mult}, where the second plot has a log-scale. The scaling of $\|\widehat{A}^{\MN}-A^{\MN}\|_F^2$ with respect to $s^{\MN}\log M$ is similar from the theoretical bound. Its scaling w.r.t. $T$ is a little larger than $\frac{1}{T}$, since the multinomial log-likelihood loss has a low curvature under our set-up of $A$.}
\label{fig:MSE_mult}
\end{figure}
\subsubsection{Mixed Membership with $q^t=q$}\label{sec:simulation_cq}
Here the data is generated under \eqref{eq:model_disc} (initial data $\{X^0_m\}_{m=1}^M$ are i.i.d.\ multinomial random vectors) and \eqref{eq:model_cts} with constant vector $q=(0.8)^{M\times 1}$, and the estimator is as specified in \eqref{eq:est_cq}. We set $K=2$, the covariance $\Sigma=I_{(K-1)\times (K-1)}$ and intercept $\nu^{\MM}=0^{M\times (K-1)}$. 
$A^{\LN}\in\bbR^{M\times (K-1)\times M\times K}$ is generated in the same way as in Section \ref{sec:simulation_mult} except that dimension is different. The penalty parameter $\lambda$ is set as $0.13\times K\sqrt{\frac{\log M}{T}}$ where $0.13$ arises from cross-validation. The scaling of mean squared error $\|\widehat{A}^{\LN}-A^{\LN}\|_F^2$ with respect to sparsity $s^{LN}$, dimension $M$ and sample size $T$ are shown in Figure \ref{fig:MSE_cq}. 
\begin{figure}[ht]
\centering
\includegraphics[width=0.4\linewidth]{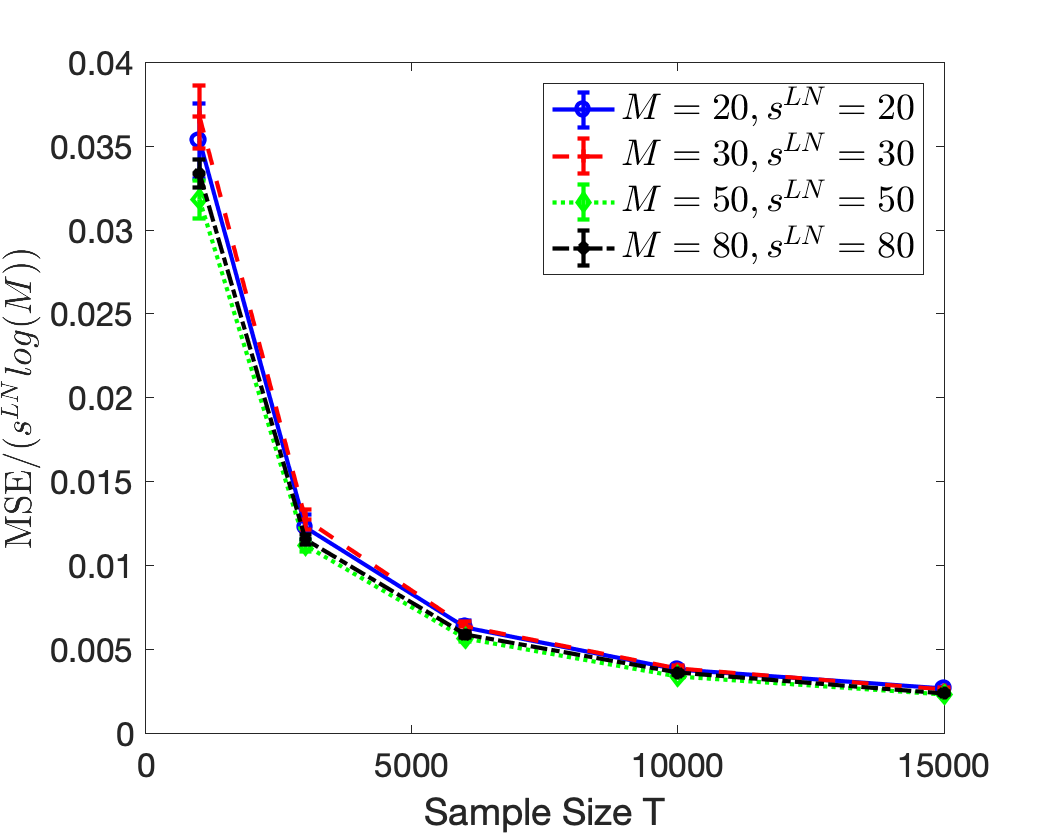}
\includegraphics[width=0.4\linewidth]{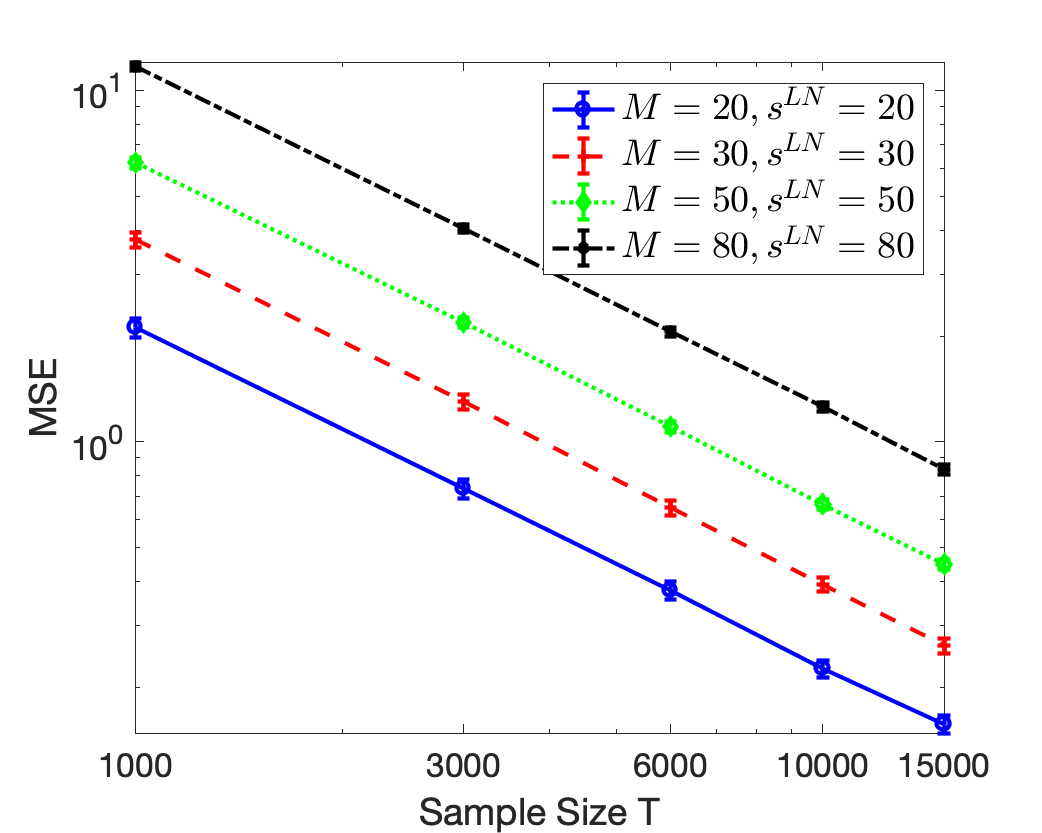}
\caption{\small MSE/$(s^{\LN}\log M)$, MSE vs. $T$ under the logistic-normal data generation process with constant $q^t$, and estimator \eqref{eq:est_cq}, where the second figure is under log-scale. The scaling of MSE aligns well with Theorem \ref{thm:cq} in $s^{\LN}, M$ and $T$.}
\label{fig:MSE_cq}
\end{figure}

\subsubsection{Mixed Membership with $q^t$ Depending on the Past}\label{sec:simulation_q_A}
We generate data according to \eqref{eq:model_disc}, \eqref{eq:model_cts} and \eqref{eq:q_A_dis} (initial data $\{X^0_m\}_{m=1}^M$ are i.i.d. multinomial random vectors), and estimate $A^{\LN}$ and $B^{\bern}$ using \eqref{eq:est_q_past}. For each $1\leq m\leq M$, we sample the support set $S_m$ uniformly from $1\leq m'\leq M$. 
Given that $A^{\LN}_{m,:,m',:}$ or $B^{\bern}_{m,m',:}$ is non-zero, each entry is sampled independently from $U(-2,2)$. We set $K=2$, the covariance $\Sigma=I_{(K-1)\times (K-1)}$, intercept $\nu^{\LN}=(0)^{M\times (K-1)}$, and $\eta^{\bern}=(\log 4)^{M\times 1}$ to ensure a base probability of $0.8$. 
The penalty parameter $\lambda=0.08\times K\sqrt{\frac{\log M}{T}}$ where $0.08$ arises from cross-validation and $\alpha=0.4$. We present the scaling of mean squared errors $\|\widehat{A}^{\LN}-A^{\LN}\|_F^2$ and $\|\widehat{B}^{\bern}-B^{\bern}\|_F^2$ in Figure \ref{fig:MSE_q_past_A} and Figure \ref{fig:MSE_q_past_B}. 
\begin{figure}[ht]
\centering
\begin{minipage}[t]{\textwidth}
\centering
\includegraphics[width=0.4\textwidth]{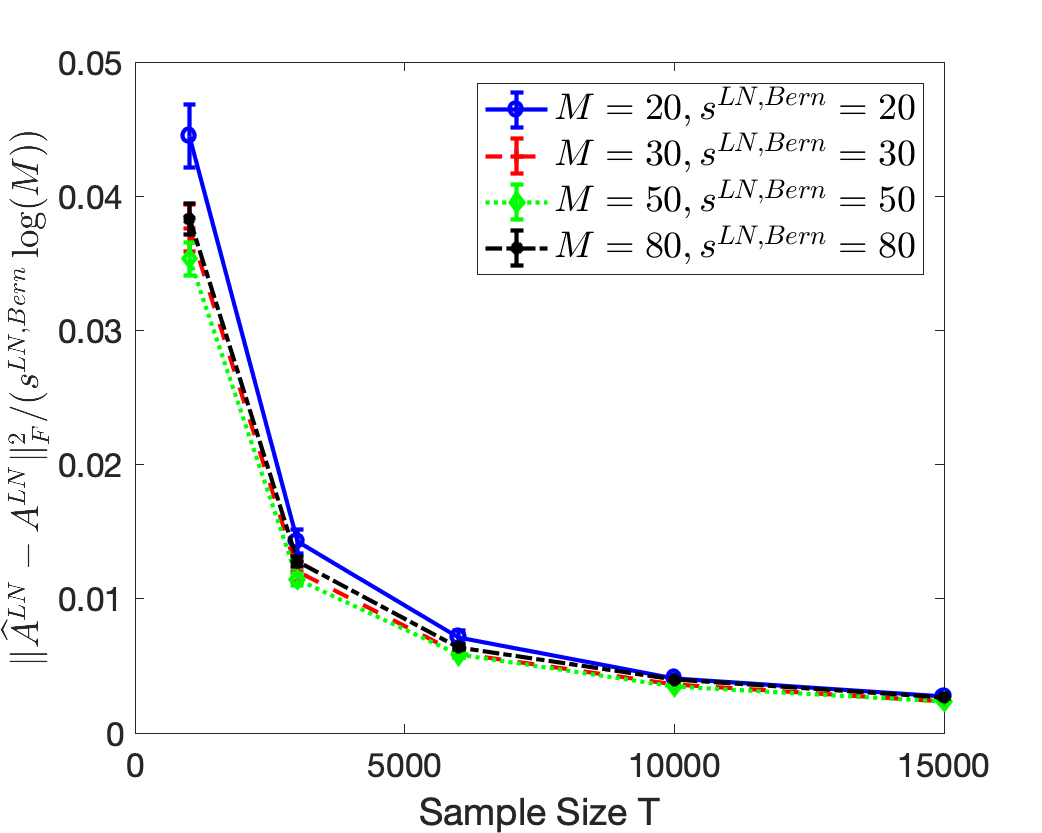}
\includegraphics[width=0.4\textwidth]{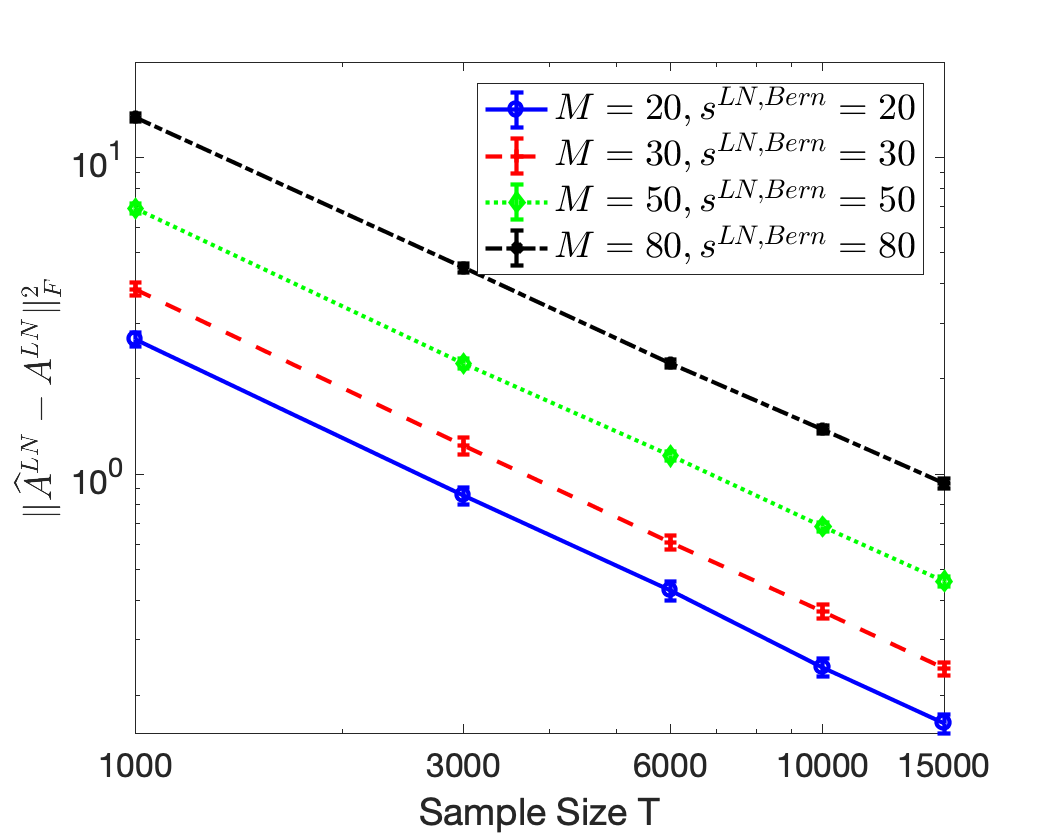}
\caption{$\frac{\|\widehat{A}^{\LN}-A^{\LN}\|_F^2}{s^{\LN, \bern}\log M},\|\widehat{A}^{\LN}-A^{\LN}\|_F^2$ vs. $T$ under the logistic-normal data generation process with $q^t$ depending on the past, and estimator \eqref{eq:est_q_past}. The second plot is under log-scale. The scaling of $\|\widehat{A}^{\LN}-A^{\LN}\|_F^2$ aligns well with Theorem \ref{thm:q_A} in $s^{\LN, \bern}, M$ and $T$.} 
\label{fig:MSE_q_past_A}
\end{minipage}
\end{figure}
\begin{figure}
    \centering
    \hfill
\begin{minipage}[t]{\textwidth}
\centering
\includegraphics[width=0.4\textwidth]{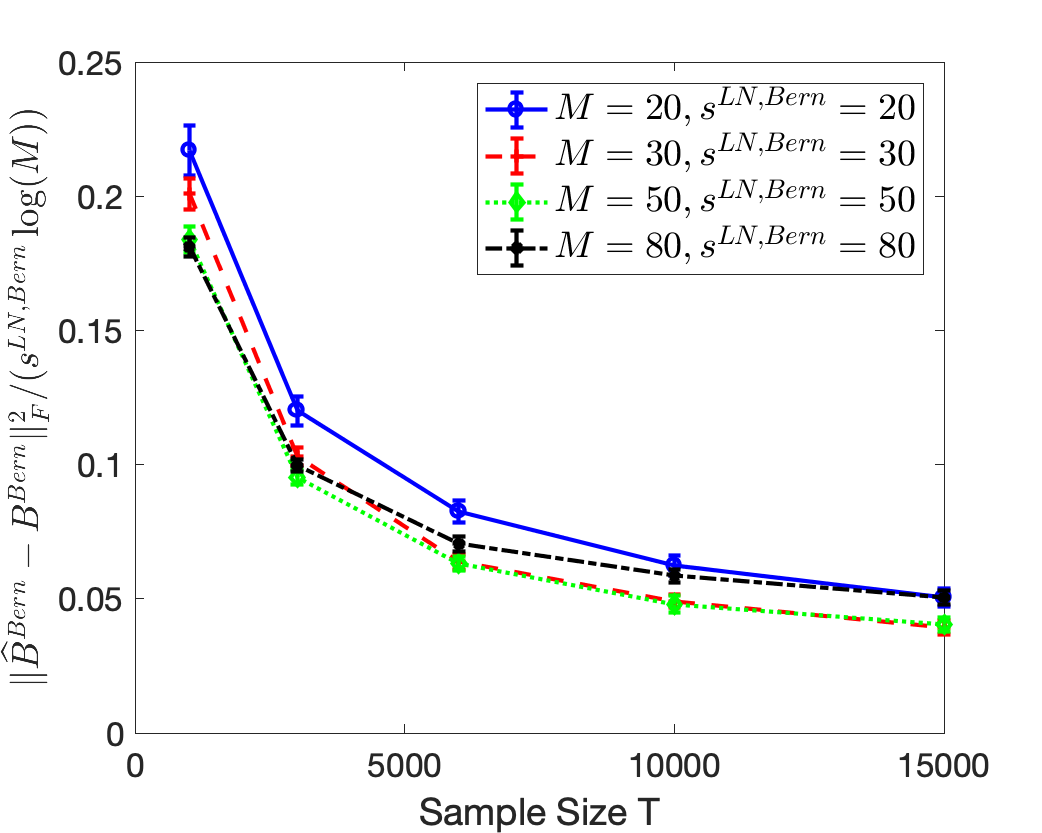}
\includegraphics[width=0.4\textwidth]{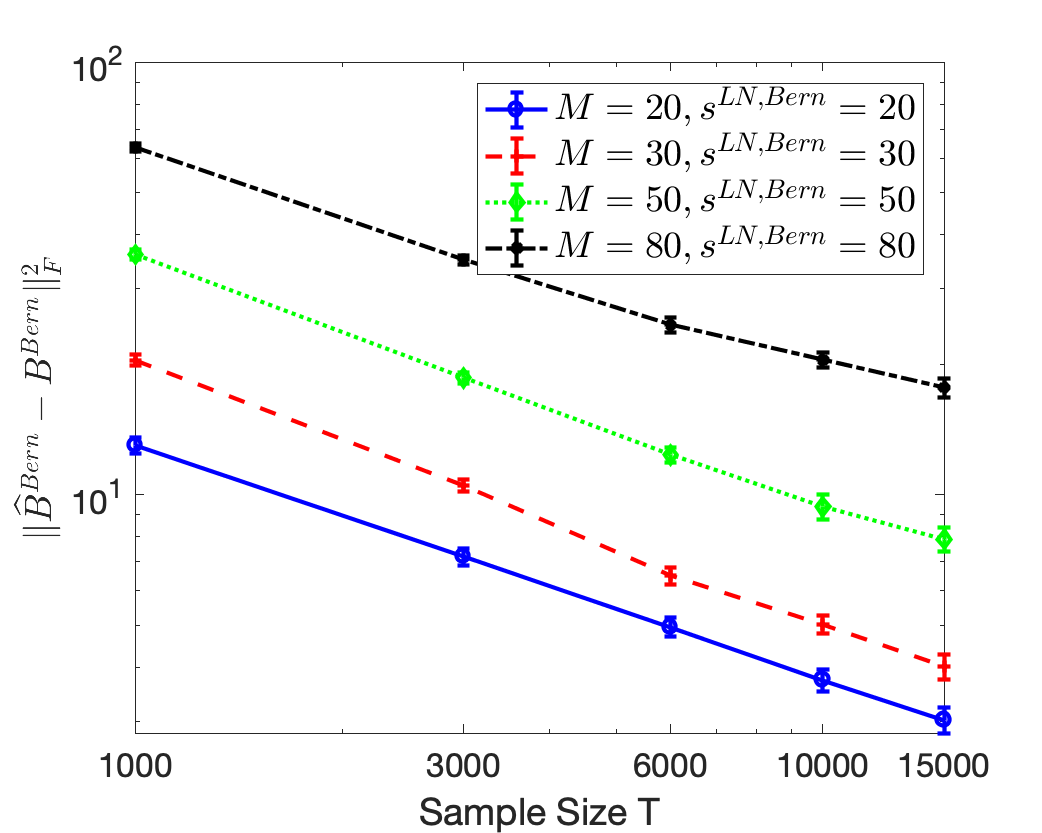}
\caption{\small $\frac{\|\widehat{B}^{\bern}-B^{\bern}\|_F^2}{s^{\LN, \bern}\log M},\|\widehat{B}^{\bern}-B^{\bern}\|_F^2$ vs. $T$ 
under the logistic-normal data generation process with $q^t$ depending on the past, and estimator \eqref{eq:est_q_past}. The scaling of $\|\widehat{B}^{\bern}-B^{\bern}\|_F^2$ w.r.t. $s^{\LN,\bern}\log M$ is similar from the theoretical bound in Theorem \ref{thm:q_A}. The second plot is under log-scale, and the scaling of $\|\widehat{B}^{\bern}-B^{\bern}\|_F^2$ w.r.t. $T$ is a little larger than $\frac{1}{T}$, since the Bernoulli log-likelihood loss has a low curvature under our set-up of $A$.}
\label{fig:MSE_q_past_B}
\end{minipage}
\end{figure}

We also check the influence of $\alpha$ on the estimation error, when the noise covariance $\Sigma$ of the logistic-normal distribution varies. We consider the setting where $M=20, s^{\LN, \bern}=20, K=2$ and $T=1000$, each non-zeros entry of $A^{\LN}, B^{\bern}$ is sampled from $U(-1,1)$. Various $\alpha$ from 0 to 1 are experimented for 20 trials, and for each trial cross-validation is used for choosing $\lambda$. We set $\Sigma=\sigma^2I_{(K-1)\times (K-1)}$ where $\sigma^2=1$ or $2$, and Figure \ref{fig:comp_alpha} shows that $\alpha$ should be smaller when $\sigma^2$ increases.
	 	 \begin{figure}[ht]
	\centering
	\begin{minipage}{\textwidth}
	\centering
		\includegraphics[width=0.4\textwidth]{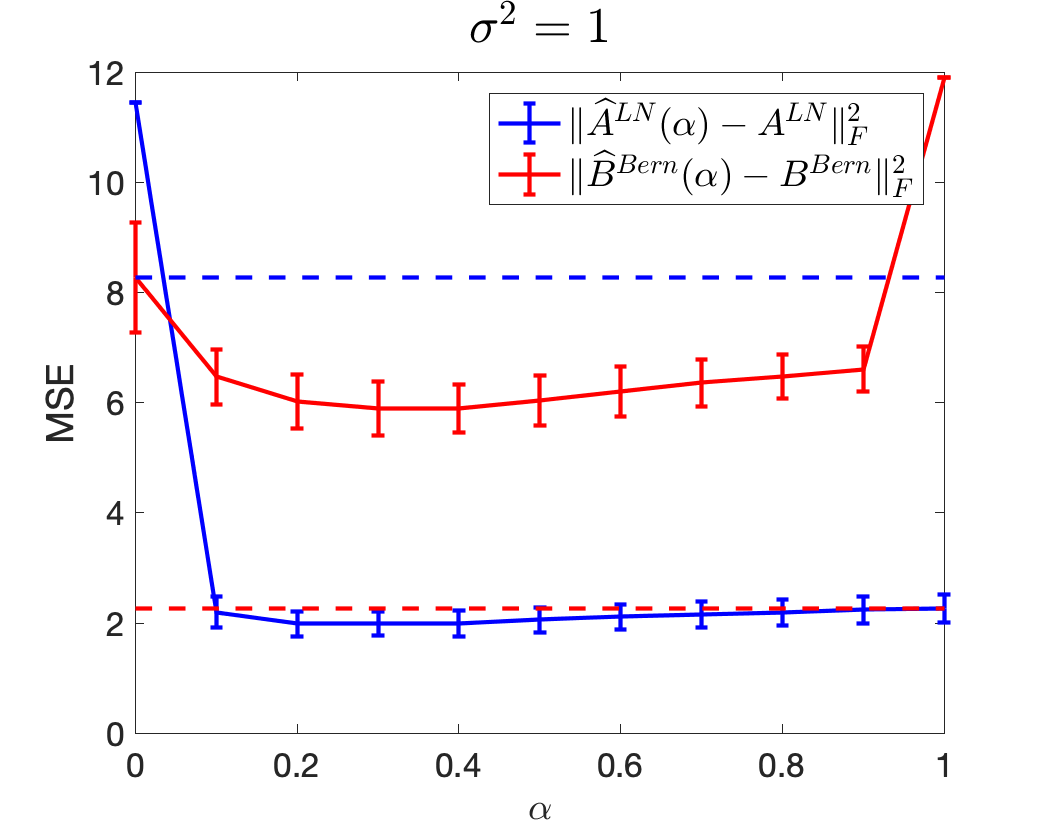}
		\includegraphics[width=0.4\textwidth]{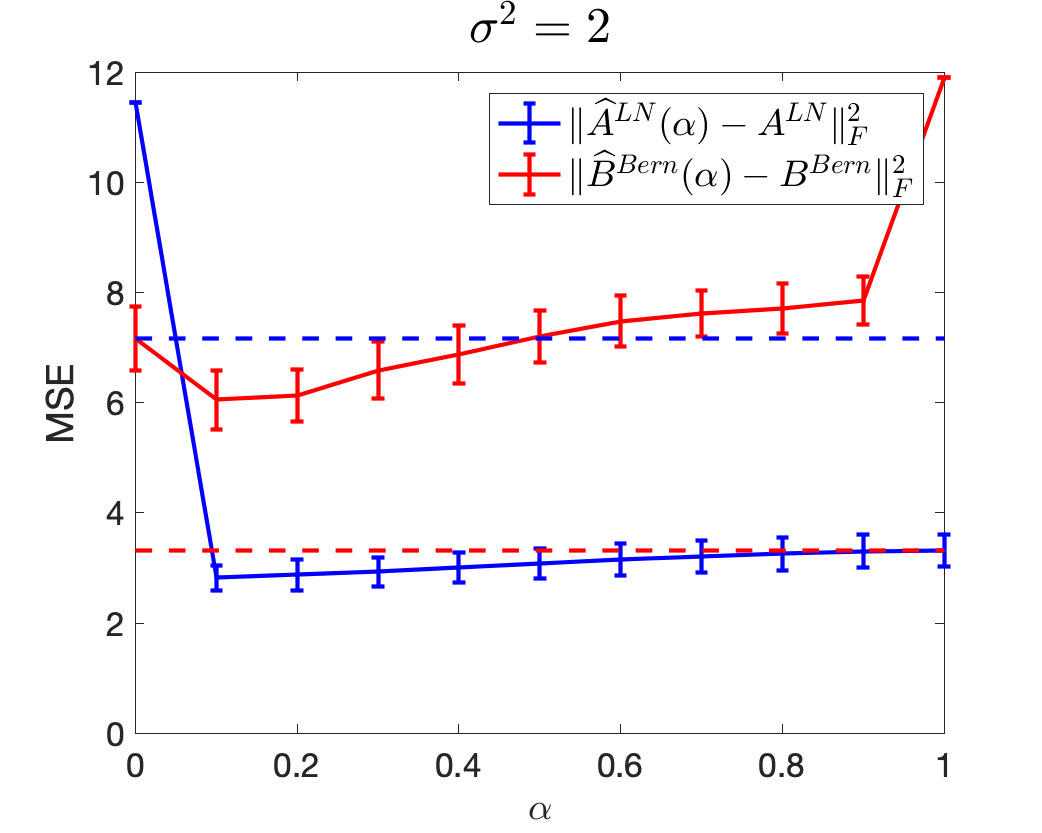}
	\caption{\small 
	The MSE of $\widehat{A}^{\LN}$ and $\widehat{B}^{\bern}$ as a function of $\alpha$. The first figure shows the results when $\sigma^2=1$, while the second one is when $\sigma^2=2$. 
	The dashed lines are $\|\widehat{A}^{\LN}(1)-A^{\LN}\|_F^2$ and $\|\widehat{B}^{\bern}(0)-B^{\bern}\|_F^2$. When $\alpha=0$ or $1$, $\widehat{A}^{\LN}$ or $\widehat{B}^{\bern}$ would stay at the initializers (set as zeros tensors); while $\|\widehat{A}^{\LN}(1)-A^{\LN}\|_F^2$, $\|\widehat{B}^{\bern}(0)-B^{\bern}\|_F^2$ would be the estimation error of separate estimations. When $\alpha$ moves from the extremes (0 or 1) to the middle, the estimation error of both are lower. When variance $\sigma^2=1$, choosing $\alpha$ around 0.4 would make $\|\widehat{A}^{\LN}(\alpha)-A^{\LN}\|_F^2$ and $\|\widehat{B}^{\bern}(\alpha)-B^{\bern}\|_F^2$ both lower than separate estimation. When $\sigma^2=2$, the figure suggests choosing smaller $\alpha$. }
	\label{fig:comp_alpha}
	\end{minipage}
\end{figure}

\subsection{Synthetic Mixture Model}\label{sec:simulation_toymodel}
The simulation study in Section \ref{sec:simulation_3models} shows that the three methods all perform well when data is generated from their corresponding generation processes. However, in reality and as we will see with our real data examples, data is unlikely to match a true model. In particular, one might expect that: (i) some nodes' events have {\em mixed memberships in different categories}, (ii) while other nodes in the network only focus on one particular category of events and thus each of their events {\em falls in one category}. This is inspired by a news media example where some media sources cover multiple topics and others focus on primarily one topic. 
We will discuss more about this phenomenon in Section \ref{sec:real_data} with the Memetracker dataset. 

In this section, we simulate a network and explore the hypothesis:
\emph{The logistic-normal approach will be more effective at estimating influences among nodes whose events exhibit mixed memberships in multiple categories; while for a node more likely to have events mainly in a single one category, the multinomial approach will be more effective.} We will validate this hypothesis both through this synthetic model and using real data in Section \ref{sec:real_data}.

\subsubsection{Set-up}

In our simulation set-up, nodes are partitioned into two sets $\mathcal{M}_1$ and $\mathcal{M}_2$, imitating the media sources that cover multiple topics and media sources focusing primarily upon one topic.
\begin{enumerate}
    \item[(i)] For each node in $\mathcal{M}_1$, the total influence it receives or exerts is the same in all categories except the baseline category, and its events are equally likely to be in those categories in the absence of outside influences. 
    Future events for these nodes depend upon the past events of neighboring nodes through the {\em logistic-normal model} with event probability $q^t$ depending on the past\footnote{The logistic-normal model with constant event probability is a special case of the model where event probability depends on the past, so it suffices to consider the latter model.} as in  \eqref{eq:model_disc}, \eqref{eq:model_cts} and \eqref{eq:q_A_dis}, so that each event has mixed membership.
    \item[(ii)] Nodes in $\mathcal{M}_2$ receive and exert influence in one category only, and its events are much more likely to fall in that category than any other category in the absence of outside influences. We refer to this category as the {\em focus} category. We model the dependence of its future events on past events of neighboring nodes through the {\em multinomial model} \eqref{eq:mult}, so that each event falls in only one category. The multinomial vectors are contaminated to be logistic-normally distributed random vectors prior to observation, since in reality we usually cannot observe exact categories of events, 
    and the logistic-normal algorithm requires each event's membership be non-zeros in all categories.
\end{enumerate}  
A more detailed explanation of the data generation process is provided in Appendix \ref{sec:dgp_toymodel}.

The true network parameters used for generating data include $\{A^{\LN}_m\in \bbR^{(K-1)\times M\times K}, B^{\bern}_m\in \bbR^{M\times K}: m\in \mathcal{M}_1\}$ and $\{A^{\MN}_m: m\in \mathcal{M}_2\}$. 
As explained in Section \ref{sec:set_up_mult} and Section \ref{sec:setup_mixed}, $A^{\LN}_m$ and $B^{\bern}_m$ encode the relative influence and overall influence exerted upon node $m$ respectively, while $A^{\MN}_m$ encodes the absolute influence exerted upon node $m$. Our detailed set-up for the network parameters are deferred to Appendix \ref{sec:dgp_toymodel}. 
We present the edges encoded by $\{A_m^{\LN}: m\in \mathcal{M}_1\}$\footnote{By our construction, $B^{\bern}_{m,m',k'}=A^{\LN}_{m,k',m',k'}$ for $k'=1,\dots,K-1$ and $B^{\bern}_{m,m',K}=0$, thus the visualization for $B^{\bern}_m$ is exactly the same as $A^{\LN}_m$.} and $\{A_m^{\MN}: m\in \mathcal{M}_2\}$ in Figure \ref{fig:toy_model_true_network}. 

There are 17 nodes ($M=17$) and 5 categories of events ($K=5$) in total: ``blue", ``black", ``red", ``green", and ``yellow" events. We set no influence in the ``yellow'' category (no yellow edge), which is used as the baseline in the logistic-normal model, so that the relative influence captured by that model is close to the absolute influence, as explained in Section \ref{sec:setup_mixed}. Purple nodes (nodes 1-5) belong to $\mathcal{M}_1$, 
		while nodes 6-17 are from $\mathcal{M}_2$. The colors of nodes 6-17 illustrate which category each node focuses on.

	\begin{figure}[ht!]
		\centering                        
		\includegraphics[width=0.3\linewidth]{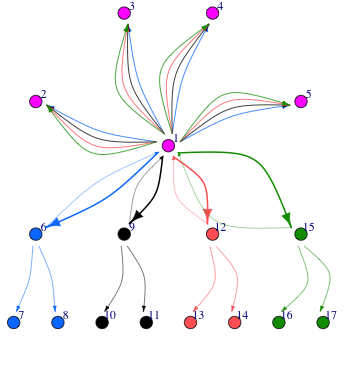}
		\caption{\small The network encoded by true parameters $\{A_m^{\LN}: m\in \mathcal{M}_1\}$ (edges pointing to nodes 1-5) and $\{A_m^{\MN}: m\in \mathcal{M}_2\}$ (edges pointing to nodes 6-17). 
		Edges pointing to nodes 1-5 are the relative influences of events in each of the first 4 categories (indicated by edge color) upon \{future events of nodes 1-5 in the same category compared to ``yellow'' category\}. Edges pointing to nodes 6-17 
		are the absolute influences of events in each category (indicated by edge color) upon future events of nodes 6-17 in the same category. Edge width is proportional to the absolute value of the corresponding influence parameter. All edges are solid, suggesting stimulatory influences. For example, edges from node 1 to node 2 suggest that the events in the first 4 categories associated with node 1 all encourage \{node 2's future events in the same category relative to the ``yellow'' category\}; 
		while the edge from node 1 to node 6 shows that ``blue" events associated with node 1 encourage ``blue" events at node 6.}\label{fig:toy_model_true_network}
	\end{figure}
	
	\subsubsection{Fitting Procedure and Estimated Networks}\label{sec:toymodel_fitting}
	After generating data $\{X^t\}_{t=0}^{T}$ according to the aforementioned procedure with $T=10000$, we obtain the estimators $\widehat{A}^{\LN}\in\bbR^{M\times (K-1)\times M\times K}$, $\widehat{B}^{\bern}\in \bbR^{M\times  M\times K}$ based on \eqref{eq:est_q_past} and $\{X^t\}_{t=0}^T$, while applying \eqref{eq:est_mult} to the rounded $\{X^t\}_{t=0}^T$ leads to the estimator $\widehat{A}^{\MN}\in \bbR^{M\times K\times M\times K}$. Here if $X^t_m=0$, its rounded version is also the zero vector; otherwise, we round $X^t_m\in \triangle^{K-1}$ to $e_k$ where $k=\arg\max_i X^t_{mi}$, and $e_k$ is the $k$th vector in the canonical basis. All tuning parameters are selected via cross-validation.\footnote{We use the same cross-validation method as that in the previous simulations, but the criterion is prediction error instead of log-likelihood loss. This is because $\alpha$ also needs to be tuned, while weighted log-likelihood loss in \eqref{eq:est_q_past} would take different forms when $\alpha$ changes. We will elaborate on the calculation of prediction errors in Section \ref{sec:real_data}.} We wish to compare the estimated networks of the two approaches with the true network so that we know which method performs better, while direct comparison is impossible for that they have different interpretations (details can be found in Section \ref{sec:network_interpretation}). 
There is a straightforward way to transform the absolute networks to relative ones, but no such transformation the other way around, as explained in Section \ref{sec:network_interpretation}. Therefore, we transform the estimated absolute network $\widehat{A}^{\MN}$ and true absolute influence $\{A^{\MN}_m: m\in \mathcal{M}_2\}$ to relative ones: $\widehat{A}_{\mathrm{rel}}^{\MN}$ and $\{A_{\mathrm{rel},m}^{\MN}, m\in \mathcal{M}_2\}$. 
We compare the estimated relative networks encoded by $\widehat{A}^{\MN}_{\mathrm{rel}}$ and by $\widehat{A}^{\LN}$ with the true relative network encoded by $\{A^{\MN}_{\mathrm{rel},m}: m\in \mathcal{M}_1\}$ and $\{A^{\LN}_m: m\in \mathcal{M}_2\}$ in Figure \ref{fig:estimated_networks_exprm2}. 
The estimated absolute network $\widehat{A}^{\MN}$ of the multinomial method is also presented to illustrate its similarity to the  transformed relative network, showing that we don't lose much information in the transformation.

	\begin{figure}[ht!]
		\centering                        
		\begin{minipage}[t]{\linewidth}
			\centering
			\subfigure[True Relative Network]{
				\includegraphics[width=0.23\linewidth]{fig/true_mixed1_overall.png}}
				\subfigure[Logistic-normal Estimate (Relative Network)]{
			\includegraphics[width=0.23\linewidth]{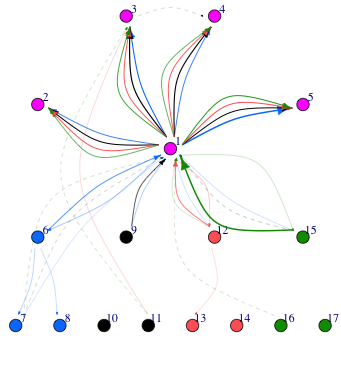}}
			\subfigure[Multinomial Estimate (Relative Network)]{
			\includegraphics[width=0.23\linewidth]{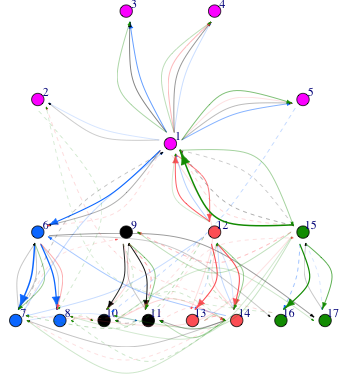}}
			\subfigure[Multinomial Estimate (Absolute Network)]{
			\includegraphics[width=0.23\linewidth]{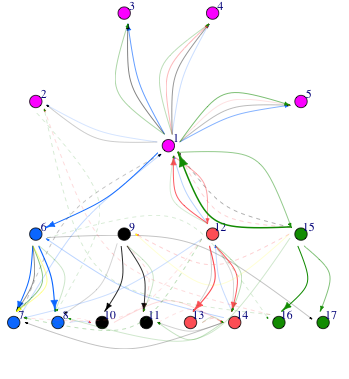}}
		\end{minipage}
		\caption{\small True relative network, and estimated networks by the multinomial and logistic-normal approaches. Solid edges are stimulatory while dashed ones are inhibitory. After normalizing the maximal absolute value of network parameters to $1$ for each network, edges whose corresponding parameters have larger absolute values than 0.1 are visualized, and edge width is proportional to that value. The absolute network estimated by the multinomial approach is the parameter this approach directly estimates, while the relative one by the multinomial approach is transformed from the absolute one. {\em We can see that the multinomial approach is more likely to underestimate the edges connecting purple nodes (nodes in $\mathcal{M}_1$), compared to the nodes 6-17 (nodes in $\mathcal{M}_2$); while the logistic-normal approach is more likely to ignore edges connecting nodes in $\mathcal{M}_2$.}}\label{fig:estimated_networks_exprm2}
	\end{figure}  
	
	We can see from the estimated networks in Figure \ref{fig:estimated_networks_exprm2} that the multinomial approach mainly picks edges correctly among nodes whose events are primarily about single categories (nodes 6-17), while the logistic-normal approach works better for nodes whose events exhibit mixed membership in multiple categories (nodes 1-5), which validates our hypothesis mentioned in the beginning of Section \ref{sec:simulation_toymodel}. This phenomenon is always true when we vary the seed for generating $\{X^t\}_{t=0}^T$, showing that it is not a result of random noise.
	
	
\section{Real Data Examples}\label{sec:real_data}
	 We validate our methodology and main hypothesis on a political tweets data set \citep{DVN/PDI7IN_2016}, and a MemeTracker data set\footnote{Data available at http://www.memetracker.org/data.html} \citep{leskovec2009meme}. These two data sets display the relative strengths and weaknesses of the multinomial and logistic-normal (event probability $q^t$ depending on the past) approaches described in Section \ref{sec:model_est} as well as advantages over existing approaches. We first elaborate on the general procedures of validating the two methods on real data sets in the following, and then discuss each example in detail (Sections \ref{sec:real_data_tweet} and \ref{sec:real_data_meme}).

	  One of the major challenges for network estimation is validation since there is no obvious ground truth. For both applications, we provide two validations: (1) prediction error performance that demonstrates the advantage of allowing influence to depend on categories; (2) a subset of directed edges are supported by external knowledge (political tweets example) or information extracted from a cascade data set (MemeTracker example) which further validates the hypothesis from the synthetic model in the previous section.
	  
	 
	 \paragraph{\bf Comparison of estimates:} Since the two approaches take different data as input (rounded data for the multinomial method and unrounded for the logistic-normal method), we also use the rounded data to measure the prediction errors for the multinomial approach and the unrounded data for the logistic-normal approach, thus they are not directly comparable. The detailed procedure for calculating prediction errors are deferred to Appendix \ref{sec:real_data_prediction}. To investigate the benefit of learning different networks for different categories, we compare the prediction errors of the two methods relative to
 (1) a context-independent network model where the influences among nodes do not depend on categories\footnote{This is equivalent to assuming a Bernoulli auto-regressive (BAR) model \citep{hall2016inference} only considering whether event occurs, and each node's events membership in categories follow the same multinomial/logistic-normal distribution over time. We use $\ell_1$ penalized MLE for estimating BAR parameter and MLE for estimating the multinomial/logistic-normal distribution parameter}
 , and (2) a constant process where the network parameters are all zeros (no influence from the past)\footnote{MLE is used for estimating the constant process parameter.}. 
	
	 
See also the note in Section~\ref{sec:network_interpretation} about comparing estimates from two different models.

\subsection{Political tweet data}\label{sec:real_data_tweet}

A central question in political science and mass communication is how politicians influence each other. Here we measure influence using the time series of their posts on Twitter.
While constructing an adjacency matrix for this network (\eg by looking at who follows whom) is a simple task, it does not reveal how the level of influences among politicians varies as a function of political tendencies of posts (\ie left-wing or right-wing). To address this challenge, we use a collection of tweets from the 2016 United States Presidential Election Tweets Data set \citep{DVN/PDI7IN_2016}, collected from Jan 1, 2016 to November 11, 2016. The collection includes $83,459$ tweets sent by 23 Twitter accounts ($M=23$): 17 presidential candidates' accounts and the House, Senate, party accounts for each party (Democrats and Republicans). We consider two categories of tweets: left-leaning and right-leaning ($K=2$), and we aim to learn the influence network among the 23 Twitter accounts that depend on the ideologies of tweets.

Due to the lack of a pre-trained NLP model for identifying political tendencies of tweets given their contents, we use the tweets from the first half of the time period (55,859 tweets from Jan 1, 2016 to June 6, 2016) to train a neural network for categorizing tweets into two political tendencies (left- and right-leaning) and apply it on the tweets from the second half of the time period. The detailed procedure for training the neural network and how we obtain the data $\{X^t\}_{t=0}^T$ is contained in Appendix \ref{sec:real_data_prep}.

Figure \ref{fig:Twitter_hist} shows the histogram of the unrounded $\{X^t_{m,2}:X^t_m\neq 0\}$, the right-leaning weights of all tweets (averaged for multiple tweets from the same user and time window). Since the sum of the left-leaning weight and right-leaning weight of any tweet equals $1$, it suffices to present only one of them.  
One important thing to note is that there are two peaks in frequency centred at 0 and 1 which suggests many clearly left-leaning tweets (0 score) or right-leaning tweets (1 score).



\begin{figure}[ht]
    \centering
    \includegraphics[width=0.4\textwidth]{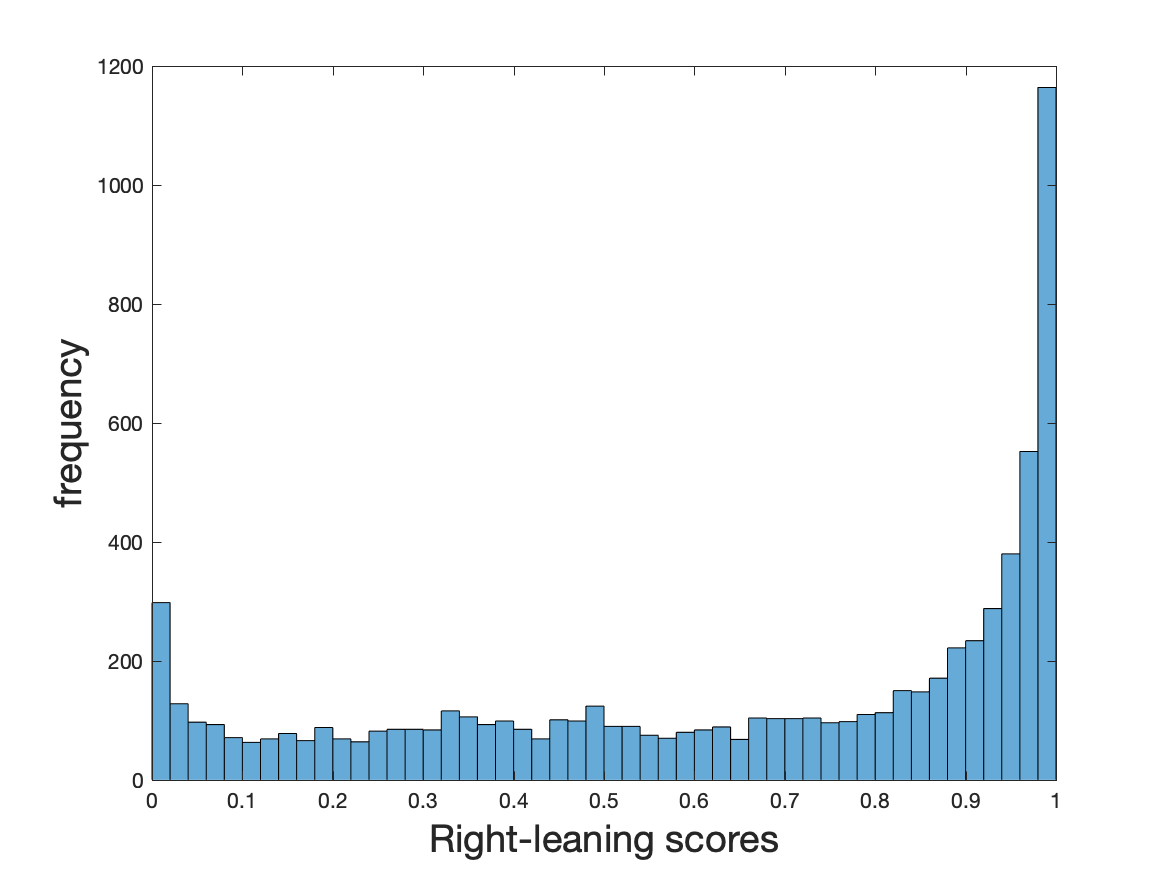}
    \caption{\small The histogram of right-leaning weights of all tweets (averaged for multiple tweets from the same user and time window) in political tweets example. The peaks in frequency at 0 and 1 suggest that the political tendencies of these tweets contain little ambiguity.}
    \label{fig:Twitter_hist}
\end{figure}

\paragraph{Prediction performance:}
We fit both models (multinomial and logistic-normal\footnote{Since there are only two categories, the baseline category for the logistic-normal model can be set arbitrarily: setting a different baseline category would only flip the sign of the relative network parameter $A^{\LN}$, while the network structure wouldn't change. We set the baseline category as ``right-leaning".}) using the first 70\% of the input data (from June 7 to September 25, 2016), and test their prediction performance on the latter 30\% (from September 26 to November 11, 2016). 
As explained before Section \ref{sec:real_data_tweet}, we calculate the prediction errors of the two fitted models and that of their corresponding fitted sub-models, which are presented in Table \ref{tab:Twitter_pred_comp_MN} and Table \ref{tab:Twitter_pred_comp_LN} respectively. 
The prediction error tables show that the multinomial approach takes advantage of the context information since our context-dependent model yields a slightly lower prediction error, but the logistic-normal approach doesn't since the context-independent approach out-performs our approach. 
\begin{table}[ht]
		\centering
		\begin{tabular}{|c|c|c|c|}
		\hline
		\multirow{2}{*}{Method}	&\multirow{2}{*}{Constant Process}&Context-independent &Multinomial\\&&Network Model&(Our Model)\\
		\hline
		Prediction Error&0.30580 &0.25520 &0.25200\\
		\hline
		\end{tabular}
	\caption{\small The prediction errors of the fitted multinomial model (full model), and that of its two sub-models: fitted constant multinomial process and context-independent network model under multinomial framework, evaluated on the hold-out data set from Sep 25, 2016 to Nov 11, 2016. We can see that the prediction error of the context-dependent network (full model) is lower than that of the context-independent one, which {\em illustrates the benefit of incorporating context information when using the multinomial method.}}\label{tab:Twitter_pred_comp_MN}
	\end{table}
\begin{table}[ht]
		\centering
		\begin{tabular}{|c|c|c|c|}
		\hline
		\multirow{2}{*}{Method}	&\multirow{2}{*}{Constant Process}&Context-independent &Logistic-normal\\& &Network Model& (Our Model)\\
		\hline
		Prediction Error&0.15800&0.14373&0.14442\\
		\hline
		\end{tabular}
	\caption{\small The prediction errors of the fitted logistic-normal model (full model), and that of its two sub-models: fitted constant logistic-normal process and context-independent network model under logistic-normal modeling framework, evaluated on the hold-out data set from Sep 25, 2016 to Nov 11, 2016. The prediction error of the fitted logistic-normal model (full model) is slightly larger than that of context-independent network model, suggesting that {\em logistic-normal approach does not capture the contextual information well}. 
	}\label{tab:Twitter_pred_comp_LN}
	\end{table}

\paragraph{Network estimates:}
After fitting the two models on the whole data set, with the same tuning parameters used in the prediction task, we present the estimated networks for both methods. Although there is no notion of ground truth, we treat the following plausible hypothesis as external knowledge: Republicans' right-leaning tweets tend to have more influence than their left-leaning tweets, encouraging other Republicans' right-leaning tweets and vice versa for Democrats and their' left-leaning tweets. As explained in Section \ref{sec:network_interpretation}, we present the absolute network estimated by the multinomial approach in Figure \ref{fig:Twitter_network_abs}, the relative networks by the multinomial and logistic-normal approaches in Figure \ref{fig:Twitter_network_rel}. 

The largest absolute entry of each of the three network parameters is normalized to one and each visualized edge width is proportional to the normalized absolute value of its corresponding parameter. For clarity, only the edges with absolute parameters larger than 0.5 are shown for each network, and blue nodes are Democrats, red nodes are Republicans. Solid edges are positive influences (stimulatory) while dashed edges are negative influences (inhibitory).\footnote{Note that for relative networks, solid edges represent positive influence on left-leaning tweets compared to right-leaning ones, while dashed edges encourage right-leaning ones.} 
As we can see from the networks in Figure \ref{fig:Twitter_network_abs} and Figure \ref{fig:Twitter_network_rel}, the edges estimated by the multinomial approach align better with our external knowledge than those estimated by the logistic-normal approach. 

The network estimates together with the prediction performance suggest that the multinomial approach works well and better than the logistic-normal approach in this example. Note that all Twitter users here have clear political tendencies, and most tweets tend to only have exactly one ideology, as shown by the histogram in Figure \ref{fig:Twitter_hist}. Since each nodes tweets tend to belong clearly to one category, the better performance of the multinomial approach is consistent with our hypothesis from the previous section.

\begin{figure}[ht!]
\centering
	\subfigure[left $\rightarrow$ left]{
	\includegraphics[width=.3\linewidth]{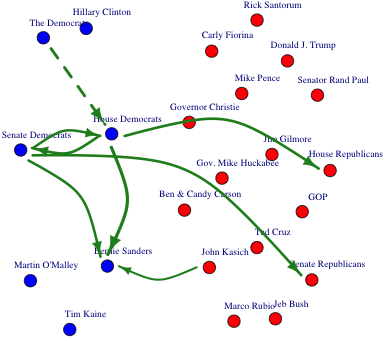}}
	\subfigure[left $\rightarrow$ right]{
	\includegraphics[width=.3\linewidth]{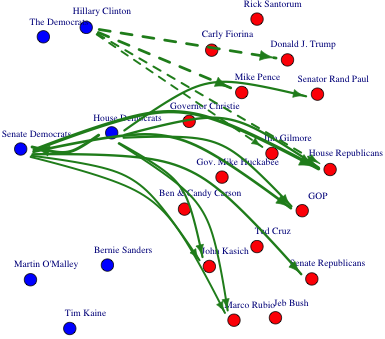}}
\subfigure[right $\rightarrow$ left]{
\includegraphics[width=.3\linewidth]{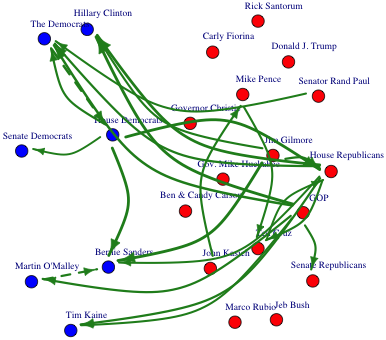}}
\subfigure[right $\rightarrow$ right]{
\includegraphics[width=.3\linewidth]{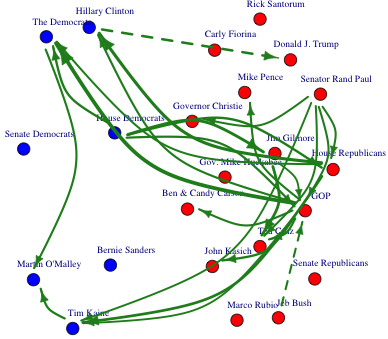}}
\caption{\small {\bf (Multinomial approach)} The absolute influence network among politicians on Twitter during 2016 presidential debates, estimated by the multinomial approach. We can see that {\em the partisanship of source users and target users of the edges align well with the categories of the sub-networks}. For example, Figure (a) and (b) suggest that left-leaning tweets sent by Democrats are more likely to trigger Democrats in sending left-leaning tweets and Republicans in right-leaning tweets. There are also more edges sent to Republicans in Figure (d) than Figure (c).}\label{fig:Twitter_network_abs}
\end{figure}

\begin{figure}[ht!]
\centering
	\subfigure[Multinomial: left $\rightarrow \frac{\mathrm{left}}{\mathrm{right}}$]{
	\includegraphics[width=.3\linewidth]{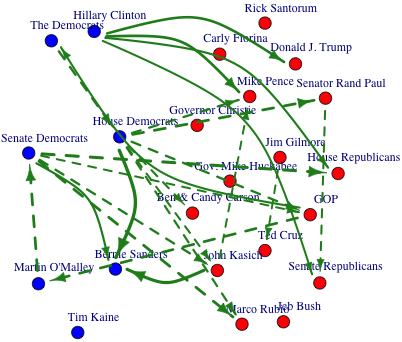}}
	\subfigure[Multinomial: right $\rightarrow \frac{\mathrm{left}}{\mathrm{right}}$]{
	\includegraphics[width=.3\linewidth]{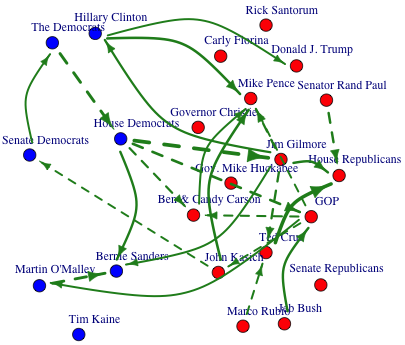}}
	\subfigure[Logistic-normal: left $\rightarrow \frac{\mathrm{left}}{\mathrm{right}}$]{
	\includegraphics[width=.3\linewidth]{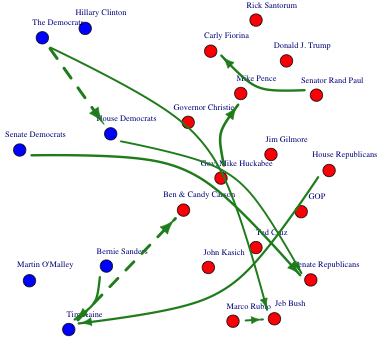}}
	\subfigure[Logistic-normal: right $\rightarrow \frac{\mathrm{left}}{\mathrm{right}}$]{
	\includegraphics[width=.3\linewidth]{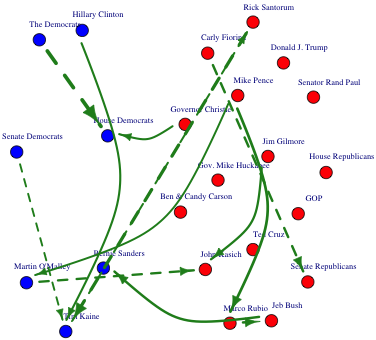}}
\caption{\small The relative influence networks among politicians on Twitter during 2016 presidential debates, estimated by the multinomial and logistic-normal approaches. Each edge in (a) and (c) represents the relative influence of the source user's left-leaning tweets upon \{the target user's left-leaning tweets compared to right-leaning ones\}, while those in (b) and (d) are relative influences of source users' right leaning tweets upon \{target users' left-leaning tweets compared to right-leaning ones\}. 
Solid edges suggest positive relative influences and thus encourage the future tweets sent by target users to be left-leaning, while dashed ones encourage them to be right-leaning. We can see more edges sent by Republicans in (b) than in (a), which suggests that Republicans have stronger influence when they post right-leaning tweets. Most dashed edges (encouraging right-leaning tweets) in (a) and (b) are sent to Republicans, which also shows an {\em alignment between categories of edges and partisanship of users in the estimated network by the multinomial approach}. As a comparison, {\em these patterns are not clear in the estimated network by the logistic-normal approach}, shown in Figure (c) and (d).}\label{fig:Twitter_network_rel}
\end{figure}

\subsection{MemeTracker Data Set}\label{sec:real_data_meme}

In this section we consider the question of how past posts sent by one online media source influence another media source in posting new articles, and how this influence network depends on the topics of articles. To answer this question, we apply our methods on the ``Raw phrases data'' in the MemeTracker data set \citep{leskovec2009meme}. This data set consists of news stories and blog posts from 1 million online sources (including mass media sources and personal blogs) over the time period from August 2008 to April 2009. For each news or blog item, only its phrases/quotes that have variants occurring frequently across the entire online news corpus are recorded in the data set, and we use them as the approximate content of the post. 

First note that most news media sources cover multiple topics (although not with the same amount of coverage), so we don't have labels for each news article and thus cannot use supervised learning like we did for the Twitter example to obtain the membership vectors as the political tweets example. Instead we use topic modeling (Latent Dirichlet Allocation proposed in \cite{blei2003latent}) for extracting mixed membership vectors, and we set the number of topics as $K=5$. Based on the top key words generated from topic modeling for each topic (shown in Table \ref{tab:kw_5topics} in Appendix \ref{sec:real_data_prep}), we choose the topic names as ``Sports'', ``International Affairs'', ``Lifestyle'', ``Finance'' and ``Health''. For simplicity and interpretability, we also filter out $M=58$ media sources based on their languages, frequencies, etc.. The detailed pre-processing of the data (how we obtain $\{X^t\}_{t=0}^T$ is contained in Appendix \ref{sec:real_data_prep}.

	\paragraph{Prediction performance:}
	We fit both models (multinomial and logistic-normal) using the first 70\% of the data (from September 1st, 2008 to February 16th, 2009), and test their prediction performance on the latter 30\% (from February 17 to April 30, 2009). 
	We choose the baseline topic as ``Health" for the logistic-normal approach since we believe the influence upon it should be weak, and thus the relative influence captured by the logistic-normal model is close to absolute influence. Detailed reasoning is contained in Appendix  \ref{sec:meme_baseline}. 
	
As explained before Section \ref{sec:real_data_tweet}, we calculate the prediction errors of the two fitted models and that of their corresponding fitted sub-models, which are presented in Table \ref{tab:meme_pred_comp_MN} and Table \ref{tab:meme_pred_comp_LN} respectively. Both the multinomial and logistic-normal approaches demonstrate the advantage of estimating context-dependent networks since the context-dependent network gives lower prediction error in both cases.
\begin{table}[ht]
		\centering
		\begin{tabular}{|c|c|c|c|}
		\hline
		\multirow{2}{*}{Method}	&\multirow{2}{*}{Constant Process} &Context-independent &Multinomial \\&&Network Model&(Our Model)\\
		\hline
		Prediction Error&0.49741&0.45062&0.43351\\
		\hline
		\end{tabular}
	\caption{\small The prediction errors of the fitted multinomial model (full model), and that of its two sub-models: fitted constant multinomial process and context-independent network model under multinomial framework, evaluated on latter 30\% of the data set. The prediction error of the full model is lower than the error of the context-independent one, {\em showing the benefit of incorporating context information using the multinomial approach}.}\label{tab:meme_pred_comp_MN}
	\end{table}


	
	\begin{table}[ht]
		\centering
		\begin{tabular}{|c|c|c|c|}
		\hline
		\multirow{2}{*}{Method}	&\multirow{2}{*}{Constant Process}&Context-independent &Logistic-normal\\&&Network Model&(Our Model)\\
		\hline
		Prediction Error&0.11269&0.10809&0.10229\\
		\hline
		\end{tabular}
	\caption{\small The prediction errors of the fitted logistic-normal model (full model), and that of its two sub-models: fitted constant logistic-normal process and context-independent network model under logistic-normal framework, evaluated on the latter 30\% of the data set. The prediction error of the full model is smaller than that of the context-independent one, {\em showing the benefit of incorporating context information using the logistic-normal approach}.}\label{tab:meme_pred_comp_LN}
	\end{table}
	
	\paragraph{Network estimates:}
	\begin{figure}[ht!]
	\centering
	\begin{minipage}[t]{\linewidth}
		\centering
		\subfigure[Logistic-normal estimate:\newline relative sub-network with\newline baseline topic ``Health"]{
			\includegraphics[width=0.3\linewidth]{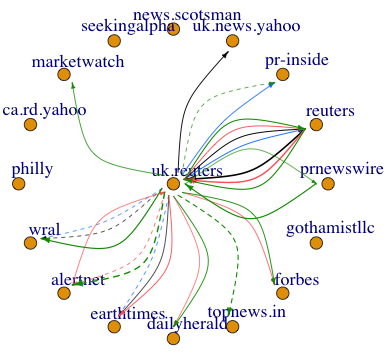}}
		\subfigure[Multinomial estimate:\newline relative sub-network with\newline baseline topic ``Health"]{
			\includegraphics[width=0.3\linewidth]{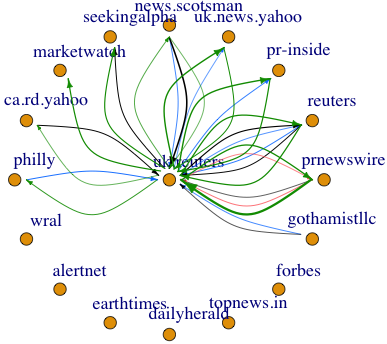}}
		\subfigure[Multinomial estimate:\newline absolute sub-network]{
			\includegraphics[width=0.3\linewidth]{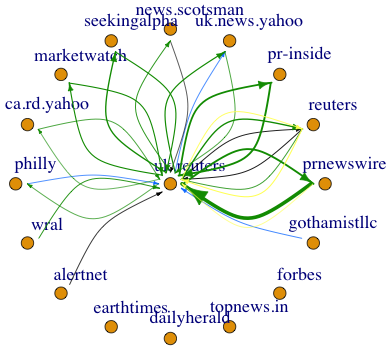}}
		\subfigure[Topic weights distribution\newline of {\em uk.reuters}]{
			\includegraphics[width=0.3\linewidth]{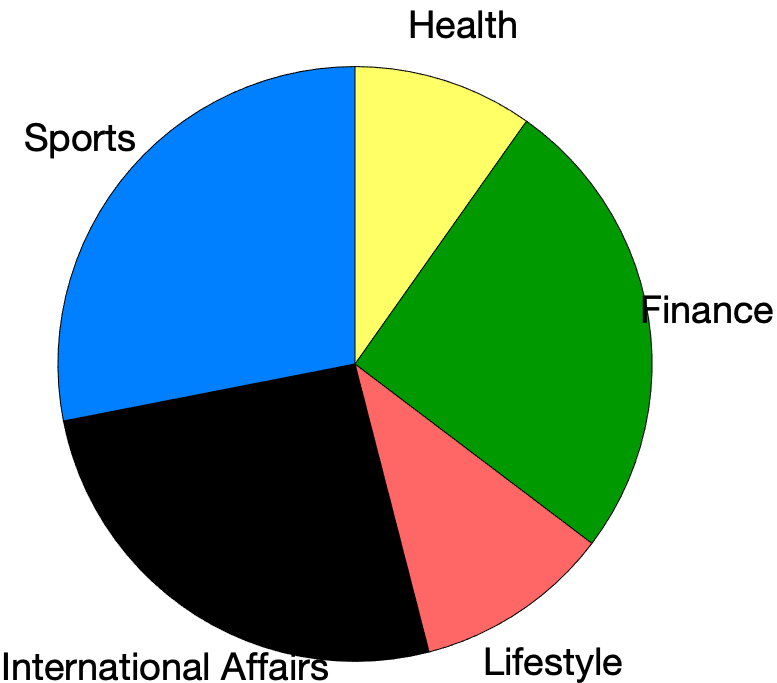}}
	\end{minipage}
	\caption{\small Three estimated neighborhoods around \emph{uk.reuters} by the two approaches and its weights distribution in 5 topics. Edge colors correspond to the topic colors in the pie charts; solid edges represent stimulatory influences while dashed ones are inhibitory. Edge widths are proportional to the absolute values of corresponding parameters after the maximal absolute entry of the network parameters is normalized to $1$. There are more red edges (influence in ``Lifestyle'') in the estimated relative sub-network by the logistic-normal approach (Figure (a)) than the estimated sub-networks by the multinomial approach (Figure (b) and (c)).}\label{fig:meme_network_ukreuters}
\end{figure}
\begin{figure}[ht!]
	\centering
	\begin{minipage}{\linewidth}
		\centering
		\subfigure[Logistic-normal estimate:\newline relative sub-network with\newline baseline topic ``Health"]{
			\includegraphics[width=0.3\linewidth]{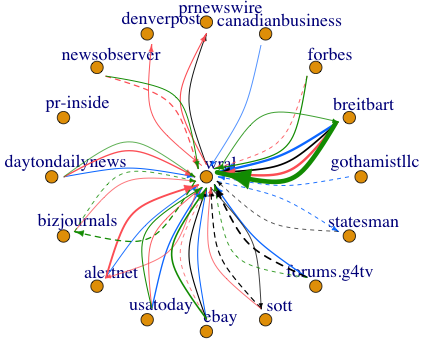}}
		\subfigure[Multinomial estimate:\newline relative sub-network with\newline baseline topic ``Health"]{
			\includegraphics[width=0.3\linewidth]{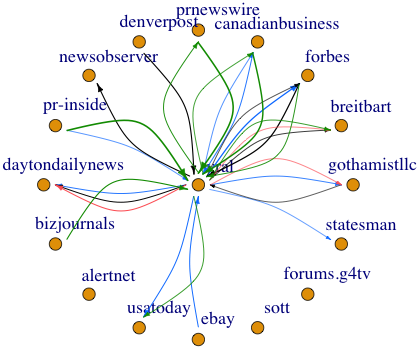}}
		\subfigure[Multinomial estimate:\newline absolute sub-network]{
			\includegraphics[width=0.3\linewidth]{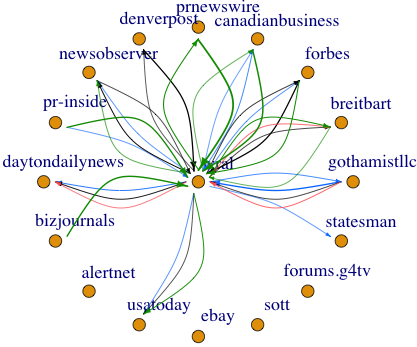}}
		\subfigure[Topic weights distribution\newline of {\em wral}]{
			\includegraphics[width=0.3\linewidth]{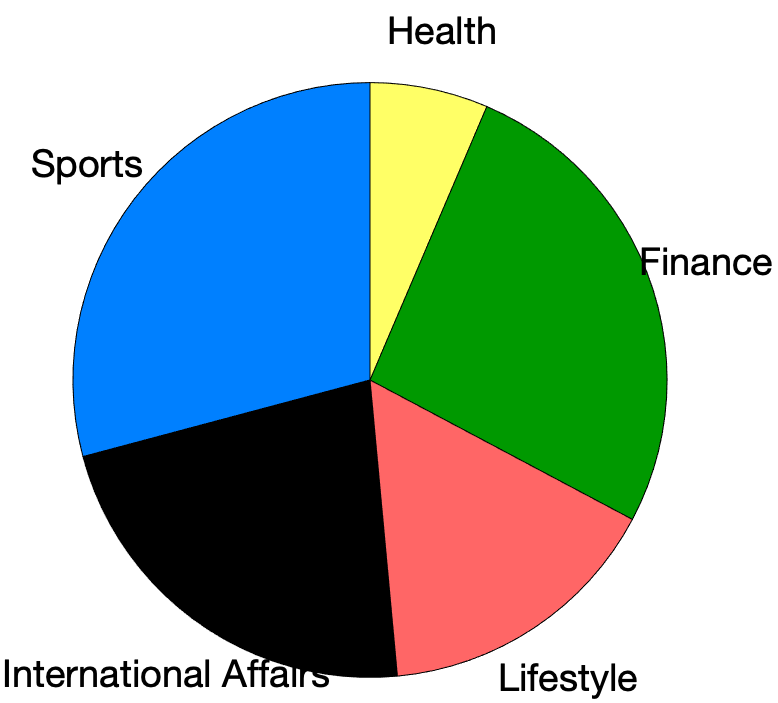}}
	\end{minipage}
	\caption{\small Three estimated neighborhoods around \emph{wral} by the two approaches and its weights distribution in 5 topics. Edge colors correspond to the topic colors in the pie charts; solid edges represent stimulatory influences while dashed ones are inhibitory. Edge widths are proportional to the absolute values of corresponding parameters after the maximal absolute entry of the network parameters is normalized to $1$. The neighbor media source ``rover.ebay" is abbreviated to ``ebay" to save space. There are more red edges (influence in ``Lifestyle'') in the estimated relative sub-network by the logistic-normal approach than the estimated sub-networks by the multinomial approach.}\label{fig:meme_network_wral}
\end{figure}
	We apply both the multinomial and logistic-normal approaches on the whole data set, with the same tuning parameters as those used in the prediction task. For simplicity, we present the neighborhood estimates around each media source, instead of the whole network estimates among 58 media sources. Figures \ref{fig:meme_network_ukreuters} and \ref{fig:meme_network_wral} present the three estimated sub-networks around \emph{uk.reuters} and \emph{wral},
	respectively, accompanied by the pie charts for the two central media sources' topic weights distribution among 5 topics. In each sub-network, we include the central media source's top 10 neighbors in any of the three network estimates. More details about the construction of the visualizations are contained in Appendix \ref{sec:meme_vis}. 
	
	As discussed in the beginning of Section \ref{sec:real_data_meme} and also seen from the pie charts in Figures \ref{fig:meme_network_ukreuters} and \ref{fig:meme_network_wral}, some media sources post on multiple topics. This is different from the political tweets example, where each Twitter user has exactly one ideological tendency that is known to us. Therefore, the validation used in the tweets example is not applicable for this example. Instead, we first comment on a general difference between the network estimates for the two approaches, and then validate some particular edges based on a cascade data set.
	
	\paragraph{General difference between the network estimates:} We can see from Figures \ref{fig:meme_network_ukreuters} and \ref{fig:meme_network_wral} that the logistic-normal approach estimates more red edges (influence in ''Lifestyle'') than the multinomial approach. In fact, the sections of the media sources' websites suggest that most media sources that post on multiple topics usually cover the topic ``Lifestyle'' (\eg {\em dailyherald}, {\em reuters}) while media sources focusing only on one topic seldom posts on ``Lifestyle'' (\eg {\em prnewswire}, {\em marketwatch}). For the first type of media sources, the logistic-normal approach may be more accurate since it captures the influences in ``Lifestyle'', while the multinomial approach may be more accurate for the latter kind of media sources. Neighborhood estimates around other media sources also show similar patterns, although not presented here. This supports our main hypothesis in Section \ref{sec:simulation_toymodel}. 

\paragraph{Phrase cluster data validation for edges:} We present supporting evidence based on a cascade data set (details provided shortly), suggesting that one method may do better than the other for the following 4 edges in Figures \ref{fig:meme_network_ukreuters} and \ref{fig:meme_network_wral}: {\em uk.reuters}$\rightarrow${\em reuters}, {\em breitbart}$\rightarrow${\em wral}, {\em canadianbusiness}$\rightarrow${\em wral} and {\em bizjournals}$\rightarrow${\em wral}. We first summarize the estimation results for the 4 edges in Table \ref{tab:meme_validation}. 
\begin{table}[ht!]
\centering
	\begin{tabular}{|c|c|c|c|}
		\hline
		\multirow{2}{*}{\backslashbox{Edges}{Networks}}&{\bf LN} relative&{\bf MN} relative&{\bf MN} absolute\\
		&network&network&network\\
		\hline
		\emph{uk.reuters}&\multirow{2}{*}{\bf S, I, L, F}&\multirow{2}{*}{S, I, L}&\multirow{2}{*}{I, F}\\
	    $\rightarrow$ \emph{reuters}&&&\\
		\hline
		\emph{breitbart}&\multirow{2}{*}{\bf S, I, L, F}&\multirow{2}{*}{No edge}&\multirow{2}{*}{F}\\
	    $\rightarrow$ \emph{wral}&&&\\
		\hline
		\emph{canadianbusiness}&\multirow{2}{*}{S}&\multirow{2}{*}{{\bf F} and S}&\multirow{2}{*}{{\bf F} and S}\\
	    $\rightarrow$ \emph{wral}&&&\\
		\hline
		\emph{bizjournals}&\multirow{2}{*}{L}&\multirow{2}{*}{\bf F}&\multirow{2}{*}{\bf F}\\
		$\rightarrow$ \emph{wral}&&&\\
		\hline
	\end{tabular}
	\caption{\small {\bf Edge topics} suggested by the estimated networks (column 2-4) in Figures \ref{fig:meme_network_ukreuters} and \ref{fig:meme_network_wral} for edges in column 1. Here we use ``S'', ``I'', ``L'' and ``F'' as abbreviations for the topics ``Sports'', ``International Affairs'', ``Lifestyle'' and ``Finance''. In the first row of the table, ``LN'' refers to the logistic-normal approach while ``MN'' refers to the multinomial approach. Our supporting evidence will validate the estimated edge topics marked in bold. We can see that {\em the logistic-normal approach works better for the first two edges, while the multinomial approach works better for the latter two}.}\label{tab:meme_validation}
\end{table}

Now we elaborate on our validation procedure. 
To validate our estimated edges we exploit a cascade data set: the \emph{``Phrase cluster data"} from August 2008 to January 2009 in the MemeTracker data set, which is also used in \cite{yu2017estimation} for studying influences among media sources. In contrast to the ``Raw phrases data'' used for our network estimation, where original phrases are recorded for each post, the ``Phrase cluster data'' collects phrase clusters consisting of variants of the same phrases, and for each phrase cluster there are records of which media source posts variants in it and when. 

For convenience, in the following we say that a media source posts a phrase cluster if it posts a phrase in that cluster. 
For each phrase cluster and any pair of influencer ($m$) and receiver ($m'$) media sources, if the first time $m'$ posts the phrase cluster is within an hour after $m$ posts it, we refer to it as an \emph{influence-involved phrase cluster} from $m$ to $m'$. Here we set the time limit as an hour since 1-hour discretization is used in the estimation task. In order to demonstrate the topics of these phrase clusters, we combine all the influence-involved phrase clusters from $m$ to $m'$ into one ``document'' and generate a word cloud and topics weights for the document. To assign topic weights, we apply the previously trained topic model (mentioned in the beginning of Section \ref{sec:real_data_meme}) on the document, quantifying how much the document falls in each topic. The words clouds and topic weights for the validated edges are presented in Figures \ref{fig:word_cloud_ukreuters_reuters}, \ref{fig:wc_breitbart_wral}, \ref{fig:wc_canadianbusiness_wral} and \ref{fig:wc_bizjournals_wral}. Details about the generation of words clouds and topic weights are deferred to Appendix \ref{sec:meme_validation}.



The {\em number} and {\em topics} of the influence-involved phrase clusters should reflect stimulatory influences between media sources {\em qualitatively}, and thus can facilitate our comparison between the logistic-normal and multinomial approaches given that there is no ground truth. However, we don't expect the validation procedure to provide us with an accurate network estimate due to the following reasons: this procedure only looks at marginal dependence of each receiver media source on an influencer media source, instead of its conditional dependence on that influencer media source given all media sources; meanwhile, the group sparsity structure is not leveraged to handle the high-dimensional problem.  

\paragraph{Validation for {\em uk.reuters}$\rightarrow${\em reuters}:} We look at the number of influence-involved phrase clusters, from \emph{uk.reuters} to each of its neighbors (those appearing in the sub-networks). 
We calculate the percentage of all phrase clusters each neighbor ever posts that are influence-involved, and the top 5 neighbors with highest percentages are presented in Table \ref{tab:phrase_number_ukreuters_out}, where {\em reuters} has the highest percentage. 
We further investigate the topics of influence from \emph{uk.reuters} to \emph{reuters}, through the word cloud and topic weights in Figure \ref{fig:word_cloud_ukreuters_reuters}, which suggest that the logistic-normal
approach estimates the edges more accurately than the multinomial approach. 
\begin{table}[ht!]
\centering
	\begin{tabular}{|c|c|c|c|}
		\hline
		\multirow{2}{*}{Neighbor}&Total Phrase Clusters &Influence-involved &\multirow{2}{*}{Percent}\\
		&this Neighbor Posts&Phrase Clusters&\\
		\hline
		\emph{reuters}&7928&875&11.04\% \\
		\emph{alertnet.org}&2552&139&5.45\%\\
		\emph{ca.rd.yahoo}&6227&362&5.81\%\\
		\emph{uk.news.yahoo}&16502&955&5.79\%\\
		\emph{earthtimes.org}&2808&89&3.17\%\\
		\hline
	\end{tabular}
	\caption{\small {Number of phrase clusters that are posted at least once by each neighbor media source of \emph{uk.reuters} (column 2)}, and number of influence-involved phrase clusters from \emph{uk.reuters} to each neighbor (column 3). The last column is the percentage of the phrase clusters each neighbor posts that are influence-involved. {\em  Top 5 neighbor media sources with largest percentages are presented, upon which the influence of \emph{uk.reuters} is supported by the ``Phrase cluster data''}. 
	}\label{tab:phrase_number_ukreuters_out}
\end{table}

			\begin{figure}[ht!]
			\begin{minipage}{0.38\linewidth}
			\centering
			\includegraphics[width=0.9\linewidth]{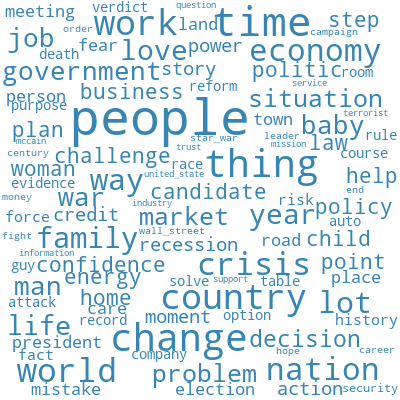}
			\end{minipage}
			\begin{minipage}{0.6\linewidth}
			\centering
				\begin{tabular}{|c|c|c|c|c|}
					\hline	\multirow{2}{*}{Sports}&International&\multirow{2}{*}{Lifestyle}&\multirow{2}{*}{Finance}&\multirow{2}{*}{Health}\\
					& Affairs&&&\\
					\hline
					\textbf{0.2359}&\textbf{0.4634}&\textbf{0.1274}&\textbf{0.1384}&0.0348\\
					\hline
				\end{tabular}
			\end{minipage}
			\caption{\small {\bf (\emph{uk.reuters}$\rightarrow$ \emph{reuters})} The word cloud and topic weights of the document consisting of influenced-involved phrase clusters from \emph{uk.reuters} to \emph{reuters}. We can see from the word cloud that these phrase clusters cover ``International affairs" (\eg words like ``country", ``world"), ``Finance", (\eg ``economy", ``market"), and ``Lifestyle
			 (\eg ``love", ``baby"). Although we can see few words clearly referring to sports, all the first 4 topics have non-negligible topic weights in the table above. We believe the reason is that, the topic ``Sports" is not exclusively about sports although we name it so, as indicated by the key words in Table \ref{tab:kw_5topics}. Specifically, its top 10 keywords include ``time", ``lot", ``thing", which do not clearly refer to any topic. {\em The word cloud together with the topic weights provide evidence for the edges estimated by the logistic-normal method other than the multinomial method}, since the latter does not estimate a red edge (influence in ``Lifestyle''), either in the absolute sub-network (Figure \ref{fig:meme_network_ukreuters}(c)) or the relative sub-network (Figure \ref{fig:meme_network_ukreuters}(d)).}\label{fig:word_cloud_ukreuters_reuters}
			\end{figure}
\paragraph{Validating the three edges pointing to {\em wral}:} We consider the number of influence-involved phrase clusters, from each neighbor to \emph{wral}. 
We also calculate the percentage of all phrase clusters each neighbor posts that are influence-involved, and the top 5 neighbors with highest percentages are presented in Table \ref{tab:phrase_number_wral_in}. Neighbors ({\em breitbart}, {\em canadianbusiness}, {\em bizjournals}) sending the three edges that will be validated are all listed in the table. 

	\begin{table}[ht!]
	\centering
		\begin{tabular}{|c|c|c|c|}
			\hline
			\multirow{2}{*}{Neighbor}&Total Phrase Clusters  &Influence-involved &\multirow{2}{*}{Percent}\\
			&this Neighbor Posts&Phrase Clusters&\\
			\hline
			\emph{daytondailynews}&6571&768&11.69\%\\
			\emph{canadianbusiness}&2339&252&10.77\%\\
			\emph{breitbart}&19279&1408&7.30\% \\
			\emph{bizjournals}&1069&27&2.53\%\\
			\emph{newsobserver}&5107&120&2.35\%\\
			\hline
		\end{tabular}
		\caption{\small  Number of phrase clusters that are posted for at least once by each neighbor media source of \emph{wral} (column 2), and number of influence-involved phrase clusters from each neighbor to \emph{wral} (column 3). The last column is the percentage of the phrase clusters each neighbor posts that are influence-involved. {\em Top 5 media sources with largest percentages are presented, whose influence upon \emph{wral} is supported by the ``Phrase cluster data''.}. 
		}\label{tab:phrase_number_wral_in}
	\end{table}
	

	\begin{figure}[hb!]
	\begin{minipage}{0.38\linewidth}
	\centering
		\includegraphics[width=0.9\linewidth]{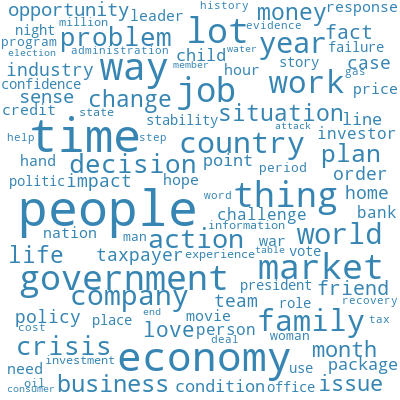}
	\end{minipage}
	\begin{minipage}{0.6\linewidth}
	\centering
		\begin{tabular}{|c|c|c|c|c|}
			\hline
			\multirow{2}{*}{Sports}&International&\multirow{2}{*}{Lifestyle}&\multirow{2}{*}{Finance}&\multirow{2}{*}{Health}\\
			& Affairs&&&\\
			\hline
			\textbf{0.2412}&\textbf{0.3392}&\textbf{0.1030}&\textbf{0.2744}&0.0423\\
			\hline
		\end{tabular}
	\end{minipage}
		\caption{\small {\bf (\emph{breitbart}$\rightarrow$ \emph{wral})} The word cloud and topic weights of the document consisting of influenced-involved phrase clusters from \emph{breitbart} to \emph{wral}
			. The word cloud suggests that the influence is mainly on ``International affairs" (\eg words like ``people", ``government") and ``Finance" (\eg ``market", ``crisis", ``company", ``money"), but also about ``Lifestyle" (\eg ``family", ``love", ``friend", ``child") and ``Sports" (\eg ``team", ``point"). Meanwhile, the first 4 topics all have non-negligible weights, as shown in the table above. {\em This is consistent with the edges estimated by the logistic-normal approach but not the multinomial approach}, since the latter only estimates a green edge in the absolute sub-network (Figure \ref{fig:meme_network_wral}(c)) and no edge in the relative sub-network (Figure \ref{fig:meme_network_wral}(d)), from \emph{breitbart} to \emph{wral}.
		}\label{fig:wc_breitbart_wral}
	\end{figure}
%
The word cloud and topic weights in Figure \ref{fig:wc_breitbart_wral} suggest that the logistic-normal approach estimates the edges from \emph{breitbart} more accurately, while those in Figure \ref{fig:wc_canadianbusiness_wral} and \ref{fig:wc_bizjournals_wral} demonstrate that the multinomial approach estimates the edges from \emph{canadianbusiness} and \emph{bizjournals} more accurately. 
	\begin{figure}[hb!]
	\begin{minipage}{0.38\linewidth}
	\centering
		\includegraphics[width=0.9\linewidth]{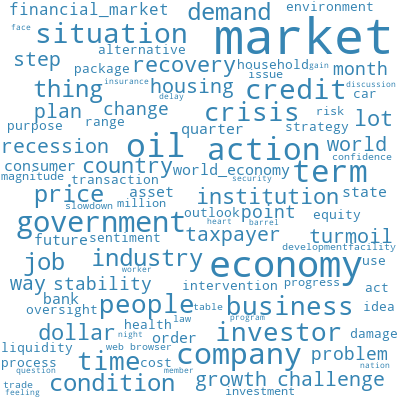}
	\end{minipage}
	\begin{minipage}{0.6\linewidth}
	\centering
		\begin{tabular}{|c|c|c|c|c|}
			\hline
			\multirow{2}{*}{Sports}&International&\multirow{2}{*}{Lifestyle}&\multirow{2}{*}{Finance}&\multirow{2}{*}{Health}\\
			& Affairs&&&\\
			\hline
			0.1062&0.2153&0.0144&\textbf{0.6255}&0.0387\\
			\hline
		\end{tabular}
	\end{minipage}
		\caption{\small {\bf (\emph{canadianbusiness}$\rightarrow$ \emph{wral})} The word cloud and topic weights of the document consisting of influenced-involved phrase clusters from \emph{canadianbusiness} to \emph{wral}. Both the word cloud and topic weights suggests the influence of \emph{canadianbusiness} on \emph{wral} to be primarily about ``Finance'': most words in the word cloud are finance-related, \eg ``market'', ``economy'', ``company'', ``demand''; the topic weight in ``Finance'' is more than 0.5. The multinomial approach estimates a green (``Finance'') edge and a blue (``Sports'') edge from \emph{canadianbusiness} to \emph{wral} (both in Figure \ref{fig:meme_network_wral}(c) and \ref{fig:meme_network_wral}(d)), while the logistic-normal approach only estimates a blue (``Sports'') edge. {\em Therefore the multinomial approach may be more accurate than the logistic-normal approach in estimating this edge}.}\label{fig:wc_canadianbusiness_wral}
	\end{figure}
	
\clearpage
	\begin{figure}[hb!]
	\begin{minipage}{0.38\linewidth}
	\centering
		\includegraphics[width=0.9\linewidth]{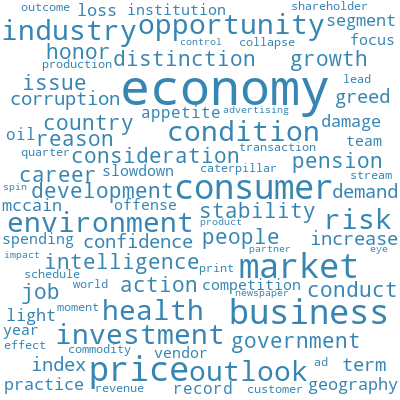}
	\end{minipage}
	\begin{minipage}{0.6\linewidth}
	\centering
		\begin{tabular}{|c|c|c|c|c|}
			\hline
			\multirow{2}{*}{Sports}&International&\multirow{2}{*}{Lifestyle}&\multirow{2}{*}{Finance}&\multirow{2}{*}{Health}\\
			& Affairs&&&\\
			\hline
			0.1020&0.1057&0.0017&\textbf{0.7472}&0.0435\\
			\hline
		\end{tabular}
	\end{minipage}
		\caption{\small {\bf (\emph{bizjournals}$\rightarrow$ \emph{wral})} The word cloud and topic weights of the document consisting of influenced-involved phrase clusters from \emph{bizjournals} to \emph{wral}
			. Both the word cloud and topic weights show that the influence of \emph{bizjournals} on \emph{wral} is primarily about ``Finance'': most words in the word cloud are finance-related, \eg ``economy'', ``price'', ''business'', ``market''; the topic weight in ``Finance'' is larger than 0.5. The multinomial approach estimates a green (``Finance'') edge in both the absolute and relative sub-networks, while the logistic-normal approach estimates a red (``Lifestyle'') edge and a dashed (inhibitory) green (``Finance'') edge from \emph{bizjournals} to \emph{wral}. {\em Therefore the multinomial method may be more accurate than the logistic-normal method in estimating this edge}. 
		}\label{fig:wc_bizjournals_wral}
	\end{figure}

	
\paragraph{Hypothesis support based on validated edges:} Table \ref{tab:meme_validation} and the detailed arguments above suggest that the logistic-normal method estimates edges better if they connect {\em uk.reuters}, {\em reuters} and {\em breitbart}, while the multinomial method estimates edges better if they connect {\em canadianbusiness} and {\em bizjournals}. The first three media sources tend to cover multiple topics, while the latter two media sources tend to be primarily about one topic. To further emphasize this {\em mixed membership} or {\em single category} behavior, we consider the \emph{top} topic weights of averaged posts sent by each media source within each time interval, and take an average over all time intervals when each media source posts. We present the average \emph{top} topic weights of these 5 media sources in Table \ref{tab:top_wts_4medias}. A higher top topic weight suggests less mixed membership. We can see that posts sent by {\em uk.reuters}, \emph{reuters}, \emph{breitbart} within the same time units are more mixed in topics, while those by \emph{bizjournals}, \emph{canadianbusiness} are more exclusively about one topic. This finding further validates our main hypothesis from the previous section.
\begin{table}[ht!]
	\centering
	\begin{tabular}{|c|c|c|c|c|c|}
		\hline
		\multirow{2}{*}{Media sources}&\multirow{2}{*}{breitbart}&\multirow{2}{*}{reuters}&\multirow{2}{*}{uk.reuters}&\multirow{2}{*}{bizjournals}&canadian-\\
		&&&&&business\\
		\hline
		Top topic weight&0.4061&0.4110&0.4400&0.5468&0.5694\\
		\hline
		\% of media sources&\multirow{2}{*}{8.62\%}&\multirow{2}{*}{10.34\%}&\multirow{2}{*}{27.59\%}&\multirow{2}{*}{68.98\%}&\multirow{2}{*}{84.48\%}\\
		with lower top weights&&&&&\\
		\hline
	\end{tabular}
\caption{\small Top topic weights of the averaged posts within each time unit sent by the 5 media sources, averaged over time. The 5 medias all have edges estimated well by one of the methods but not the other. The third row is the percentage of all 58 media sources that have lower top topic weights than the media source in the first row. A higher top topic weight and percentage suggests that the posts sent by the media source are more likely to fall in one topic, while a lower top topic weight suggests more mixed membership. The edges of the first three media sources ({\em lower top topic weights}) in this table are estimated well by the {\em logistic-normal} approach, while that of the last two media sources ({\em higher top topic weights}) are estimated well by the {\em multinomial} approach.}\label{tab:top_wts_4medias}
\end{table}
\subsection{Summary of findings}
Since  real data validation is quite involved, we briefly summarize the key findings in Table \ref{tab:real_data_summary}, which provides further evidence for the hypothesis that the logistic-normal approach will be more effective at estimating influences among nodes whose events exhibit mixed memberships in multiple categories; while for a node more likely to have events in one category than others and thus each of its events falls in that category, the multinomial approach will be more effective.
\begin{table}[ht!]
    \centering
    \begin{tabular}{|c|c|c|c|}
    \hline
         \multirow{2}{*}{Examples}&\multirow{2}{*}{Prediction}&\multirow{2}{*}{Network estimates}&Mixed membership\\
         &&&v.s. single category\\
         \hline
         Political&\multirow{2}{*}{MN is better}&\multirow{2}{*}{MN is better}&Each Twitter user has\\
         tweets&&&one ideology tendency\\
         \hline
         \multirow{6}{*}{MemeTracker}&&LN better for &{\em reuters}, {\em uk.reuters}\\
         &&{\em uk.reuters}  $\rightarrow$ {\em reuters}& and {\em breitbart}\\
         &Both methods&and {\em breitbart} $\rightarrow$ {\em wral}&cover multiple topics\\
         \cline{3-4}
         &work well&MN better for&{\em canadianbusiness}\\
         &&{\em canadianbusiness} $\rightarrow$  {\em wral}&and {\em bizjournals} are\\
         &&and {\em bizjournal} $\rightarrow$ {\em wral}&primarily about one topic\\
         \hline
    \end{tabular}
    \caption{Summary of comparison between the two methods in the two real data examples. ``MN'' refers to the multinomial method while ``LN'' refers to the logistic-normal method. The last column shows whether nodes exhibit mixed membership in multiple categories or falls mainly in single categories and further validates our main hypothesis.}
    \label{tab:real_data_summary}
\end{table}

\newpage

\section{Proofs}\label{sec:proofs}
In this section we provide proofs for Theorem \ref{thm:mult}, \ref{thm:cq} and \ref{thm:q_A}. Proofs for the lemmas are deferred to the appendix.
\subsection{Proof of Theorem \ref{thm:mult}}\label{sec:proof_mult}
We prove the error bounds for arbitrary $1\leq m\leq M$ and then take a union bound.
 Let $\Delta_m\in \mathbb{R}^{K\times M\times K}$, and define
\begin{equation}
\begin{split}
    F(\Delta_m)=& L_m^{\MN}(A_m^{\MN}+\Delta_m)-L_m^{\MN}(A_m^{\MN})+\lambda \|A_m^{\MN}+\Delta_m\|_R-\lambda \|A_m^{\MN}\|_R,
\end{split}
\end{equation}
where $$L_m^{\MN}(A_m)=\frac{1}{T}\sum_{t=0}^{T-1}\left[f(\langle A_m, X^{t}\rangle+\nu^{\MN}_m)-\sum_{k=1}^K \langle A_{mk}, X^t\rangle X^{t+1}_{mk}\right],\quad f(x)=\log \left(\sum_{i=1}^K e^{x_i}+1\right),$$ and $$\|A_m\|_R=\sum_{m'=1}^M\|A_{m,:,m',:}\|_F.$$
Our goal is to show that if $F(\Delta_m)\leq 0$, the following holds with high probability:
\begin{equation}\label{eq:err_bounds_mult}
    \|\Delta_m\|_F^2\leq \frac{C\rho^{\MN}_m\log M}{T},\quad\|\Delta_m\|_R\leq C\rho^{\MN}_m\sqrt{\frac{\log M}{T}}.
\end{equation}
The following lemma shows that we only need to prove the claim above for $\|\Delta_m\|_R\leq C$.
\begin{lemma}\label{lem:bounded_err}
    For any convex function $g$ and norm $\|\cdot\|$, if $g(0)=0$, $g(x)>0$ as long as $\|x\|= C$, then $g(x)\leq 0$ implies $\|x\|<C$.
\end{lemma}
Since $F(\cdot)$ is convex, we only need to show that $F(\Delta_m)\leq 0$ and $\|\Delta_m\|_R\leq C$ imply the error bounds \eqref{eq:err_bounds_mult}. This is because that the error bounds suggest $\|\Delta_m\|_R\leq C\rho^{\MN}_m\sqrt{\frac{\log M}{T}}<C$, thus the condition in Lemma \ref{lem:bounded_err} holds.

Denote the Bregman divergence induced by any function $g$ as $D_g(\cdot, \cdot)$, then if $F(\Delta_m)\leq 0$, 
\begin{equation}\label{eq:std_eq_mult}
    \begin{split}
        D_{L_m^{\MN}}(A_m^{\MN}+\Delta_m,A_m^{\MN})\leq &-\langle \nabla L_m^{\MN}(A_m^{\MN}),\Delta_m\rangle+\lambda \|A_m^{\MN}\|_R-\lambda \|A_m^{\MN}+\Delta_m\|_R,
    \end{split}
\end{equation}
The following lemmas provide an upper bound for the R.H.S.
\begin{lemma}\label{lem:dev_bnd_mult}
   	   		Under the model generation process \eqref{eq:mult}, with probability at least $1-\exp\{-c\log M\}$, 
   	   		\begin{equation*}
   	   		\left\|L_m^{\MN}(A_m^{\MN})\right\|_{R^*}< CK\sqrt{\frac{\log M}{T}}\leq \frac{\lambda}{2},
   	   		\end{equation*}
   	   		where $C>0$ is a universal constant.
   	   	\end{lemma}
   	   	
      	Thus we can bound the R.H.S. of \eqref{eq:std_eq_mult} by 
      	\begin{equation*}
      	    \frac{\lambda}{2}\|\Delta_m\|_{R}+\lambda\|\Delta_{m,:,S_m^{\MN},:}\|_R-\lambda\|\Delta_{m,:,(S_m^{\MN})^c,:}\|_R\leq \frac{3\lambda}{2}\|\Delta_{m,:,S_m^{\MN},:}\|_R-\frac{\lambda}{2}\|\Delta_{m,:,(S_m^{\MN})^c,:}\|_R.
      	\end{equation*}
      	By the definition of $L_m^{\MN}$, 
      	\begin{equation*}
      	    \begin{split}
      	         D_{L_m^{\MN}}(A_m^{\MN}+\Delta_m,A_m^{\MN})=&\frac{1}{T}\sum_{t=0}^{T-1}D_f(\langle A_m^{\MN}, X^t\rangle+\nu^{\MN}_m+\langle \Delta_m,X^t\rangle,\langle A_m^{\MN}, X^t\rangle+\nu^{\MN}_m)\\
      	         \geq &\frac{1}{T}\sum_{t=0}^{T-1}\frac{\lambda_{\min}(\nabla^2f(\xi^t))}{2}\|\langle \Delta_m,X^t\rangle\|_2^2,
      	    \end{split}
      	\end{equation*}
       where $\xi^t\in\bbR^K$ is some point lying between $\langle A_m^{\MN}+\Delta_m,X^t\rangle+\nu_m^{\MN}$ and $\langle A^{\MN}_m,X^t\rangle+\nu_m^{\MN}$. Since we have assumed 
       $$\|A^{\MN}\|_{\infty,\infty,1,\infty}\leq R_{\max}^{\MN},\quad \|\Delta_m\|_R\leq C,
       $$
       we know $\langle A_m^{\MN}+\Delta_m,X^t\rangle+\nu^{\MN}_m,\langle A^{\MN}_m,X^t\rangle+\nu^{\MN}_m\in [-C,C]^K$ where $C$ depends on $R_{\max}^{\MN}, \|\nu^{\MN}\|_{\infty}$, and thus $\xi^t\in [-C,C]^K$.
       
       The next step is to lower bound $\lambda_{\min}(\nabla^2 f(\xi^t))$. First we calculate the Hessian matrix of $f$:
      	\begin{equation*}
      	\left(\nabla^2 f(x)\right)_{ij}=-\frac{e^{x_i+x_j}}{\left(\sum_{k=1}^K e^{x_k}+1\right)^2}+\frac{e^{x_i}\ind{i=j}}{\sum_{k=1}^K e^{x_k}+1},
      	\end{equation*}
      	then for any $u\in \bbR^K$,
      	\begin{equation*}
      	\begin{split}
      	u^\top \nabla^2 f(x)u=&\sum_{i,j}u_iu_j \left(\nabla^2 f(x)\right)_{ij}\\
      	=&\left(\sum_{k=1}^K e^{x_k}+1\right)^{-2}\left\{-\left(\sum_{i=1}^K u_ie^{x_i}\right)^2+\left(\sum_{i=1}^K u_i^2e^{x_i}\right)\left(\sum_{i=1}^K e^{x_i}+1\right)\right\}\\
      	\geq &\left(\sum_{k=1}^K e^{x_k}+1\right)^{-2}\left(\sum_{i=1}^K u_i^2e^{x_i}\right)\\
      	\geq &\|u\|_2^2 \min_i e^{x_i}\left(\sum_{k=1}^K e^{x_k}+1\right)^{-2}.
      	\end{split}
      	\end{equation*}
      	The third line is due to Cauchey-Schwartz inequality: 
      	\begin{equation*}
      	    \left(\sum_{i=1}^K u_ie^{x_i}\right)^2=\left(\sum_{i=1}^K u_ie^{\frac{x_i}{2}}e^{\frac{x_i}{2}}\right)^2\leq \left(\sum_{i=1}^K u_i^2e^{x_i}\right)\left(\sum_{i=1}^K e^{x_i}\right).
      	\end{equation*} 
      	Therefore, $\lambda_{\min}(\nabla^2 f(\xi^t))\geq \frac{e^{-C}}{\left(K e^{C}+1\right)^{2}}> 0$. Combining this with \eqref{eq:std_eq_mult}, we know that 
      	\begin{equation}\label{eq:Delta_Cone_mult}
      	\left\|(\Delta_m)_{:,(S_m^{\MN})^c,:}\right\|_{R}\leq 3\left\|(\Delta_m)_{:,S_m^{\MN},:}\right\|_{R},
      	\end{equation}  
       
       Now we would like to lower bound $\frac{1}{T}\sum_{t=0}^{T-1}\left\|\langle \Delta_m, X^t\rangle\right\|_2^2$ with the following restricted eigenvalue condition. First we define set $\cC(S,\kappa)$ of $K\times M\times K$ tensors, for any set $S\subseteq \{1,\dots,M\}$, and constant $\kappa>0$:
       $$
       \cC(S,\kappa)=\{U\in \bbR^{K\times M\times K}: \left\|U_{:,S^c,:}\right\|_{R}\leq \kappa\left\|U_{:,S,:}\right\|_{R}\}.
       $$ 
       \begin{lemma}\label{lem:REC_mult}
       	Under the model generation process \eqref{eq:mult}, if $T\geq C_1(\rho^{\MN}_m)^2\log M$, then with probability at least $1-\exp\left\{-c_1\log M\right\}$, 
       	\begin{equation*}
       	\inf_{U\in \cC(S_m^{\MN},3)}\frac{1}{T}\sum_{t=0}^{T-1}\frac{\left\|\langle U, X^t\rangle\right\|_2^2}{\|U\|_F^2}\geq c_2,
       	\end{equation*}
       	where $c_1>0$ is a universal constant and $C_1, c_2>0$ depend on $K, R_{\max}^{\MN}, \|\nu^{\MN}\|_{\infty}$.
       \end{lemma}
   By \eqref{eq:Delta_Cone_mult}, $\Delta_m\in \cC(S_m^{\MN},3)$. Therefore, with probability at least $1-C\exp\left\{-c\log M\right\}$, 
   \begin{equation*}
   \left\|\Delta_m\right\|_{F}^2\leq C\lambda\left\|(\Delta_m)_{:,S_m^{\MN},:}\right\|_{R}\leq C\lambda\sqrt{\rho^{\MN}_m}\left\|(\Delta_m)\right\|_{F}
   \end{equation*}
   which further implies,
   \begin{equation*}
   \left\|\Delta_m\right\|_{F}\leq C_1\sqrt{\frac{\rho^{\MN}_m\log M}{T}}
   \end{equation*}
   and 
   \begin{equation*}
   \left\|\Delta_m\right\|_{R}\leq 4\left\|(\Delta_m)_{:,S_m^{\MN},:}\right\|_{R}\leq 4\sqrt{\rho^{\MN}_m}\left\|\Delta_m\right\|_{F}\leq C_2\rho^{\MN}_m\sqrt{\frac{\log M}{T}},
   \end{equation*}
   where constant $C_1, C_2>0$ depend only on $R_{\max}^{\MN}, \|\nu^{\MN}\|_{\infty}$ and $K$. 
   
\subsection{Proof of Theorem \ref{thm:cq}}\label{sec:proof_cq}
We follow similar steps from the proof of Theorem \ref{thm:mult}. Here for any $\Delta_m\in \bbR^{(K-1)\times M\times K}$ we define $F(\Delta_m)$ as
\begin{equation}
\begin{split}
    F(\Delta_m)=& L_m^{\LN}(A_m^{\LN}+\Delta_m)-L_m^{\LN}(A_m^{\LN})+\lambda \|A_m^{\LN}+\Delta_m\|_R-\lambda \|A_m^{\LN}\|_R,
\end{split}
\end{equation}
where $L_m^{\LN}(A_m)=\frac{1}{2T}\sum_{t\in \mathcal{T}_m}\|Y^{t+1}_m-\mu^{t+1}_m(A_m)\|_2^2$. We will prove that $F(\Delta_m)\leq 0$ implies the error bounds for $\Delta_m$.
We start with the standard equations 
\begin{equation}\label{eq:std_eq_GSM_cq}
    \begin{split}
        D_{L_m^{\LN}}(A_m^{\LN}+\Delta_m,A_m^{\LN})\leq -\langle \nabla L_m^{\LN}(A_m^{\LN}),\Delta_m\rangle+\lambda \|A_m^{\LN}\|_R-\lambda\|A_m^{\LN}+\Delta_m\|_R.
    \end{split}
\end{equation}
\begin{lemma}[Deviation Bound]\label{lem:db_GSM_cq}
Under the data generation process \eqref{eq:model_disc} and \eqref{eq:model_cts} with $q^t=q$, 
\begin{align*}
\left\| \nabla L_m^{\LN}(A_m^{\LN})\right\|_{R^*}\leq CK\max_k\Sigma_{kk}\sqrt{\frac{\log M|\mathcal{T}_m|}{T^2}}\leq \frac{\lambda}{2}.
\end{align*}
With probability at least $1-\exp(-c\log M)$, for universal constants $c, C>0$.
\end{lemma}
Similarly we can also write 
\begin{equation*}
     -\langle \nabla L_m^{\LN}(A_m^{\LN}),\Delta_m\rangle+\lambda \|A_m^{\LN}\|_R-\lambda\|A_m^{\LN}+\Delta_m\|_R\leq \frac{3\lambda}{2}\|\Delta_{m,:,S_m^{\LN},:}\|_R-\frac{\lambda}{2}\|\Delta_{m,:,(S_m^{\LN})^c,:}\|_R,
\end{equation*}
and thus $\|\Delta_{m,:,(S_m^{\LN})^c,:}\|_R\leq 3\|\Delta_{m,:,S_m^{\LN},:}\|_R$. By the definition of $L_m^{\LN}$, $D_{L_m^{\LN}}(A_m^{\LN}+\Delta_m,A_m^{\LN})=\frac{1}{2T}\sum_{t\in \cT_m}\|\langle \Delta_m,X^t\rangle\|_2^2$, and it can be lower bounded based on the following Lemma that holds for  
$\cC(S_m^{\LN},3)=\{U\in \bbR^{(K-1)\times M\times K}: \left\|U_{:,(S_m^{\LN})^c,:}\right\|_{R}\leq \kappa\left\|U_{:,S_m^{\LN},:}\right\|_{R}\}$,
\begin{lemma}[Restricted Eigenvalue Condition]\label{lem:REC_GSM_cq}
Under the data generation process \eqref{eq:model_disc} and \eqref{eq:model_cts} with $q^t=q$, if $T\geq C_1\frac{(\rho^{\LN}_m)^2\log M}{q_m^2}$, 
$$
\inf_{U\in \cC(S_m^{\LN},3)}\frac{1}{2T\|U\|_F^2}\sum_{t\in \mathcal{T}_m}\|\langle U,X^t\rangle\|_2^2\geq c_1q_m,
$$
with probability at least $1-\exp\{-c\log M\}$. Here $C_1, c_1>0$ depend only on $R_{\max}^{\LN}$, $\|\nu^{\LN}\|_{\infty}$, $\|\Sigma\|_{\infty}$, $\lambda_{\min}(\Sigma)$ and $K$.
\end{lemma}

 Due to Lemma \ref{lem:REC_GSM_cq}, with probability at least $1-C\exp\{-c\log M\}$,
\begin{equation}
\begin{split}
    \|\Delta_m\|_F^2 \leq &C\frac{\lambda}{q_m}\|\Delta_{m,S_m^{\LN}}\|_R\leq C\sqrt{\frac{\rho^{\LN}_m\log M\max_{m'}|\mathcal{T}_{m'}|}{q_m^2T^2}}\|\Delta_m\|_F,\\
    \|\Delta_m\|_R\leq &C\rho^{\LN}_m\sqrt{\frac{\log M\max_{m'}|\mathcal{T}_{m'}|}{q_m^2 T^2}}.
\end{split}
\end{equation}
The following lemma provides an upper bound for $|\mathcal{T}_m|$:
\begin{lemma}\label{lem:T_m_bnd}
$$
\bbP\left(|\mathcal{T}_m|>2q_mT\right)\leq \exp\{-2q_m^2T\}.
$$
\end{lemma}
Therefore, if $T\geq C_1\max_m \frac{(\rho^{\LN}_m)^2\log M}{q_m^2}$, with probability at least $1-C\exp\{-c_1\log M\}$,
\begin{equation}
\begin{split}
    \|\Delta_m\|_F^2 \leq & C_2\frac{\max_{m'}q_{m'}}{q_m^2}\frac{\rho^{\LN}_m\log M}{T},\\
    \|\Delta_m\|_R\leq &C_2\rho^{\LN}_m\sqrt{\frac{\max_{m'}q_{m'}}{q_m^2}\frac{\log M}{T}},
\end{split}
\end{equation}
holds for $1\leq m\leq M$, and thus
\begin{equation}
\begin{split}
    \|\widehat{A}^{\LN}-A^{\LN}\|_F^2 \leq & C_2\frac{\max_{m}q_{m}}{\min_{m} q_m^2}\frac{s^{\LN}\log M}{T},\\
    \|\widehat{A}^{\LN}-A^{\LN}\|_R\leq &C_2s^{\LN}\sqrt{\frac{\max_{m}q_{m}}{\min_m q_m^2}\frac{\log M}{T}}.
\end{split}
\end{equation}
Here $c_1, C_1, C_2>0$ depend only on $R_{\max}^{\LN}$, $\|\nu^{\LN}\|_{\infty}$, $\|\Sigma\|_{\infty}$, $\lambda_{\min}(\Sigma)$ and $K$.

\subsection{Proof of Theorem \ref{thm:q_A}}\label{sec:proof_q_A}
 Similarly from the previous proofs, we only prove the error bounds for an arbitrary $m$ first. Let $\Delta_m^A\in \mathbb{R}^{(K-1)\times M\times K}$, $\Delta_m^B\in \mathbb{R}^{M\times K}$, and $\Delta_m(\alpha)\in \mathbb{R}^{K\times M\times K}$ be concatenated by $\sqrt{\alpha}\Delta_m^A$ and $\sqrt{1-\alpha}\Delta_m^B$ in the first dimension. Formally, $\Delta_{m,1:(K-1),:,:}(\alpha)=\sqrt{\alpha}\Delta_m^A$, $\Delta_{m,K,:,:}(\alpha)=\sqrt{1-\alpha}\Delta_m^B$. For simplicity, we will omit $\Delta_m(\alpha)$ to $\Delta_m$. Define
\begin{equation}
\begin{split}
    F(\Delta_m)=&\alpha L_m^{\LN}(A_m^{\LN}+\Delta_m^A)+(1-\alpha)L_m^{\bern}(B_m^{\bern}+\Delta_m^B)+\lambda R_{\alpha}(A_m^{\LN}+\Delta_m^A,B_m^{\bern}+\Delta_m^B)\\
    &-\alpha L_m^{\LN}(A_m^{\LN})-(1-\alpha)L_m^{\bern}(B_m^{\bern})-\lambda R_{\alpha}(A_m^{\LN},B_m^{\bern}).
\end{split}
\end{equation}
Our goal is to show that if $F(\Delta_m)\leq 0$, the following holds with high probability:
\begin{equation}\label{eq:err_bounds_q_A}
\begin{split}
    \|\Delta_m\|_F^2=\alpha\|\Delta_m^A\|_F^2+(1-\alpha)\|\Delta_m^B\|_F^2\leq \frac{C\rho^{\LN,\bern}_m\log M}{T},\\
    \|\Delta_m\|_R=R_{\alpha}(\Delta_m^A,\Delta_m^B)\leq C\rho^{\LN,\bern}_m\sqrt{\frac{\log M}{T}}.
\end{split}
\end{equation}
Given Lemma \ref{lem:bounded_err}, we only need to show that $F(\Delta_m)\leq 0$ and $\|\Delta_m\|_R\leq \sqrt{1-\alpha}$ imply the error bounds \eqref{eq:err_bounds_q_A}. This is because that the error bounds suggest $\|\Delta_m\|_R\leq C\rho^{\LN,\bern}_m\sqrt{\frac{\log M}{T}}<\sqrt{1-\alpha}$, thus the condition in Lemma \ref{lem:bounded_err} holds.

If $F(\Delta_m)\leq 0$, 
\begin{equation}\label{eq:std_eq_q_A}
    \begin{split}
        &\alpha D_{L_m^{\LN}}(A_m^{\LN}+\Delta_m^A,A_m^{\LN})+(1-\alpha)D_{L_m^{\bern}}(B_m^{\bern}+\Delta_m^B,B_m^{\bern})\\
        \leq &-\alpha \langle \nabla L_m^{\LN}(A_m^{\LN}),\Delta_m^A\rangle -(1-\alpha)\langle \nabla L_m^{\bern}(B_m^{\bern}),\Delta_m^B\rangle\\
        &+\lambda R_{\alpha}(A_m^{\LN},B_m^{\bern})-\lambda R_{\alpha}(A_m^{\LN}+\Delta_m^A,B_m^{\bern}+\Delta_m^B).
    \end{split}
\end{equation}
The following lemmas provide an upper bound for the R.H.S.
\begin{lemma}[Deviation bound for continuous error]\label{lem:db_GSM_q_A}
Under the data generation process \eqref{eq:model_disc}, \eqref{eq:model_cts}, \eqref{eq:q_A_dis}, with probability at least $1-\exp(-c\log(M))$,
\begin{align*}
\left\| \nabla L_m^{\LN}(A_m^{\LN})\right\|_{\infty}\leq C\max_k\sqrt{\Sigma_{kk}}\sqrt{\frac{\log(M)}{T}},
\end{align*}
for universal constants $c, C>0$.
\end{lemma}
\begin{lemma}[Deviation bound for discrete error]\label{lem:db_BAR}
    Under the data generation process \eqref{eq:model_disc}, \eqref{eq:model_cts}, \eqref{eq:q_A_dis}, with probability at least $1-\exp(-c\log M)$,
    \begin{equation*}
    \left\| \nabla L_m^{\bern}(B_m^{\bern})\right\|_{\infty}\leq C\sqrt{\frac{\log(M)}{T}},
    \end{equation*}
    for universal constants $c, C>0$.
\end{lemma}

By Lemma \ref{lem:db_GSM_q_A} and Lemma \ref{lem:db_BAR}, with probability at least $1-\exp\{-c\log M\}$,
\begin{equation*}
    \begin{split}
        &-\alpha\left\langle \nabla L_m^{\LN}(A_m^{\LN}),\Delta_m^A\right\rangle -(1-\alpha)\left\langle \nabla L_m^{\bern}(B_m^{\bern}),\Delta_m^B\right\rangle\\
        =&-\sum_{m'=1}^M \left\langle \sqrt{\alpha}(\nabla L_m^{\LN}(A_m^{\LN}))_{:,m',:},\sqrt{\alpha}\Delta^A_{m,:,m',:}\right\rangle+\left\langle \sqrt{1-\alpha}(\nabla L_m^{\LN}(B_m^{\bern}))_{m',:},\sqrt{1-\alpha}\Delta^B_{m,m',:}\right\rangle\\
        \leq &\sum_{m'=1}^M\left(\alpha\|\nabla L_m^{\LN}(A_m^{\LN})_{:,m',:}\|_F^2+(1-\alpha)\|\nabla L_m^{\bern}(B_m^{\bern})_{m',:}\|_2^2\right)^\frac{1}{2}\|\Delta_{m,:,m',:}(\alpha)\|_F\\
        \leq &\left(C_1\alpha(K-1)K\Sigma_{kk}\frac{\log M}{T}+C_2(1-\alpha)K\frac{\log M}{T}\right)^{\frac{1}{2}}\|\Delta_m\|_R\\
        \leq &\left(C_1(K-1)\max_k\Sigma_{kk}\alpha+C_2(1-\alpha)\right)^{\frac{1}{2}}\sqrt{\frac{K\log M}{T}}\|\Delta_m\|_R.
    \end{split}
\end{equation*}
Setting $\lambda=C(\alpha)K\sqrt{\frac{\log M}{T}}$, where $C(\alpha)=\left[C_1\max_k \Sigma_{kk}\alpha+C_2(1-\alpha)\right]^{\frac{1}{2}}$ for some universal constants $C_1, C_2>0$. Then we have 
\begin{equation*}
    -\alpha\left\langle \nabla L_m^{\LN}(A_m^{\LN}),\Delta_m^A\right\rangle -(1-\alpha)\left\langle \nabla L_m^{\bern}(B_m^{\bern}),\Delta_m^B\right\rangle\leq \frac{\lambda}{2}\|\Delta_m\|_R.
\end{equation*}
Let $S_m^{\LN, \bern}=\{(i,j,k):\alpha\|A^{\LN}_{m,:,j,:}\|_F^2+(1-\alpha)\|B^{\bern}_{m,:,j,:}\|_F^2>0\}$ be the support set of $A_m^{\LN}$ and $B_m^{\bern}$, then we can write 
\begin{equation*}
\begin{split}
    &R_{\alpha}(A_m^{\LN},B_m^{\bern})-R_{\alpha}(\widehat{A}_m^{\LN},\widehat{B}_m^{\bern})\\
    =&R_{\alpha}(A_{m,S_m^{\LN, \bern}}^{\LN},B_{m,S_m^{\LN, \bern}}^{\bern})-R_{\alpha}(\widehat{A}_{m,S_m^{\LN, \bern}}^{\LN},\widehat{B}_{m,S_m^{\LN, \bern}}^{\bern})-R_{\alpha}(\widehat{A}_{m,(S_m^{\LN, \bern})^c}^{\LN},\widehat{B}_{m,(S_m^{\LN, \bern})^c}^{\bern})\\
    \leq &R_{\alpha}(\Delta_{m,S_m^{\LN, \bern}}^A,\Delta_{m,S_m^{\LN, \bern}}^B)-R_{\alpha}(\Delta^A_{m,S_m^{\LN, \bern c}},\Delta^B_{m,(S_m^{\LN, \bern})^c})\\
    =&\|\Delta_{m,S_m^{\LN,\bern}}\|_R-\|\Delta_{m,(S_m^{\LN, \bern})^c}\|_R
\end{split}
\end{equation*}
Therefore, the R.H.S of \eqref{eq:std_eq_q_A} is bounded by $\frac{3\lambda}{2}\|\Delta_{m,S_m^{\LN, \bern}}\|_R-\frac{\lambda}{2}\|\Delta_{m,(S_m^{\LN, \bern})^c}\|_R$. Since $L_m^{\LN}$ and $L_m^{\bern}$ are both convex, the L.H.S. of \eqref{eq:std_eq_q_A} is non-negative. Thus $\|\Delta_{m,(S_m^{\LN, \bern})^c}\|_R\leq 3\|\Delta_{m,S_m^{\LN, \bern}}\|_R$. Define set $\cC(S_m^{\LN, \bern},\kappa)$ of $K\times M\times K$ tensors for any $\kappa>0$ as follows:
       \begin{equation}\label{eq:cone_def_q_A}
           \cC(S_m^{\LN, \bern},\kappa)=\{U\in \bbR^{K\times M\times K}: \left\|U_{(S_m^{\LN, \bern})^c}\right\|_{R}\leq \kappa\left\|U_{S_m^{\LN, \bern}}\right\|_{R}\},
       \end{equation}
       then $\Delta_m\in \cC(S_m^{\LN, \bern},3)$.

Now we would like to show the strong convexity of $L_m^{\LN}$ and $L_m^{\bern}$ as a function of $\langle A_m, X^t\rangle$ and $\langle B_m,X^t\rangle$. As shown in the proof of Theorem \ref{thm:cq}, 
\begin{equation}\label{eq:stconv1_q_A}
    D_{L_m^{\LN}}(A_m^{\LN}+\Delta_m^A,A_m^{\LN})=\frac{1}{2T}\sum_{t=0}^{T-1}\ind{X^{t-1}_m\neq 0}\|\langle \Delta_m^A,X^t\rangle\|_2^2.
\end{equation}
Meanwhile, $\|\Delta_m^B\|_{1,\infty}\leq \|\Delta_m^B\|_{1,2}\leq \frac{\|\Delta_m\|_R}{\sqrt{1-\alpha}}$, thus the strong convexity of $L_m^{\bern}$ is guaranteed by the following lemma:
\begin{lemma}[Strong convexity ($L_m^{\bern}$)]\label{lem::rsc_BAR}
Define $\sigma_B\triangleq \frac{e^{-R_{\max}^{\LN, \bern}-1}}{(1+e^{R_{\max}^{\LN, \bern}+1})^2}$, then we have
\begin{equation*}
    D_{L_m^{\bern}}(B_m^{\bern}+\Delta_m^B,B_m^{\bern})\geq \frac{\sigma_B}{2T}\sum_{t=0}^{T-1}\langle \Delta^B_m,X^t\rangle^2.
\end{equation*}
\end{lemma}

The following Lemma provides a lower bound for 
\begin{equation*}
    \frac{\alpha}{2T}\sum_{t=0}^{T-1}\ind{X^{t-1}_m\neq 0}\|\langle \Delta_m^A,X^t\rangle\|_2^2+\frac{(1-\alpha)\sigma_B}{2T}\sum_{t=0}^{T-1}\langle \Delta^B_m,X^t\rangle^2
\end{equation*}
in terms of $\|\Delta_m\|_F^2$.
\begin{lemma}[Restricted Eigenvalue Condition]\label{lem:REC_q_A}
For any $U\in \mathbb{R}^{K\times M\times K}$, let $U^{(1)}=U_{1:(K-1),:,:}$ and $U^{(2)}=U_{K,:,:}$. There exists a constant $c_1$, such that if $T\geq C_1(\rho^{\LN,\bern}_m)^2\log M$, 
$$
\inf_{U\in \cC(S_m^{\LN, \bern},3)\cap B_F(1)}\frac{1}{2T}\sum_{t=0}^{T-1}\ind{\X^{t-1}_m\neq 0}\|\langle U^{(1)},X^t\rangle\|_2^2+\frac{\sigma_B}{2T}\sum_{t=0}^{T-1}\langle U^{(2)},X^t\rangle^2\geq c_1,
$$
with probability at least $1-\exp\{-c\log M\}$. Here $C_1, c_1>0$ depend only on $R_{\max}^{\LN, \bern}$, $\|\Sigma\|_{\infty}$, $\lambda_{\min}(\Sigma)$, $\|\nu^{\LN}\|_{\infty}$, $\|\eta^{\bern}\|_{\infty}$ and $K$.
\end{lemma}
Therefore, combining \eqref{eq:std_eq_q_A}, \eqref{eq:stconv1_q_A}, Lemma \ref{lem::rsc_BAR} and Lemma \ref{lem:REC_q_A} leads us to
\begin{equation*}
    \begin{split}
        \|\Delta_m\|_F^2\leq C_1C(\alpha)\frac{\rho^{\LN,\bern}_m\log M}{T}, &\quad\|\Delta_m\|_R\leq 4\sqrt{\rho^{\LN,\bern}_m}\|\Delta_m\|_F\leq C_1C(\alpha)\rho^{\LN,\bern}_m\sqrt{\frac{\log M}{T}},
    \end{split}
\end{equation*}
with probability at least $1-C\exp\{-c\log M\}$. Here $C(\alpha)=\left[C_2\max_k \Sigma_{kk}\alpha+C_3(1-\alpha)\right]^{\frac{1}{2}}$ for some universal constants $C_2, C_3>0$, and $C_1$ depends only on $R_{\max}^{\LN, \bern}$, $\|\Sigma\|_{\infty}$, $\lambda_{\min}(\Sigma)$, $\|\nu^{\LN}\|_{\infty}$, $\|\eta^{\bern}\|_{\infty}$ and $K$. Taking a union bound over $1\leq m\leq M$ gives us the final result.

\section{Conclusion}

In this paper, we develop two procedures that estimate context-dependent networks from point process event data. The first approach is a standard regularized multinomial approach for estimating the influence between pairs of nodes $(m, m')$ and pairs of categories $(k, k')$ given that each event belongs to a particular category. Our second logistic-normal approach builds on ideas from compositional time series and is more nuanced since each event consists of a composition of several different topics. We extend existing compositional time series approaches by accounting for the scenario 
in which no event occurs in our algorithm; significantly, the logistic-normal distribution leads to a convex objective. Our theoretical guarantees show that we can achieve consistent estimation even when the number of network nodes, $M$,  is much larger than the duration of the observation period, $T$.

We validate our network estimation procedures both with synthetic and two real data examples. Both the synthetic and real data examples suggest that the multinomial approach is better suited to nodes or networks where events tend to belong to a single category, whereas the logistic-normal approach is better suited to nodes in which each event tends to have mixed membership. 

\section*{Acknowledgements}
LZ, GR, BM, and RW were partially supported by ARO W911NF-17-1-0357, NGA HM0476-17-1-2003. GR was also partially supported by NSF DMS-1811767. RW was also partially supported by NSF DMS-1930049, NSF Awards 0353079, 1447449, 1740707, and 1839338.

\bibliographystyle{abbrvnat} 
\bibliography{reference.bib}
\appendix
\section{Proof of Lemmas}
In this section, we present the proofs of the lemmas used in \ref{sec:proofs}.
\subsection{Proof of Lemmas in Section \ref{sec:proof_mult}}\label{sec:proof_lem_mult}
\begin{proof}[proof of Lemma \ref{lem:bounded_err}]
    We prove by contradiction. Assume that their exists $\|x\|>C$ and $g(x)\leq 0$, then let $\gamma=\frac{C}{\|x\|}<1$. Due to the convexity of $g$,
    \begin{equation*}
        g(\gamma x)=g(\gamma x+(1-\gamma)*0)\leq \gamma g(x)+(1-\gamma)g(0)=\gamma g(x)\leq 0.
    \end{equation*}
    However, $\|\gamma x\|=C$. This contradicts with our condition, so we are forced to conclude that $\|x\|\leq C$ is necessary for $g(x)\leq 0$.
\end{proof}
\begin{proof}[Proof of Lemma \ref{lem:dev_bnd_mult}]
By the definition of $L_m^{\MN}$,
\begin{equation*}
    \nabla L_m^{\MN}(A_m^{\MN})=-\frac{1}{T}\sum_{t=0}^{T-1}(X^{t+1}_m-\nabla f(\langle A_m^{\MN}, X^t\rangle)\otimes X^t.
\end{equation*}
Define $\epsilon^{t+1}_m:=X^{t+1}_m-\bbE\left(X^{t+1}_m|\cF_t\right)$, where $\cF_t=\sigma(X^0,\dots,X^t)$ is the filtration.
Since 
   	$$
   	(\nabla f(x))_i=\frac{e^{x_i}}{\sum_{j=1}^K e^{x_j}+1},
   	$$ 
   	we can write $\nabla L_m^{\MN}(A_m^{\MN})=-\frac{1}{T}\sum_{t=0}^{T-1}\epsilon^{t+1}_m\otimes X^t.$ 
	First note that 
	\begin{equation*}
	\left\|\frac{1}{T}\sum_{t=0}^{T-1}\epsilon^{t+1}_m\otimes X^t\right\|_{R^*}=\max_{m'} \left\|\frac{1}{T}\sum_{t=0}^{T-1}\epsilon^{t+1}_m X_{m'}^{t\top}\right\|_{F}\leq \max_{m',k',k}K\left|\frac{1}{T}\sum_{t=0}^{T-1}\epsilon^{t+1}_{mk} X_{m'k'}^t\right|,
	\end{equation*}
	thus we only need to look into $\frac{1}{T}\sum_{t=0}^{T-1}\epsilon^{t+1}_{mk}X^t_{m'k'}$ for any $m',k',k$, and then take a union bound.
	Let $Y_n= \frac{1}{T}\sum_{t=0}^{n-1}\epsilon^{t+1}_{mk}X^t_{m'k'}$, then $\{Y_n\}_{n=0}^T$ is a martingale sequence, with $Y_0=0$. Since 
	\begin{equation*}
	\xi_n\triangleq Y_n-Y_{n-1}=\frac{1}{T}\epsilon^{n}_{mk}X^{n-1}_{m'k'},
	\end{equation*}
	$|\xi_n|\leq \frac{1}{T}$. Thus by Azuma-Hoeffding's inequality, for any $y>0$,
	\begin{equation*}
	\bbP(Y_T\geq y)\leq \exp\{-\frac{Ty^2}{2}\}.
	\end{equation*}
	Let $y=C\sqrt{\frac{\log M}{T}}$ and take a union bound over each $m',k',k$, we know that 
	\begin{equation*}
	\begin{split}
	&\bbP\left(\left\|\frac{1}{T}\sum_{t=0}^{T-1}\epsilon^{t+1}_m\otimes X^t\right\|_{R^*}\geq CK\sqrt{\frac{\log M}{T}}\right)\\
	\leq &KM^2\exp\left\{-\frac{Ty^2}{2}\right\}\\
	=&\exp\left\{\log K-(C^2/2-2)\log M\right\}\\
	\leq &\exp\{-c\log M\}.
	\end{split}
	\end{equation*}
\end{proof}

\begin{proof}[Proof of Lemma \ref{lem:REC_mult}]
	For notational convenience, we view $X^t$ as a $MK$-dimensional vector and $U$ as $K\times MK$ dimensional matrix in this proof. First note that 
	\begin{equation*}
	\begin{split}
	\frac{1}{T}\sum_{t=0}^{T-1}\left\|\langle U,X^t\rangle\right\|_2^2=&\sum_{k=1}^KU_k^\top\frac{1}{T}\sum_{t=0}^{T-1}\bbE(X^tX^{t\top}|\cF_{t-1})U_k\\
	&+\sum_{k=1}^KU_k^\top\frac{1}{T}\sum_{t=0}^{T-1}\left(X^tX^{t\top}-\bbE(X^tX^{t\top}|\cF_{t-1})\right)U_k.
	\end{split}
	\end{equation*}
	In the following steps we provide a lower bound for the first term, and concentrate the second term around 0.
	\begin{itemize}
		\item[(1)] Lower bound for the first term\\
		We can decompose the conditional expectation $\bbE(X^tX^{t\top}|\cF_{t-1})$ as two terms:
		$$\bbE(X^tX^{t\top}|\cF_{t-1})=\bbE(X^t)\bbE(X^{t\top})+\Cov(X^t|\cF_{t-1}),
		$$
		where the first term is positive semi-definite, and the second term is a block diagonal matrix ($\Cov(X^t_m,X^t_{m'}|\cF_{t-1})=0$ if $m\neq m'$). Thus we only have to lower bound the eigenvalue of the each $\Cov(X^t_m|\cF_{t-1})$.
		Define matrix $p^t\in \bbR^{M\times (K+1)}$ as follows:
		\begin{equation*}
		\begin{split}
		    p^t_{mk}=&\bbP(X^t_{mk}=1|\cF_{t-1})=\frac{\exp\{\langle A^{\MN}_{mk},X^{t-1}\rangle\}+\nu^{\MN}_{mk}}{1+\sum_{l=1}^K \exp\{\langle A^{\MN}_{ml},X^{t-1}\rangle+\nu^{\MN}_{ml}\}},\quad 1\leq k\leq K\\
		    p^t_{m,K+1}=&\bbP(X^t_m=0|\cF_{t-1})=\frac{1}{1+\sum_{l=1}^K \exp\{\langle A^{\MN}_{ml},X^{t-1}\rangle+\nu^{\MN}_{ml}\}}.
		\end{split}
		\end{equation*} 
		Since $\|A^{\MN}\|_{\infty, \infty, 1,\infty}\leq R_{\max}^{\MN}$, $\exists 0<C_1<C_2<1$, such that $p^t\in [C_1,C_2]^{M\times (K+1)}$ where $C_1, C_2$ depend on $R_{\max}^{\MN}, \|\nu^{\MN}\|_{\infty}$ and $K$. We can write 
		\begin{equation*}
		\begin{split}
		\Cov(X^t_m|\cF_{t-1})=\begin{pmatrix}
		p^t_{m1}&0&\dots&0\\
		0&p^t_{m2}&\dots&0\\
		\vdots&\vdots&\ddots&\vdots\\
		0&\dots&\dots&p^t_{mK}
		\end{pmatrix}-\begin{pmatrix}p^t_{m1}\\\vdots\\p^t_{mK}\end{pmatrix}\begin{pmatrix}p^t_{m1}&\dots&p^t_{mK}\end{pmatrix},
		\end{split}
		\end{equation*}
		For any vector $u\in \bbR^K$, 
		\begin{equation*}
		    \begin{split}
		        u^\top \Cov(X^t_m|\cF_{t-1})u=&\sum_{k=1}^K p^t_{mk}u_k^2-\left(\sum_{k=1}^Kp^t_{mk}u_k\right)^2\\
		        \geq &\sum_{k=1}^K p^t_{mk}u_k^2-\sum_{k=1}^K p^t_{mk}\left(\sum_{k=1}^K p^t_{mk}u_k^2\right)\\
		        =&p^t_{m,K+1}\left(\sum_{k=1}^K p^t_{mk}u_k^2\right)\\
		        \geq &p^t_{m,K+1}\min_k p^t_{mk}\|u\|_2^2,
		    \end{split}
		\end{equation*}
		Thus the eigenvalues of $\Cov(X^t_m|\cF_{t-1})$ are lower bounded by some constant $c$ depending on $K, R_{\max}^{\MN}$ and $\|\nu^{\MN}\|_{\infty}$.
		\item[(2)] Concentration bound for the second term\\
		Since $U\in \cC(S_m^{\MN},3)$, 
		\begin{equation*}
		\begin{split}
		    &\left|\sum_{k=1}^KU_k^\top\frac{1}{T}\sum_{t=0}^{T-1}\left(X^tX^{t\top}-\bbE(X^tX^{t\top}|\cF_{t-1})\right)U_k\right|\\
		    \leq &\sum_{k=1}^K\|U_k\|_1^2\left\|\frac{1}{T}\sum_{t=0}^{T-1}\left(X^tX^{t\top}-\bbE(X^tX^{t\top}|\cF_{t-1})\right)\right\|_{\infty}\\
		    \leq &16K^2\rho_m^{\MN}\|U\|_F^2\left\|\frac{1}{T}\sum_{t=0}^{T-1}\left(X^tX^{t\top}-\bbE(X^tX^{t\top}|\cF_{t-1})\right)\right\|_{\infty}.
		\end{split}
		\end{equation*}
		We can bound $\left\|\frac{1}{T}\sum_{t=0}^{T-1}\left(X^tX^{t\top}-\bbE(X^tX^{t\top}|\cF_{t-1})\right)\right\|_{\infty}$ using the same argument as the proof of Lemma \ref{lem:dev_bnd_mult}. For arbitrary $m,k$, let $$Y_n:=\frac{1}{T}\sum_{t=0}^{n-1}\left(X^t_{mk}X^{t\top}_{m'k'}-\bbE(X^t_{mk}X^{t\top}_{m'k'}|\cF_{t-1})\right)$$ for $n\geq 1$, and $Y_0=0$, then $\{Y_n\}$ is a bounded difference martingale sequence. Since $|Y_n-Y_{n-1}|\leq \frac{1}{T}$, applying Azuma-Hoeffding's inequality and taking a union bound over $m,k$ would lead us to 
		\begin{equation*}
		    \bbP\left(\left\|\frac{1}{T}\sum_{t=0}^{T-1}\left(X^tX^{t\top}-\bbE(X^tX^{t\top}|\cF_{t-1})\right)\right\|_{\infty}>C\sqrt{\frac{\log M}{T}}\right)\leq \exp\{-c\log M\}.
		\end{equation*}
		\end{itemize}
		Therefore,
		\begin{equation*}
		    \inf_{U\in \cC(S_m^{\MN},3)}\frac{1}{T}\sum_{t=0}^{T-1}\frac{\left\|\langle U,X^t\rangle\right\|_2^2}{\|U\|_F^2}\geq c-C\rho_m^{\MN}\sqrt{\frac{\log M}{T}}\geq \frac{c}{2},
		\end{equation*}
		when $T$ is sufficiently large.
\end{proof}

\subsection{Proof of Lemmas in Section \ref{sec:proof_cq}}\label{sec:proof_lem_cq}
\begin{proof}[proof of Lemma \ref{lem:db_GSM_cq}]
First we prove the upper bound conditioning on $\mathcal{T}_m=\{t_1,\dots,t_{|\mathcal{T}_m|}\}$.
	Since $\nabla L_m^{\LN}(A_m^{\LN})=-\frac{1}{T}\sum_{i=1}^{|\mathcal{T}_m|}\epsilon^{t_i+1}_{m}\otimes X^{t_i}$, we start by bounding each entry of $\frac{1}{T} \sum_{i=1}^{|\mathcal{T}_m|}\epsilon^{t_i+1}_{mk}X^{t_i}_{m' k'}$. Let 
	$$
	Y_n=\frac{1}{T} \sum_{i=1}^{n-1}\epsilon^{t_i+1}_{mk}X^{t_i}_{m' k'},
	$$ 
	with $Y_0=0$ and $Y_{|\mathcal{T}_m|}=\frac{1}{T} \sum_{i=1}^{|\mathcal{T}_m|}\epsilon^{t_i+1}_{mk}X^{t_i}_{m' k'}$. Then $\{Y_n\}_{n=0}^{|\mathcal{T}_m|}$ is a martingale with filtrations $\mathcal{F}_n=\sigma(X^1,\dots,X^{t_n},\mathcal{T}_m)$. Let $\xi_n=Y_n-Y_{n-1}=-\frac{1}{T}\epsilon^{t_{n-1}+1}_{mk}X^{t_{n-1}}_{m' k'}$ be the corresponding martingale difference sequence. The moment generating function of $Y_n$ satisfies
	\begin{equation}\label{mgf:db}
	\mathbb{E}(e^{\eta Y_n})=\mathbb{E}[e^{\eta Y_{n-1}}\mathbb{E}(e^{\eta \xi_n}|\mathcal{F}_{n-1})],
	\end{equation}  
	for any $\eta$. Since $\epsilon^{t_{n-1}+1}_{mk}\sim \mathcal{N}(0,\Sigma_{kk})$ given $\mathcal{F}_n$, we can bound $\mathbb{E}(e^{\eta \xi_n}|\mathcal{F}_{n-1})]$ in the following:
	\begin{equation*}
	\mathbb{E}(e^{\eta \xi_n}|\mathcal{F}_{n-1})= \mathbb{E}\left(\exp\left\{\frac{\eta X^{t_{n-1}}_{m' k'}}{T}\epsilon^{t_{n-1}+1}_{mk}\right\}|\mathcal{F}_{n-1}\right)\leq \exp\left\{\frac{\eta^2\Sigma_{kk}(X^{t_{n-1}}_{m' k'})^2}{2T^2}\right\} \leq \exp\left\{\frac{\eta^2\Sigma_{kk}}{2T^2}\right\}.
	\end{equation*}
	Therefore, combining this with \eqref{mgf:db} we have
	\begin{equation*}
	\mathbb{E}(e^{\eta Y_{T}})\leq e^{\frac{\eta^2\Sigma_{kk}|\mathcal{T}_m|}{2T^2}}.
	\end{equation*}
	Applying Chernoff bound further shows that, for any $\eta>0$,
	\begin{equation*}
	\begin{split}
	\mathbb{P}(|Y_{T}|> r)\leq& e^{-\eta r}\mathbb{E}(e^{\eta Y_{T}}+e^{-\eta Y_{T}})\\
	\leq &2\exp\left\{\frac{\eta^2\Sigma_{kk}|\mathcal{T}_m|}{2T^2}-\eta r\right\}.
	\end{split}
	\end{equation*}
	Let $\eta=\frac{rT^2}{\Sigma_{kk}|\mathcal{T}_m|}$, then
	\begin{equation*}
	\mathbb{P}(|Y_{T-1}|> r)\leq 2\exp\left\{-\frac{r^2T^2}{2\Sigma_{kk}|\mathcal{T}_m|}\right\}.
	\end{equation*}
	Now we take a union bound for all entries of $\frac{1}{T} \sum_{t \in  \mathcal{T}_m}\epsilon^{t+1}_m\otimes X^t$. 
	\begin{equation*}
	\begin{split}
	\mathbb{P}\left(\left\|\frac{1}{T} \sum_{t \in  \mathcal{T}_m}\epsilon^{t+1}_m\otimes X^t\right\|_{R^*}> r\right)\leq& \mathbb{P}\left(\left\|\frac{1}{T} \sum_{t \in  \mathcal{T}_m}\epsilon^{t+1}_m\otimes X^t\right\|_{\infty}> \frac{r}{K}\right)\\
	\leq& 2MK^2 \exp\left\{-\frac{r^2T^2}{2K^2\Sigma_{kk}|\mathcal{T}_m|}\right\}.
	\end{split}
	\end{equation*}
	Plug in $r=CK\sqrt{\Sigma_{kk}}\sqrt{\frac{\log M|\mathcal{T}_m|}{T^2}}\leq \frac{\lambda}{2}$, we obtain the final result.
\end{proof}

\begin{proof}[proof for Lemma \ref{lem:REC_GSM_cq}]
Similar from the proof of Lemma \ref{lem:REC_mult}, we can write
\begin{equation}\label{eq::REC_decomp_Dir_cq}
\begin{split}
       \frac{1}{T}\sum_{t\in \mathcal{T}_m}\|\langle U,X^t\rangle\|_2^2=&\sum_k U_k^\top \frac{1}{T}\sum_{t=0}^{T-1}\mathbb{E}(X^t X^{t\top}\ind{X^{t+1}_m\neq 0}|\mathcal{F}_{t})U_k\\
       +&\sum_k U_k^\top\frac{1}{T}\sum_{t=0}^{T-1}\left[X^t X^{t\top}\ind{X^{t+1}_m\neq 0}-\mathbb{E}(X^t X^{t\top}\ind{X^{t+1}_m\neq 0}|\mathcal{F}_{t})\right]U_k
\end{split}
\end{equation}
\begin{enumerate}
    \item[(1)]Bounding the eigenvalue of $\frac{1}{T}\sum_{t=0}^{T-1}\mathbb{E}(X^t X^{t\top}\ind{X^{t+1}_m\neq 0}|\mathcal{F}_{t})$\\
    We can write
	\begin{equation*}
	\mathbb{E}(X^tX^{t\top}\ind{X^{t+1}_m\neq 0}|\mathcal{F}_{t-1})=q_m\mathbb{E}(X^t|\mathcal{F}_{t-1})\mathbb{E}(X^t|\mathcal{F}_{t-1})^\top +q_m\Cov(X^t|\mathcal{F}_{t-1}),
	\end{equation*}
	where $\mathbb{E}(X^t|\mathcal{F}_{t-1})\mathbb{E}(X^t|\mathcal{F}_{t-1})^\top$ is positive semi-definite, thus the smallest eigenvalue can be lower bounded by that of $q_m\Cov(X^t|\mathcal{F}_{t-1})$. 
	
	Given $\mathcal{F}_{t-1}$, $X^t_1,\dots,X^t_M$ are all independent, which suggests $\Cov(X^t|\mathcal{F}_{t-1})$ to be a block diagonal matrix. We only need to lower bound the smallest eigenvalue of each $\Cov(X^t_m|\mathcal{F}_{t-1})$. Since each $X^t_m$ is non-degenerate, the smallest eigenvalue of $\Cov(X_m^t|\mathcal{F}_{t-1})$ is positive, being a function of $\langle A_m^{\LN},X^{t-1}\rangle$, $\nu^{\LN}_m$ and $\Sigma$.
	
	We denote the smallest eigenvalue as $\omega_m(\nu^{\LN}_m+\langle A_m^{\LN},X^{t-1}\rangle,\Sigma)$. Noting that moments are continuous function of distribution parameter, and eigenvalues are continuous functions of matrices, we know that $\omega_m(\nu^{\LN}_m+\langle A_m^{\LN},X^{t-1}\rangle,\Sigma)$ is continuous w.r.t. $\nu^{\LN}_m+\langle A_m^{\LN},X^{t-1}\rangle$ and $\Sigma$. Therefore, there exists a smallest $c>0$ such that the smallest eigenvalue of $\Cov(X^t_m|\mathcal{F}_{t-1})$ is always lower bounded by $c>0$ which depends on $K$, $R^{\LN}_{\max}=\|A^{\LN}\|_{\infty,\infty,1,\infty}$, $\|\nu^{\LN}_m\|_{\infty}$,$\|\Sigma\|_{\infty}$, and $\lambda_{\min}(\Sigma)$.
	
	Therefore, 
    \begin{equation*}
    \sum_k U_k^\top \frac{1}{T}\sum_{t=0}^{T-1}\mathbb{E}(X^t X^{t\top}\ind{X^{t+1}_m\neq 0}|\mathcal{F}_{t-1})U_k\geq cq_m\left\|U\right\|_F^2.
    \end{equation*}
    \item[(2)] Uniform concentration of martingale sequence\\
      Note that each element of $X^t X^{t\top}\ind{X^{t+1}_m\neq 0}-\mathbb{E}(X^t X^{t\top}\ind{X^{t+1}_m\neq 0}|\mathcal{F}_{t})$ is bounded by 1, we can still use the same argument as in the proof of Lemma \ref{lem:REC_mult} and obtain
      \begin{equation*}
		    \inf_{U\in \cC(S_m^{\LN},3)}\frac{1}{T}\sum_{t\in \cT_m}\frac{\left\|\langle U,X^t\rangle\right\|_2^2}{\|U\|_F^2}\geq cq_m-C\rho_m^{\LN}\sqrt{\frac{\log M}{T}}\geq \frac{c}{2},
		\end{equation*}
		when $T$ is sufficiently large.
\end{enumerate}
	\end{proof}
	
\begin{proof}[proof for Lemma \ref{lem:T_m_bnd}]
Note that we can write $|\mathcal{T}_m|=\sum_{t=0}^{T-1}\ind{X^{t+1}_m\neq 0}$, where $\ind{X^{t+1}_m\neq 0}$ are i.i.d. Bernoulli r.v., with sub-Gaussian parameter bounded by $\frac{1}{2}$. Applying Hoeffding's inequality would give us   
   $$
   \mathbb{P}\left(|\mathcal{T}_m|> 2q_mT\right)=\mathbb{P}\left(\sum_{t=0}^{T-1}\left(\ind{X^{t+1}_m\neq 0}-q_m\right)> q_mT\right)\leq \exp\{-2q_m^2T\}.
   $$
\end{proof}

\subsection{Proof of Lemmas in Section \ref{sec:proof_q_A}}\label{sec:proof_lem_q_A}
\begin{proof}[proof of Lemma \ref{lem:db_GSM_q_A}]
    The proof is the same as that of Lemma \ref{lem:db_GSM_cq}, except that we need to bound the infinity norm instead of $\|\cdot\|_R$. Using the same argument as in the proof of Lemma \ref{lem:db_GSM_cq}, we obtain
    \begin{equation*}
        \mathbb{P}\left(\left\|\nabla L_m^{\LN}(A_m^{\LN})\right\|_{\infty}>\eta\right)\leq 2K^2M\exp\{-\frac{\eta^2T}{2\Sigma_{kk}}\}.
    \end{equation*}
    Let $\eta=C\sqrt{\Sigma_{kk}}\sqrt{\frac{\log M}{T}}$, we have the final result.
\end{proof}

\begin{proof}[proof of Lemma \ref{lem:db_BAR}]
    By the definition of $L_m^{\bern}$, 
    $$
    \nabla L_m^{\bern}(B_m^{\bern})=-\frac{1}{T}\sum_{t=0}^{T-1}\varepsilon^{t+1}_m X^t,
    $$
    where $\varepsilon^{t+1}_m=\ind{X^{t+1}_m\neq 0}-P(X^{t+1}_m\neq 0|X^t).$ Since $\bbE(\varepsilon^{t+1}_m X^t|\cF_{t})=0$ each element of $\varepsilon^{t+1}_m X^t$ is bounded by $[-1,1]$, the argument used in the proof of Lemma \ref{lem:dev_bnd_mult} can be directly applied here, and leads us to
	\begin{equation*}
	    \bbP\left(\left\|\nabla L_m^{\bern}(B_m^{\bern})\right\|_{\infty}>C\sqrt{\frac{\log M}{T}}\right)\leq \exp\{-c\log M\}.
	\end{equation*}. 
\end{proof}

\begin{proof}[proof of Lemma \ref{lem::rsc_BAR}]
 Define $g(u)=\log(1+e^u)$, and $u^{t*}_m=\langle B_m^{\bern},X^t\rangle$, $\Delta u^t_m=\langle \Delta_m^B, X^t\rangle$, then we have
\begin{equation*}
\begin{split}
    D_{L_m^{\bern}}(B_m^{\bern}+\Delta_m^B,B_m^{\bern}) =&\frac{1}{T}\sum_{t=0}^{T-1}\left[g(u^{t*}_m+\Delta u^t_m)-g(u^{t*}_m)-g'(u^{t*}_m)\Delta u^t_m\right]\\
    =&\frac{1}{2T}\sum_{t=0}^{T-1}g''(\xi^t)(\Delta u^t_m)^2,
\end{split}
\end{equation*}
where $\xi^t$ lies between $u^{t*}_m$ and $u^{t*}_m+\Delta u^t_m$. Since $\|B_m^{\bern}\|_{1,\infty}\leq R_{\max}^{\LN, \bern}$, $\|\Delta_m^B\|_{1,\infty}\leq 1$, $u^{t*}_m\in [-R_{\max}^{\LN,\bern},R_{\max}^{\LN,\bern}]$, $\Delta u^t_m\in [-1,1]$. Therefore, 
\begin{equation*}
    g''(\xi^t)=\frac{e^{-\xi^t}}{(1+e^{-\xi^t})^2}\geq \exp\{-R_{\max}^{\LN,\bern}-1\}(1+\exp\{R_{\max}^{\LN,\bern}+1\})^{-2}=\sigma_B.
\end{equation*}
This implies
\begin{equation*}
    D_{L_m^{\bern}}(B_m^{\bern}+\Delta_m^B,B_m^{\bern})\geq \frac{\sigma_B}{2T}\sum_{t=0}^{T-1}\langle \Delta_m^B, X^t\rangle^2.
\end{equation*}
\end{proof}

\begin{proof}[proof of Lemma \ref{lem:REC_q_A}]
The proof is very similar to that of Lemma \ref{lem:REC_GSM_cq}. For notational convenience, we view $X^t$ and $U^{(2)}$ as $MK$-dimensional vector, $U^{(1)}$ as $(K-1)\times MK$ dimensional matrix. We can still write 
\begin{equation}
\begin{split}
      &\frac{1}{2T}\sum_{t=0}^{T-1}\ind{X^{t+1}_m\neq 0}\|\langle U^{(1)},X^t\rangle\|_2^2+\frac{\sigma_B}{2T}\sum_{t=0}^{T-1}\langle U^{(2)},X^t\rangle^2\\
      =&\frac{1}{2T}\sum_{t=0}^{T-1}\left\{\sum_{k=1}^{K-1}U^{(1)\top}_k\mathbb{E}\left[X^tX^{t\top}\ind{X^{t+1}_m\neq 0}|\mathcal{F}_{t-1}\right]U^{(1)}_k+\sigma_BU^{(2)\top}\mathbb{E}\left[X^tX^{t\top}|\mathcal{F}_{t-1}\right]U^{(2)}\right\}\\
      &+\frac{1}{2T}\sum_{t=0}^{T-1}\sum_{k=1}^{K-1}U^{(1)\top}_kP_1^{t}U_k^{(1)}+\sigma_B U^{(2)\top}P_2^tU^{(2)},
\end{split}
\end{equation}
where $$
P_1^t=X^tX^{t\top}\ind{X^{t+1}_m\neq 0}-\mathbb{E}\left[X^tX^{t\top}\ind{X^{t+1}_m\neq 0}|\mathcal{F}_{t-1}\right],\quad P_2^t=X^tX^{t\top}-\mathbb{E}\left[X^tX^{t\top}|\mathcal{F}_{t-1}\right].
$$
The last two terms can be bounded using the same argument as that in the proof of Lemma \ref{lem:REC_GSM_cq}. We only have to deal with the first two terms.
    Since 
    $$
    \bbE\left(X^tX^{t\top}\ind{X^{t+1}_m\neq 0}|\mathcal{F}_{t-1}\right)=\bbE\left(X^tX^{t\top}\bbP(X^{t+1}_m\neq 0|\cF_t)|\mathcal{F}_{t-1}\right),
    $$
    and 
    $$
    \bbP(X^{t+1}_m\neq 0|\cF_t)=\left(1+\exp\{-\langle B_m^{\bern},X^t\rangle\}\right)^{-1}\geq \frac{1}{1+e^{R_{\max}^{\LN,\bern}}},
    $$
    we have
    \begin{equation*}
    \begin{split}
         &U^{(1)\top}_k\mathbb{E}\left[X^tX^{t\top}\ind{X^{t+1}_m\neq 0}|\mathcal{F}_{t-1}\right]U^{(1)}_k\\
         =&\mathbb{E}\left[\bbP(X^{t+1}_m\neq 0|\cF_t)U^{(1)\top}_kX^tX^{t\top}U^{(1)}_k|\mathcal{F}_{t-1}\right]\\
         \geq& \frac{1}{1+e^{R_{\max}^{\LN,\bern}}}\mathbb{E}\left[U^{(1)\top}_kX^tX^{t\top}U^{(1)}_k|\mathcal{F}_{t-1}\right]\\
         \geq &\frac{\lambda_{\min}(\mathbb{E}\left[X^tX^{t\top}|\mathcal{F}_{t-1}\right])}{1+e^{R_{\max}^{\LN,\bern}}}\|U^{(1)}_k\|_F^2.
    \end{split}
    \end{equation*}
    Thus, 
    \begin{equation*}
        \begin{split}
            &\frac{1}{2T}\sum_{t=0}^{T-1}\left\{\sum_{k=1}^{K-1}U^{(1)\top}_k\mathbb{E}\left[X^tX^{t\top}\ind{X^{t+1}_m\neq 0}|\mathcal{F}_{t-1}\right]U^{(1)}_k+\sigma_BU^{(2)\top}\mathbb{E}\left[X^tX^{t\top}|\mathcal{F}_{t-1}\right]U^{(2)}\right\}\\
            \geq &\min_t \lambda_{\min}(\mathbb{E}\left[X^tX^{t\top}|\mathcal{F}_{t-1}\right])\left[\frac{\|U^{(1)}\|_F^2}{2(1+e^{R_{\max}^{\LN,\bern}})}+\frac{\sigma_B\|U^{(2)}\|_F^2}{2}\right]\\
            \geq &c\|U\|_F^2\min_t \lambda_{\min}(\mathbb{E}\left[X^tX^{t\top}|\mathcal{F}_{t-1}\right]).
        \end{split}
    \end{equation*}
    To lower bound $\min_t \lambda_{\min}(\mathbb{E}\left[X^tX^{t\top}|\mathcal{F}_{t-1}\right])$, we can use the same argument as in the proof of Lemma \ref{lem:REC_GSM_cq}. The only difference lies that $\lambda_{\min}(\mathbb{E}\left[X^tX^{t\top}|\mathcal{F}_{t-1}\right])$ depends on more parameters: $\{\langle A_m^{\LN}, X^{t-1}\rangle, \nu^{\LN}_m, \Sigma, \langle B_m^{\bern}, X^{t-1}\rangle, \eta_m^{\bern}\}_{m=1}^M$ and $K$. Therefore $\min_t \lambda_{\min}(\mathbb{E}\left[X^tX^{t\top}|\mathcal{F}_{t-1}\right])\geq c>0$ for $c$ depending on $K$, $R_{\max}^{\LN,\bern}$, $\|\Sigma\|_{\infty}$, $\lambda_{\min}(\Sigma)$, $\|\nu^{\LN}\|_{\infty}$, $\|\eta^{\bern}\|_{\infty}$.
	Therefore, 
	$$
	\inf_{U\in \cC(S_m^{\LN, \bern},3)\cap B_F(1)}\frac{1}{2T}\sum_{t=0}^{T-1}\ind{X^{t-1}_m\neq 0}\|\langle U^{(1)},X^t\rangle\|_2^2+\frac{\sigma_B}{2T}\sum_{t=0}^{T-1}\langle U^{(2)},X^t\rangle^2\geq c,
	$$
	with probability at least $1-\exp\left\{-c\log M\right\}$.
\end{proof}

\section{Detailed Procedures in Numerical Experiments}
\subsection{Data Generation Process of Synthetic Mixture Model}\label{sec:dgp_toymodel}
Formally, let $\mathcal{M}_1, \mathcal{M}_2\subset \{1,\dots,M\}$ be disjoint sets of nodes such that $\mathcal{M}_1\cup \mathcal{M}_2=\{1,\dots,M\}$, where $\mathcal{M}_1$ includes nodes of the first type (logistic-normally distributed), while nodes in $\mathcal{M}_2$ are of the second type (following multinomial distribution). Parameter sets $\{A_m^{\LN}\in \bbR^{(K-1)\times M\times K}, : m\in \mathcal{M}_1\}$, $\{B_m^{\bern}\in \bbR^{M\times K}: m\in \mathcal{M}_1\}$, $\{\nu_m^{\LN}\in \bbR^{K-1}: m\in \mathcal{M}_1\}$ and $\{\eta_m^{\bern}\in \bbR: m\in \mathcal{M}_1\}$ determine the conditional distribution of the first type of nodes, while $\{A_m^{\MN}\in \bbR^{K\times M\times K}: m\in \mathcal{M}_2\}$ and $\{\nu_m^{\MN}\in \bbR^{K}: m\in \mathcal{M}_2\}$ determine the conditional distribution of the second type of nodes. The data set $\{X^t\}_{t=0}^T$ is then generated as follows: initial data $\{X^0_m\in \bbR^K\}_{m=1}^M$ are i.i.d.
    multinomial random vectors, and at each time point $t+1$, $\{X^{t+1}_m\}_{m=1}^M$ are independent given the past.
    \begin{itemize}
        \item If $m\in \mathcal{M}_1$, then the distribution of $X^{t+1}_m\in \bbR^K$ given $X^t$ is specified by the logistic-normal modeling defined in \eqref{eq:model_disc}, \eqref{eq:model_cts} and \eqref{eq:q_A_dis}, with parameters $A_m^{\LN}$, $B_m^{\bern}$, $\nu_m^{\LN}$ and $\eta_m^{\bern}$;
        \item If $m\in \mathcal{M}_2$, the true categorical vector $\widetilde{X}^{t+1}_m\in \bbR^K$ follows multinomial distribution given $X^t$, as specified by \eqref{eq:mult} with parameters $A_m^{\MN}$ and $\nu_m^{\MN}$. Observed data $X^{t+1}_m=0^{K\times 1}$ if $\widetilde{X}^{t+1}_m=0^{K\times 1}$, otherwise, $X^{t+1}_m\in \bbR^K$ is a noisy version of $\widetilde{X}^{t+1}_m$, following logistic-normal distribution:
    \begin{equation}
	X^{t+1}_m\sim \begin{cases}
	\LN((-1,\dots,-1),\sigma),& \widetilde{X}^{t+1}_m=e_K,\\
	\LN(e_k,\sigma), & \widetilde{X}^{t+1}_m=e_k\text{ for }k<k,
	\end{cases}
	\end{equation}
	where $e_k$ refers to the $k$th vector in the canonical basis. Here we say a vector $Y\in \bbR^K$ follows $\LN(\mu,\sigma)$ for $\mu\in \bbR^{K-1}$ and $\sigma>0$, if $\log(\frac{Y_{1:(K-1)}}{Y_K})\sim \mathcal{N}(\mu,\sigma^2 I_{(K-1)\times (K-1)})$. Again, we assume $X^{t+1}_m$ to follow logistic-normal distribution, since it is widely used for modeling compositional data. In fact, the distribution of $X^{t+1}_m$ given $\widetilde{X}^{t+1}_m$ is designed to ensure that $\bbE\left(\log \frac{X^{t+1}_{mk}}{X^{t+1}_{mk'}}\right)=1$ if $\widetilde{X}^{t+1}_m=e_k$ and $k'\neq k$, for $1\leq k\leq K$. 
    \end{itemize} 
    
    We specify the parameters in the following. For simplicity, we assume 
the influence of events in one category is only imposed on future events in the same category, which is reasonable if we think of the categories as topics of news articles; also, events in the last category exerts and receives no influence, so that the relative influence encoded by $A_m^{\LN}$ can be interpreted as the absolute influence, as explained in Section \ref{sec:setup_mixed}. Therefore, for $m\in \mathcal{M}_1$, we set $B^{\bern}_{m,:,K}=0$, $A^{\LN}_{m,k,:,k}=B^{\bern}_{m,:,k}$, $A^{\LN}_{m,k,:,k'}=0$ for $1\leq k\leq K-1$ and $k'\neq k$; while for $m\in \mathcal{M}_2$, $A^{\MN}_{m,k,:,k'}=0$ for $k\neq k'$ or $k'=K$.

    The network parameters $\{A^{\LN}_m: m\in \mathcal{M}_1\}$ and $\{A^{\MN}_m: m\in \mathcal{M}_2\}$ have been visualized in Figure \ref{fig:toy_model_true_network}. For reproducibility, we present the non-zero parameter values here:
    \begin{equation}
    \begin{split}
    A^{\LN}_{1,(m-3)/3,m,(m-3)/3}=&0.5,\quad m=6, 9, 12,15,\\
    A^{\LN}_{m,k,1,k}=&1, \quad 2\leq m\leq 5, 1\leq k\leq 4\\
    A^{\MN}_{m,(m-3)/3,1,(m-3)/3}=&2, \quad m=6, 9, 12, 15,\\
    A^{\MN}_{(m+1):(m+2),(m-3)/3,m,(m-3)/3}=&(0.7, 0.7)^\top, \quad m=6, 9, 12, 15.
    \end{split}
    \end{equation}
    The intercept terms $\{\nu_m^{\MN}: m\in \mathcal{M}_1\}$, $\{\nu^{\LN}_m: m\in \mathcal{M}_2\}$ and $\{\eta ^{\bern}_m: m\in \mathcal{M}_2\}$ are defined to align with the preference of each node, so that nodes 1-5 are equally likely to have events in any of the first 4 categories, while each of nodes 6-8 (9-11, etc) is more likely to have events in one category than the other. More specifically, we set
    \begin{equation}
    \begin{split}
    \nu^{\LN}_{m,:}=(1,1,1,1,0), \quad 1\leq m\leq 5,\\
    \nu^{\MN}_{m,:}=
    \begin{cases}
    (1,0.5,0.5,0.5) &6\leq m\leq 8,\\
    (0.5,1,0.5,0.5), &9\leq m\leq 11,\\
    (0.5,0.5,1,0.5), &12\leq m\leq 14,\\
    (0.5,0.5,0.5,1), &15\leq m\leq 17.
    \end{cases}
    \end{split}
    \end{equation}
The noise level $\sigma$ for the contaminated multinomial vectors is set as $0.2$. The comparison results can be influenced by $\sigma^2$: when $\sigma^2$ gets too large, neither method works well and thus the performance gap between the two estimated networks on nodes 6-17 would be negligible.

\subsection{Data Preprocessing in Section \ref{sec:real_data}}\label{sec:real_data_prep}
Some details about how we obtain the membership vectors for each post in both examples are listed below.
\begin{enumerate}
    \item {\bf Identifying political tendencies of tweets:}\\
    We first use the tweets from the first half of the time period (55,859 tweets from Jan 1, 2016 to June 6, 2016) to train a neural network for categorizing tweets into two political tendencies (left- and right-leaning). The input feature vector of the neural network is an embedded vector of each tweet obtained by the standard pre-trained model BERT \citep{devlin2018bert, xiao2018bertservice} (uncased, 24-layer); and the partisanship of the user is used as the label (tweets sent by Democrats are all labeled as ``left-leaning''). The partisanship may not represent the true label, but due to the lack of human annotated labels, we believe the partisanship serves as a reasonable approximation, especially since politicians usually sent tweets with clear ideology. 
    
    The neural network is composed of three fully connected layers (two hidden layers of 128 nodes). RELU and softmax are the activation functions of the first two layers and the last layer respectively, and the cross entropy loss is used for training. 
    
    Since the tweets from the first half of the time period are already used for training the neural network, we don't include them in the input data set to our methods
to avoid over-fitting. The trained neural network model outputs a 2-dimensional vector on the simplex for each of the 27,600 tweets from June 7, 2016 to November 11, 2016, the second half of the time period. The neural network predicts the tweet to be left-leaning if the vector has larger value in its first coordinate, and right-leaning otherwise. Therefore, we use this vector as the mixed membership vector of the tweet, where the first coordinate is the membership in the left-leaning category and the second being that in the right-leaning category. 

    \item {\bf Topic membership vectors for memes in the MemeTracker example:}\\
    We first filter for the English media sources with high frequencies (more than 1500 posts included in the data set each month), 
which leads to a total of 5,684,791 posts from 101 media sources. For each post, we combine its recorded phrases/quotes together as the approximate content of the post. We then run topic modeling (Latent Dirichlet Allocation proposed in \cite{blei2003latent}) 
on these posts, where the number of topics is set as 5 ($K=5$), using the module \verb|gensim|\citep{rehurek_lrec} in python. For each topic, we present the top 10 keywords generated from topic modeling in the second column of Table \ref{tab:kw_5topics}, and we choose the topic names (the first column of Table \ref{tab:kw_5topics}) based on these keywords. 
\begin{table}[ht]
		\centering
		\begin{tabular}{|c|c|}
		\hline
			Topics &Keywords\\
			\hline
			Sports &time, people, lot, thing, game, way, team, work, player, year\\
			\hline
			International &people, country, government,  time, united\_states,\\
			Affairs& state, law, issue, case, work\\
			\hline
			Lifestyle &life, people, man, family, love, water, woman, world, story, music\\
			\hline
			\multirow{2}{*}{Finance}&market, company, business, economy, customer,\\& time, service, industry, bank, product\\
			\hline
			\multirow{2}{*}{Health} &child, patient, food, health, people, drug, hospital, \\&information, research, risk\\
			\hline
		\end{tabular}
		\caption{Keywords for the 5 topics generated from topic modeling.}\label{tab:kw_5topics}
	\end{table}
For each post item, topic modeling also outputs a corresponding $K$-dimensional weight vector on the simplex, indicating its memberships in the $K$ topics.

	Using 1-hour discretizations, we obtain a sample of size $T+1=5807$, 
	and if we want to learn the network among all of  the $101$ media sources, there would be $255,025$ ($101^2\times 5^2$) network parameters to estimate for both methods. Therefore for simplicity and interpretability, we select a subset of the 101 media sources and learn the network among them. To preserve a variety of topics covered in the posts, for each of the first 4 topics, we select the top 15 media sources that have the highest average topic weights in it.\footnote{No selected media has high weights in the topic ``health", so that we have a good choice for the baseline topic, as explained shortly.} 
	This leads us to a list of 58 media sources ($M=58$), due to some overlaps among top media sources in different topics, so the total number of network parameters to estimate is reduced to $84,100$.
\end{enumerate}
After we get the mixed membership vector of each post for each example, the time series data $\{X^t\in \bbR^{M\times K}\}_{t=0}^T$ is obtained as follows. For the political tweets data, the time period is discretized into $T+1=1000$ intervals of length approximately 3.7 hrs, while for the MemeTracker data, we use 1-hour discretization and end up with $T+1=5807$. After discretizing the time period into $T+1$ time intervals, the input data $\{X^t_m\in \bbR^K, 1\leq m\leq M, 0\leq t\leq T\}$ ($M$ is the number of nodes) is then constructed as follows: for each time window $t$, if there is no event associated with node $m$, let $X^t_m=0$; otherwise, (1) for the logistic-normal approach, let $X^t_m \in \mathbb{R}^{K}$ be the mixed membership vector (over the categories) of the event; (2) for the multinomial approach, let $X^t_m\in \mathbb{R}^K$ be the rounded mixed membership vector, that is, $X^t_m=e_k$ if the membership vector takes the largest value in the $k$th category, where $e_k$ is the $k$th canonical vector in $\bbR^K$. If there are multiple events associated with one node in the same time window, we average the mixed membership vector and use that as $X^t_m$ for the logistic-normal approach, and the rounded version of that average vector as $X^t_m$ for the multinomial approach.
\subsection{Choice of Baseline Topic for the MemeTracker example}\label{sec:meme_baseline}
We choose the baseline topic for the logistic-normal model in the MemeTracker example due to the following reasons. Due to our choice of the 58 media sources (the top 15 media sources in each of the first 4 topics) as explained in Appendix \ref{sec:real_data_prep}, there is no media focusing on the topic ``Health". 
	Therefore, we believe that the influence exerted upon the topic ``Health" might be weak. Thus, (1) it might be more interesting to see the influences received by the other 4 topics than that received by ``Health"; (2) 
	 the relative influence of a source topic on (a target topic compared to ``Health") should be close to the absolute influence of that source topic on the target topic, as mentioned in Section \ref{sec:setup_mixed}.
\subsection{Definition of Prediction Errors in Section \ref{sec:real_data}}\label{sec:real_data_prediction}
The prediction errors for the two methods are evaluated on hold-out sets (latter 30\% of each data set), after fitting the models using training sets (first 70\% of each data set). Throughout the real data experiments, all tuning parameters are chosen using cross-validation on the training sets\footnote{We use the same cross-validation method as that in the synthetic toy model experiment.}. 
	 The prediction error on a hold-out set is defined as follows:
	 \begin{itemize}
	     \item For a fitted multinomial model, given $X^{t-1}\in \bbR^{M\times K}$ (rounded data at time $t-1$ in the hold-out set), a one-step-ahead predicted probability vector $\hat{p}^t_m\in \bbR^{K+1}$ (the last dimension is the probability of no event) is output for each user $m$, according to \eqref{eq:mult}. The prediction for $X^t_m$ is defined as
\begin{equation*}
   \widehat{X}^t_m= \begin{cases}0,&\quad \arg\max_{k'} \hat{p}^t_{mk'}=K+1,\\
   e_k,&\quad \arg\max_{k'} \hat{p}^t_{mk'}=k\leq K,
    \end{cases}
\end{equation*} and the prediction error is calculated by $\frac{1}{TM}\sum_{t,m}\|X^t_m-\widehat{X}^t_m\|_2^2$, which is the proportion of wrong predictions for all nodes and time units in the hold-out set. Here $X^t_m$ is the observed rounded data. 
\item For a fitted logistic-normal model, given $X^{t-1}\in \bbR^{M\times K}$ (original, unrounded) in the hold-out set, a probability $\hat{q}^t_m$ is output for an event associated with node $m$ to occur at time $t$, specified by \eqref{eq:q_A_dis}; the expected log-ratios $\{\log \frac{\widehat{Z}^{t}_{mk}}{\widehat{Z}^{t}_{mK}}\}_{k=1}^{K-1}$ of the mixed membership vector $Z^t_m\in\triangle^{K-1}$ can also be specified by \eqref{eq:model_cts} with $\epsilon^t_{mk}=0$. Then we can transform the expected log-ratios back to $\widehat{Z}^t_m$ as the prediction for true mixed membership vector. Hence we define the prediction for $X^t_m$ as $\widehat{X}^{t}_m=\hat{q}^t_m \widehat{Z}^t_m$, and prediction error as $\sum_{t,m}\frac{\|X^t_m-\widehat{X}^t_m\|_2^2}{TM}$ (mean squared error). 
	 \end{itemize}

\subsection{Construction of Neighborhood Visualization for the MemeTracker example}\label{sec:meme_vis}
We present the neighborhood estimates around each media, instead of the whole network estimates among 58 media sources. In each sub-network, we include the central media's top 10 neighbors in any of the three networks. Edges sent to or from the central media node are presented, if their corresponding parameters have absolute values larger than $0.1$,\footnote{We use a smaller threshold here than the political tweets example (0.1 instead of 0.5), since we present the sub-networks around each node, instead of the whole network among all nodes. Smaller threshold can still preserve clarity of presentation.} and they encode influences between the same topic. That is to say, for relative sub-networks, we present the edges from each of the first 4 topics to \{the same topic compared to ``Health"\}; and for absolute sub-networks, we present the edges from each of the 5 topics to the same topic. 

\subsection{Generation of Word Clouds and Topic Weights in the MemeTracker Example}\label{sec:meme_validation}
 To understand the topics of the influence, we also combine those influence-involved phrase clusters together as one document. We remove the stop words and only preserve nouns in this document, just as what we did for the pre-processing of the topic modeling. Then we generate a word cloud for this pre-processed document using the module \verb|wordcloud|\footnote{https://github.com/amueller/word\_cloud} in Python, which assigns larger fonts to words with higher frequencies. The top 100 words with highest frequencies are included in each word cloud. 
	We also apply the previously trained topic model (mentioned in the beginning of Section \ref{sec:real_data_meme}) on the pre-processed document to obtain its topic weights, as a quantitative characterization of the influence strength in each topic.
\end{document}